\documentclass[preprint,authoryear,10pt,a4paper]{elsarticle}
\usepackage[top=0.75in, bottom=0.75in, left=0.75in, right=0.75in]{geometry}
\usepackage{palatino}
\usepackage{amsmath,amsthm}
\usepackage{amssymb}
\usepackage{bbm}
\usepackage{amsfonts}
\usepackage{graphicx}
\usepackage{subfigure}
\usepackage{grffile}
\graphicspath{{graphics/}}
\usepackage[ruled,vlined]{algorithm2e}
\usepackage{tabularx,booktabs}
\usepackage[utf8]{inputenc} 
\usepackage[T1]{fontenc}   
\usepackage{url}            % simple URL 
\usepackage{booktabs}       % 
\usepackage{amsmath,amsthm}
\usepackage{amssymb}
\usepackage{bbm}
\usepackage{amsfonts}       % blackboard math 
\usepackage{nicefrac}       % compact symbols 
\usepackage{microtype}      % microtypography
\usepackage{tikz}
\usepackage{rotfloat}
\usepackage{subfigure}
\usepackage{graphicx}
\usepackage{caption}
\usepackage{algorithmic}
\usepackage{appendix}
\usepackage{tablefootnote}
\usepackage{bm}

\usepackage{xcolor}
\definecolor{darkblue}{rgb}{0.0,0.5,0.5}
\definecolor{blue}{rgb}{0.0,0.5,0.68}
\definecolor{myblue}{RGB}{0,0,255}

\usepackage[colorlinks]{hyperref}
\hypersetup{colorlinks,breaklinks,linkcolor=darkblue,urlcolor=darkblue,anchorcolor=darkblue,citecolor=darkblue}
\usepackage{booktabs,caption}
\usepackage{multirow}
\usepackage{lscape}
\usepackage{epstopdf}
\usepackage{lineno}
\usepackage{microtype}
\newcommand{\diag}{\mathop{\mathrm{diag}}}
\newtheorem{lemma}{Lemma}
\newtheorem{theorem}{Theorem}
\newtheorem{remark}{Remark}

\newtheorem{definition}{Definition}
\newtheorem{assumption}{Assumption}
\newtheorem{corollary}{Corollary}
\DeclareMathOperator{\Tr}{Tr}
\DeclareMathOperator{\myIm}{Im}

\usepackage{soul}
\newcommand{\minew}[1]{{\color{black}{#1}}}

\journal{Transportation Research Part C: Emerging Technologies}

\begin{document}

\begin{frontmatter}



\title{{\fontfamily{lmss}\selectfont Truncated tensor Schatten $p$-norm based approach for spatiotemporal traffic data imputation with complicated missing patterns}}

\author{Tong Nie}
\author{Guoyang Qin}
\author{Jian Sun\corref{cor1}}
\ead{sunjian@tongji.edu.cn}

\address{Department of Traffic Engineering \& Key Laboratory of Road and Traffic Engineering, Ministry of Education, Tongji University, Shanghai, China}

\cortext[cor1]{Corresponding author. Address: Cao’an Road 
4800, Shanghai 201804, China}

\begin{abstract}
Rapid advances in sensor, wireless communication, cloud computing and data science have brought unprecedented amount of data to assist transportation engineers and researchers in making better decisions. However, traffic data in reality often has corrupted or incomplete values due to detector and communication malfunctions. Data imputation is thus required to ensure the effectiveness of downstream data-driven applications.
%Missing data is a pervasive problem in real-world intelligent transportation systems that impedes the effectiveness of downstream data-driven applications. 
To this end, numerous %matrix/
tensor-based methods treating the imputation problem as the low-rank tensor completion (LRTC) have been attempted in previous works. To tackle rank minimization, which is at the core of the LRTC, most of aforementioned methods utilize the tensor nuclear norm (NN) as a convex surrogate for the minimization. However, the over-relaxation issue in NN refrains it from desirable performance in practice. In this paper, we define an innovative nonconvex truncated Schatten $p$-norm for tensors (TSpN) to approximate tensor rank and impute missing spatiotemporal traffic data under the %low-rank tensor completion (LRTC) 
LRTC framework. We model traffic data into a third-order tensor structure of \textit{time intervals}$\times$\textit{locations (sensors)}$\times$\textit{days} and introduce four complicated missing patterns, including random missing and three fiber-like missing cases according to the tensor mode-$n$ fibers. Despite nonconvexity of the objective function in our model, we derive the global optimal solutions by integrating the alternating direction method of multipliers (ADMM) with generalized soft-thresholding (GST). In addition, we design a truncation rate decay strategy to deal with varying missing rate scenarios. Comprehensive experiments are finally conducted using four different types of real-world spatiotemporal traffic datasets, which demonstrate that the proposed LRTC-TSpN method performs well under various missing cases even with high rates of data loss, meanwhile outperforming other state-of-the-art tensor-based imputation models in almost all scenarios. In general, the performance of the proposed method achieves up to a 52\% improvement in mean absolute error (MAE) compared to baseline methods even when the missing ratio is as high as 90\%. Moreover, both theoretical and empirical evidences indicate that our method has robust and efficient convergence performances.

\end{abstract}


\end{frontmatter}

\section{Introduction}

In recent years, rapid advances in sensor, wireless communication, cloud computing and data science have catalyzed the development of intelligent transportation systems (ITS) and urban computing. An unprecedented amount of spatiotemporal traffic data has been generated, collected and stored by virtue of loop sensors, surveillance cameras, connected and autonomous vehicles, smart cards and smartphones with Bluetooth in distributed transportation information systems, assisting engineers and researchers in making better decisions in domains such as traffic state estimation, proactive traffic control, and demand prediction \citep{guo2019urban,du2020traffic,cipriani2021traffic,mousavizadeh2021real,rempe2022estimation}. 
%Spatiotemporal traffic data is spatially distributed time series with properties being multi-granular, heterogeneous, and dynamic. 
However, traffic data in reality often has corrupted or missing values due to detector failure, low sensor coverage and communication disruptions \citep{du2020missing}, which can adversely affect data usability and subsequent traffic management and control. Data imputation is thus required in the flow of the traffic data applications. Since replacement of missing values can, in fact, be arbitrary, some assumptions have to be made on how the data came to be missing to ensure the correct prediction of missing values \citep{sportisse2020imputation}. To this end, a large number of studies are found to leverage prior knowledge to restore unavailable information from observations. Among these studies, the low-rank tensor completion (LRTC) method, which treats data as a tensor with assumed low rank or low dimensionality, is proved to be superior, thereby stimulating extensive studies of this tensor learning-based imputation method \citep{chen2020nonconvex,chen2021scalable,deng2021graph,chen2021Autoregressive}.
As traffic data is characteristic of considerable correlations and similarities, or algebraically, high linear dependence and low rank, due to its day-to-day recurrence and neighborhood similarity, it's desirable to cast traffic data imputation as an LRTC problem on the assumption of low rank property. 
% This means that incomplete spatiotemporal data with undesirable qualities may jeopardize the effectiveness of data usage, thus exerting adverse impact on establishing proper traffic management and control strategies. Therefore, imputing plausible values of missing data effectively is a prerequisite for establishing highly trustworthy ITS. 
% Since spatiotemporal traffic data has intrinsic correlations and similarities (e.g., day-to-day recurrence and neighborhood similarity), which are usually manifested as multi-dimensional linear dependence in algebra, filling the missing value with a low-rank or low-dimensional solution is justified to be rational. 

Organizing traffic data as tensors, a high-order generalization of matrix, has found favor with utilization and capture of multi-modal correlations (e.g., time of day, day of week and location mode) simultaneously and processing of information on each mode of tensor efficiently, which enables effective and efficient data imputation. However, extending matrix to tensor makes the definition of tensor rank, which is essential to the LRTC method, other than favorable. Taking two most prevailing definitions of tensor rank, i.e., CP-rank and $n$-rank (a.k.a., tucker rank) \citep{kolda2009tensor}, for instance, 
% In order to exploit this data structure to establish an LRTC model, a natural thought is to extend the matrix rank to high-order tensor case. 
% Tensor rank is pseudo rank established on the extended notion of matrix rank, and two most prevailing definitions of tensor rank are CP-rank and $n$-rank (a.k.a., tucker rank) \citep{kolda2009tensor}. 
%However, unlike the matrix case, these tensor rank definitions are intractable, as 
the determination of CP-rank is verified to be NP-hard and the $n$-rank is a difficult non-convex problem \citep{haastad1990tensor}. The intractability of these definitions has led researchers to define various convex surrogates to approximate the original problem. Having been proved to be the tightest convex envelope of nonconvex rank \citep{liu2013tensor}, tensor nuclear norm (NN) becomes the most prominent tucker rank approximation. However, it is still prone to over-shrinkage and deviation from the actual rank values. To overcome the drawback, many researchers have recently designed novel nonconvex rank functions to substitute for NN and have shown the strength of nonconvex rank approximation \citep{yao2018large,chen2020nonconvex,yang2021novel}. Among these works, \cite{chen2020nonconvex} adopts a truncated nuclear norm (TNN) for tensors to handle the traffic data imputation task, and introduce a truncation rate parameter to globally control the numbers of contributing singular values, named LRTC-TNN. Despite the nonconvexity of the objective function of TNN, its truncated singular values are still subject to the over-relaxation issue, which can refrain the method from capturing the real low-rank mode. To further overcome the drawbacks of these NN-based imputation method, studies from computer vision and information communities including \cite{nie2012low},  \cite{nie2012robust},  \cite{feng2016image}, and \cite{chen2018human} propose a new matrix Schatten $p$-norm (SpN) and its truncated version to restore images affected by noises and randomly missing pixels, where SpN is demonstrated to be a better nonconvex rank approximation function than NN. Nevertheless, their matrix SpN-based methods can not be applied to traffic data imputation tasks directly as they are mainly tailored for RGB images with fixed 'color' channel. 

Moreover, unlike visual data imputation that mainly faces completely random pixels loss \citep{liu2013tensor}, spatiotemporal traffic data often suffers from more complex missing patterns with spatial-temporal correlations. In real traffic datasets, more "clustered" and "structured" missing situations are observed and can be basically classified into three scenarios: (1) one or several days' records of some sensors are all missing; (2) a specific duration of records of all sensors are missing; (3) records are missing due to out-of-service detectors in regular hours everyday. These extreme missing scenarios complicate the data imputation and may disable the low-rank matrix completion based method as a matrix can never be recovered under low-rank assumptions if there are whole rows (or columns) of a matrix missing \citep{tan2013tensor}, hence the huge challenge of recovering such structured missing patterns and evaluating corresponding impacts of dataset features on imputation accuracy. To tackle these existing issues for tensor completion-based traffic data imputation, we introduce a novel truncated tensor Schatten $p$-norm (TSpN), which can be viewed as extension of TNN as well as matrix truncated SpN, to conduct an accurate low-rank tensor completion (LRTC-TSpN) for spatiotemporal traffic data. The primary contributions of this work are summarized as follows: 
\begin{itemize}
\item We define a novel TSpN for tensors that jointly absorbs the merits of both TNN and SpN. By keeping the principal singular values unchanged to avoid over-shrinkage and conducting nonconvex regularization on remaining singular values, the proposed TSpN stands for a better tensor rank surrogate with more accuracy and flexibility, which can substitute for TNN in this sense.
\item We deduce the numerical solution to the nonconvex LRTC-TSpN problem by combining the alternating direction method of multipliers (ADMM) with generalized soft-thresholding (GST). We also theoretically prove that the proposed iterative algorithm guarantees to convergence to the global optimum. To the best of our knowledge, we are the first to propose tensor truncated Schatten $p$-norm and corresponding solving algorithms for its minimization problem. 
\item To deal with scenarios with varying missing rates effectively, we design a simple truncation rate decay strategy to automatically assign a suitable truncation rate parameter for model. This simple but enlightening technique compensates for the deficiency of previous methods when dealing with diverse missing ratios, especially for the extreme loss cases.
\item To fully evaluate the generalization ability of our method, we introduce three highly structured fiber-like missing patterns according to tensor mode-$n$ fibers. Furthermore, we conduct comprehensive experiments on four different types of spatiotemporal traffic datasets with multiple missing scenarios compared with some state-of-the-art models, and the results clearly reveals the superiority of our method.
\end{itemize} 

The rest of this article is organized as follows. Section \ref{Literature review} reviews related works focused on traffic data imputation. Section \ref{Notations and Problem Definitions} gives some basic notions of tensor and defines the imputation problem. In \minew{Section} \ref{methodology}, we formulate the LRTC-TSpN model and derive our solving algorithm. In \minew{Section} \ref{experiments}, we conduct experiments on different datasets and present discussions about the results. Section \ref{conclusions} concludes this paper and gives prospects of future research directions.

\section{Literature review}\label{Literature review}
We summarizes some of the existing studies of traffic data imputation according to the types of model, missing patterns and types of test data in \minew{Table} \ref{review}, which can be come down to three major classes according to the form of models: matrix-based, tensor-based, and deep learning methods.

\begin{sidewaystable}[hp]
  \centering
  \caption{Summary of existing traffic data imputation methods}
  \renewcommand\arraystretch{1.5}
  %\scalebox{0.6}{
  \resizebox{\textwidth}{!}{
    \begin{tabular}{cc|lp{33.835em}l}
    \toprule
    \multicolumn{2}{c}{\multirow{2}[2]{*}{Type of models}} & \multicolumn{1}{c}{\multirow{2}[2]{*}{Method}} & \multicolumn{1}{c}{\multirow{2}[2]{*}{Missing patterns}} & \multicolumn{1}{c}{\multirow{2}[2]{*}{Data type}} \\
    \multicolumn{2}{c}{} &       & \multicolumn{1}{c}{} &  \\
    \midrule
    \multicolumn{1}{c|}{\multirow{11}[6]{*}{Matrix based}} & \multicolumn{1}{c}{\multirow{5}[2]{*}{Matrix factorization}} & Sparsity regularized matrix factorization \citep{zhang2009spatio} & Random missing, row loss, column loss & OD \\
    \multicolumn{1}{c|}{} & \multicolumn{1}{c}{} & LSTM and Graph Laplacian regularized matrix factorization \citep{yang2021real}& \multicolumn{1}{l}{Random missing, fiber mode-1 missing} & Speed \\
    \multicolumn{1}{c|}{} & \multicolumn{1}{c}{} & Joint matrix factorization \citep{jia2021missing} & \multicolumn{1}{l}{Random missing} & Congestion data \\
    \multicolumn{1}{c|}{} & \multicolumn{1}{c}{} & Bayesian temporal matrix factorization \citep{chen2021bayesian}& \multicolumn{1}{l}{Random missing, fiber mode-0 missing} & Speed, passenger flow \\
    \multicolumn{1}{c|}{} & \multicolumn{1}{c}{} & Spatial-temporal regularized matrix factorization \citep{wang2018traffic} & Random loss, time-mode loss, space-mode loss & OD, speed, volume \\
\cmidrule{2-5}    \multicolumn{1}{c|}{} & \multicolumn{1}{c}{\multirow{2}[2]{*}{Matrix completion}} & Fixed point continuation with approximate SVD  (nuclear norm) \citep{asif2016matrix} & \multicolumn{1}{l}{Random missing} & Speed \\
    \multicolumn{1}{c|}{} & \multicolumn{1}{c}{} & Lagrangian multiplier method (Schatten p-norm) \cite{yu2020urban} & \multicolumn{1}{l}{Random missing} & Speed \\
\cmidrule{2-5}    \multicolumn{1}{c|}{} & \multicolumn{1}{c}{\multirow{3}[2]{*}{Statistical learning}} & Modified k-nearest neighbor method \citep{tak2016data} & \multicolumn{1}{l}{Random missing, groups of random missing} & Speed \\
    \multicolumn{1}{c|}{} & \multicolumn{1}{c}{} & Probabilistic/Bayesian based PCA \citep{qu2009ppca,li2013efficient} & Random missing, groups of random missing & Volume \\
    \multicolumn{1}{c|}{} & \multicolumn{1}{c}{} & Multi-output Gaussian processes \citep{rodrigues2018multi} & Random missing, missing with temporal dependency & Speed \\
    \midrule
    \multicolumn{1}{c|}{\multirow{14}[4]{*}{Tensor based}} & \multicolumn{1}{c}{\multirow{9}[2]{*}{Tensor factorization}} & Bayesian tensor factorization \citep{chen2019bayesian} & \multicolumn{1}{l}{Random missing, fiber mode-0 missing} & Speed \\
    \multicolumn{1}{c|}{} & \multicolumn{1}{c}{} & Tucker decomposition \citep{tan2013tensor} & \multicolumn{1}{l}{Random missing, extreme missing} & Volume \\
    \multicolumn{1}{c|}{} & \multicolumn{1}{c}{} & SVD-combined tucker decomposition \citep{chen2018spatial} & \multicolumn{1}{l}{Random missing, fiber mode-0 missing} & Speed \\
    \multicolumn{1}{c|}{} & \multicolumn{1}{c}{} & Spatial-temporal regularized CP factorization \citep{zhou2015spatio} & \multicolumn{1}{l}{Time-mdoe loss, space-mode loss} & OD \\
    \multicolumn{1}{c|}{} & \multicolumn{1}{c}{} & Tensor train decomposition \citep{zhang2021tensor} & \multicolumn{1}{l}{Random missing} & OD,speed, passenger flow \\
    \multicolumn{1}{c|}{} & \multicolumn{1}{c}{} & Iterative tensor decomposition \citep{zhang2019missing} & \multicolumn{1}{l}{Random missing, extreme missing} & Volume \\
    \multicolumn{1}{c|}{} & \multicolumn{1}{c}{} & Dynamic tensor decomposition \citep{tan2016short} & \multicolumn{1}{l}{Random missing, extreme missing} & Volume \\
    \multicolumn{1}{c|}{} & \multicolumn{1}{c}{} & Soft-thresholding of Tucker core \citep{goulart2017traffic} & \multicolumn{1}{l}{Fiber mode-1 missing} & Speed \\
    \multicolumn{1}{c|}{} & \multicolumn{1}{c}{} & Graph-tensor singular value decompositions \citep{deng2021graph} & Random missing, time-mode loss, space-mode loss & Volume \\
\cmidrule{2-5}    \multicolumn{1}{c|}{} & \multicolumn{1}{c}{\multirow{4}[2]{*}{Tensor completion}} & Truncated tensor nuclear norm minimization \citep{chen2020nonconvex} & \multicolumn{1}{l}{Random missing, fiber mode-0 missing} & Speed, occupancy, passenger flow \\
    \multicolumn{1}{c|}{} & \multicolumn{1}{c}{} & Truncated tensor nuclear norm with autoregressive regularization \citep{chen2021Autoregressive} & Random missing, fiber mode-0 missing, fiber mode-1 missing & \multicolumn{1}{p{14.835em}}{Speed, passenger flow, volume} \\
    \multicolumn{1}{c|}{} & \multicolumn{1}{c}{} & Tensor nuclear norm with linear unitary transformation \citep{chen2021scalable} & \multicolumn{1}{l}{Random missing, fiber mode-0 missing} & Speed \\
    \multicolumn{1}{c|}{} & \multicolumn{1}{c}{} & Tensor nuclear norm minimization \citep{ran2016tensor} & \multicolumn{1}{l}{Random missing, extreme missing} & Volume \\
    \midrule
    \multicolumn{1}{c|}{\multirow{5}[6]{*}{Deep learning}} & \multicolumn{1}{c}{\multirow{2}[2]{*}{Generative neural network}} & Generative adversarial network \citep{zhang2021gated} & \multicolumn{1}{l}{Random missing, groups of random missing} & Speed \\
    \multicolumn{1}{c|}{} & \multicolumn{1}{c}{} & Variational autoencoder \citep{boquet2019missing} & \multicolumn{1}{l}{Random missing, groups of random missing} & Speed \\
\cmidrule{2-5}    \multicolumn{1}{c|}{} & \multicolumn{1}{c}{\multirow{2}[2]{*}{Feedforward neural network}} & Denoising autoencoder \citep{duan2016efficient} & \multicolumn{1}{l}{Random missing} & Volume \\
    \multicolumn{1}{c|}{} & \multicolumn{1}{c}{} & Convolutional neural network \citep{zhuang2019innovative} & \multicolumn{1}{l}{Random missing} & Volume \\
\cmidrule{2-5}    \multicolumn{1}{c|}{} & \multicolumn{1}{c}{Multimodal neural network} & Graph convolutional bidirectional recurrent neural network \citep{zhang2021customized} & \multicolumn{1}{l}{Random missing} & Speed, volume \\
    \midrule
    \multicolumn{2}{c|}{\multirow{5}[2]{*}{Hybrid approach}} & Prophet and random forest \citep{li2020spatiotemporal} & Random missing, groups of random missing, fiber mode-0 missing & Volume \\
    \multicolumn{2}{c|}{} & Rough set and fuzzy neural network \citep{tang2021missing} & \multicolumn{1}{l}{Random missing, groups of random missing} & Volume \\
    \multicolumn{2}{c|}{} & Fuzzy c-means and genetic algorithm \citep{tang2015hybrid} & \multicolumn{1}{l}{Random missing} & Volume \\
    \multicolumn{2}{c|}{} & Matrix completion and adaptive KNN \citep{chen2017ensemble} & Random missing, groups of random missing & Volume \\
    \multicolumn{2}{c|}{} & Matrix completion and graph regularized self-representation \citep{chen2018graph}& Random missing, groups of random missing, mixed missing & Volume \\
    \bottomrule
    \end{tabular}}%}%
  \label{review}%
\end{sidewaystable}%

Conducting imputation using matrix methods needs to organize traffic data into a $time~of~days \times locations$ matrix, and usually utilize the low matrix rank assumptions \citep{zhang2009spatio,asif2016matrix}. Some studies use spatial and temporal consistency to combine with matrix factorization, which shows better performance than traditional methods. For example, \cite{wang2018traffic} models the Toeplitz matrix as a temporal regularizer, Graph Laplacian matrix as a spatial regularizer and incorporates them into the objective function of matirx factorization. Moreover, other researchers employ statistical learning skills to learn or fit a distribution from data: \cite{qu2009ppca} propose a probabilistic principal component analysis to improve the original PCA models; \cite{rodrigues2018multi} use multi-output Gaussian process to consider the uncertainty of data. However, matrix pattern only performs well in random missing and low loss ratio situations since matrix models cannot fully leverage the spatial and temporal information simultaneously and capture multi-dimensional correlations.

Deep learning methods mainly rely on historical data to learn the latent representations, thereby reconstructing or generating recovered data samples. \cite{duan2016efficient}
firstly exploits a denoising auto-encoder models and design an efficient algorithm to train the neural network hierarchically; \cite{zhang2021customized} devises an ingenious network architecture named graph convolutional bidirectional recurrent neural network to achieve network-scale online data imputation. Despite these approaches display competitive performance, however, large amount of high-quality historical datasets are essential for model training. And tuning the hyper-parameters of neural network is considered to be laborious.

Recent years, tensor-based approaches have been extensively researched for modeling and analyzing traffic data \citep{tan2016short,asif2016matrix,chen2019bayesian}. For the first time, \cite{tan2013tensor} models traffic data into a fourth-order tensor structure of $link \times week \times day \times hour$ and develops a tucker decomposition based imputation method; \cite{chen2019bayesian} introduce a Bayesian Gaussian CP decomposition to fill missing entries. These tensor factorization methods usually require predefined Tucker rank or CP rank, however, in reality, obtaining a accurate tensor rank from restricted observations has proven difficult \citep{goulart2017traffic}. LRTC-based approach has been adopted to cope with different missing cases more recently and shows better imputation performance than other matrix/tensor factorization methods \citep{ran2016tensor,chen2020nonconvex,chen2021Autoregressive}. This kind of \minew{model} often \minew{takes} advantage of the prior knowledge that the spatiotemporal traffic data is supposed to be low-rank, and then \minew{introduces} a tensor rank surrogate (e.g., nuclear norm) to conduct minimization of rank function. 

In general, according to \minew{Table} \ref{review}, previous spatiotemporal traffic data imputation methods mainly pay attention to limited or specific data loss situations, which are insufficient for modeling more complex real-world cases, while at the same time, the majority of them conduct imputation task solely on particular data \minew{types}. Additionally, existing LRTC-based approaches are dependent on NN or TNN minimization, which has been illuminated to be biased (more discussions will be showed in \minew{Section} \ref{Notations and Problem Definitions}). Therefore, it is essential to develop a more accurate and universal tensor rank surrogate that overcomes the deficiency of NN and TNN, and consequently, can be deployed to impute traffic data under multiple missing patterns with various kinds of spatiotemporal data.

\section{Notations and problem definitions}
\label{Notations and Problem Definitions}

In this section, we first introduce some basic notations and preliminaries of tensor. After that, we organize the spatiotemporal traffic data into a tensor structure according to its properties. Then we define the tensor completion based imputation model and complicated missing patterns in the context of spatiotemporal traffic data.

\subsection{Notations}
In  this work, we adopt the same notations as Kolda and Bader’s review on tensor \citep{kolda2009tensor}. Scalars are denoted by lowercase letters, e.g., $a$,  vectors are denoted by boldface lowercase letters, e.g., $\mathbf{a}\in\mathbb{R}^{M}$, and matrices are represented by boldface \minew{capital} letters, e.g., $\mathbf{A}\in\mathbb{R}^{M\times N}$. \minew{An} $N$th-order tensor is denoted by calligraphic letter, e.g., $\mathcal{A}\in\mathbb{R}^{I_1\times I_2\times\cdots\times I_N}$, whose $\left(i_1,i_2,\cdots,i_N\right)-th$ entry is $a_{i_1i_2\cdots i_N}$. A fiber is defined by fixing every index but one, which is the higher-order analogue of matrix columns and rows. The mode-$n$ fibers $a_{i_1\cdots i_{n-1}:{\ i}_{n+1}{\ \cdots i}_N}$ can be obtained by fixing the dims of  $i_1,i_2\cdots i_N$ except $i_n$.
The inner product of two $N$th-order tensors is defined by $\left\langle\mathcal{A},\mathcal{B}\right\rangle=\sum_{i_1,i_2,\cdots,i_N}a_{i_1,i_2,\cdots,i_N}b_{i_1,i_2,\cdots,i_N}$, and the corresponding tensor Frobenius norm is $\parallel\mathcal{A}\parallel_F=\sqrt{\left\langle\mathcal{A},\mathcal{A}\right\rangle}$.
The mode-$n$ unfolding (a.k.a, matricization \minew{or} flattening) of a tensor $\mathcal{A}\in\mathbb{R}^{I_1\times I_2\times\cdots\times I_N}$ is denoted by $\mathbf{A}_{(n)}=\operatorname{Unfold}_n(\mathcal{A})$, where $\mathbf{A}_{(n)}\in \mathbb{R}^{I_n\times \prod_{i\neq n}I_i}$. Specifically, the tensor entry $\left(i_1,i_2,\cdots,i_N\right)$ is mapped to matrix entry $\left(i_n,j\right)$, where
\begin{equation} \notag
    j=1+\sum_{\substack{k=1 \\ k\neq n}}^N\left(i_k-1\right)J_k,\quad with \, J_k=\prod_{\substack{m=1 \\ m\neq n}}^{k-1}I_m.
\end{equation}

Meanwhile, the inverse operation of unfold is defined as $\mathcal{A}= \operatorname{Fold}_n(\mathbf{A}_{(n)})$, which restores a matrix to the corresponding tensor in the $n$-th mode.

\subsection{Traffic data imputation and tensor completion}
\label{Traffic Data Imputation and Tensor Completion}
Spatiotemporal traffic data (e.g., volume, speed, ODs) has intrinsic multi-mode correlations and similarities, which are the basic prerequisite for accurate data recovery \citep{qu2009ppca}. In temporal domain, influenced by users' daily routine, there exists strong periodic and recurrent pattern, and these traffic time series often \minew{show} stability over time \citep{xie2018accurate}; in spatial domain, local similarities are ubiquitous due to the inherent upstream-downstream relationship between these adjacent locations. In order to intuitively display these particular features of spatiotemporal traffic data, we give some visualizations using a traffic volume dataset (more details will be discussed in \minew{Section} \ref{experiments}) in \minew{Figure} \ref{traffic data}. 
\begin{figure}[!hb]
\centering
\subfigure['Days' mode correlations]{
\centering
\includegraphics[scale=0.53]{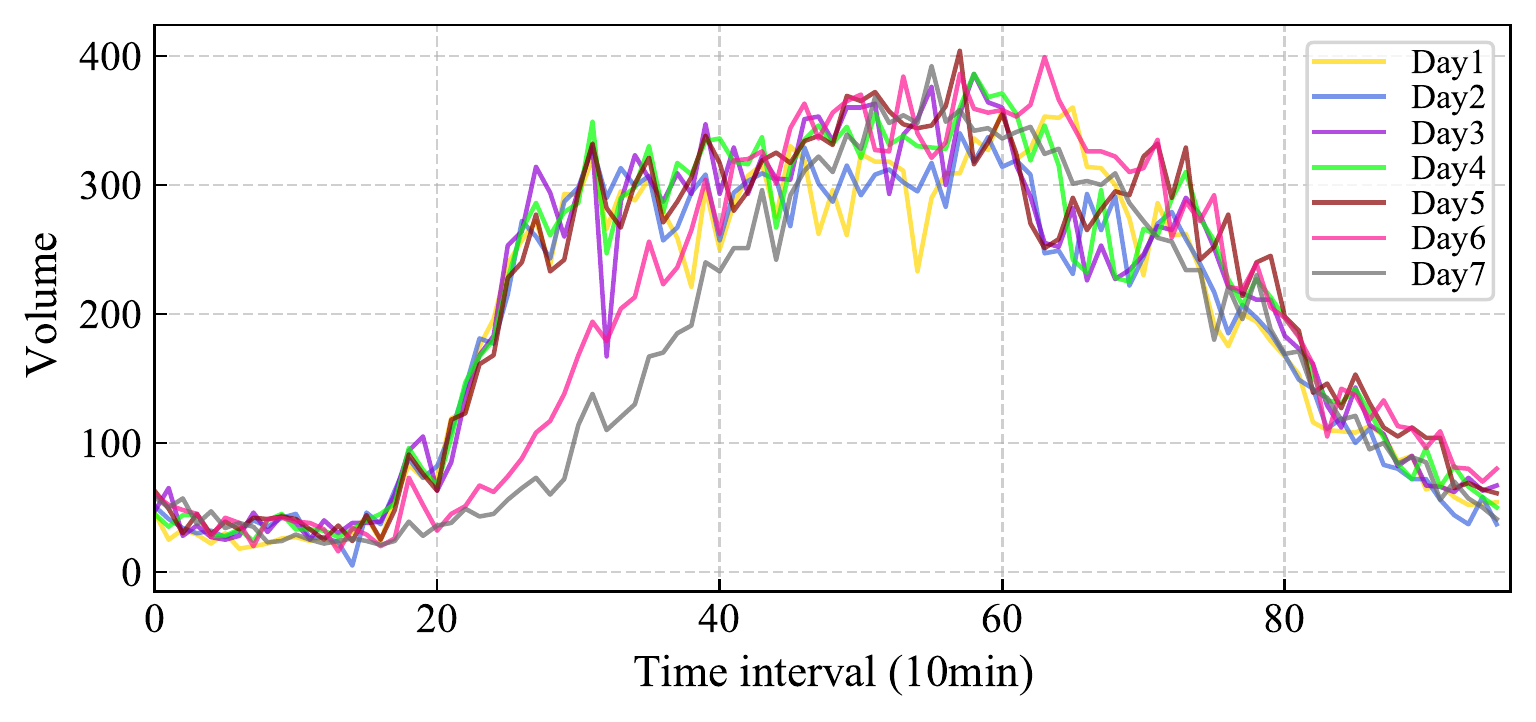}
}
\subfigure['Links' mode correlations]{
\centering
\includegraphics[scale=0.53]{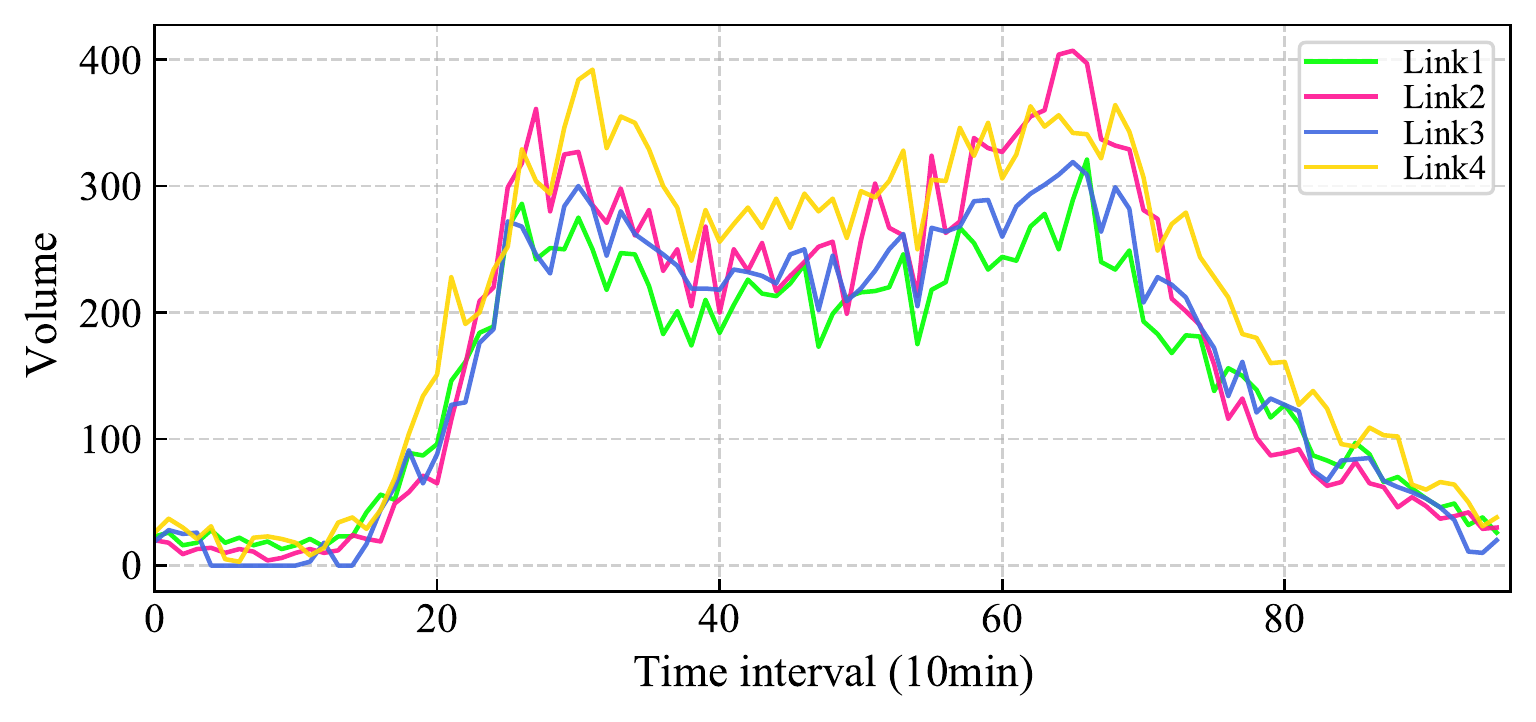}
}
\subfigure['Intervals' mode correlations]{
\centering
\includegraphics[scale=0.53]{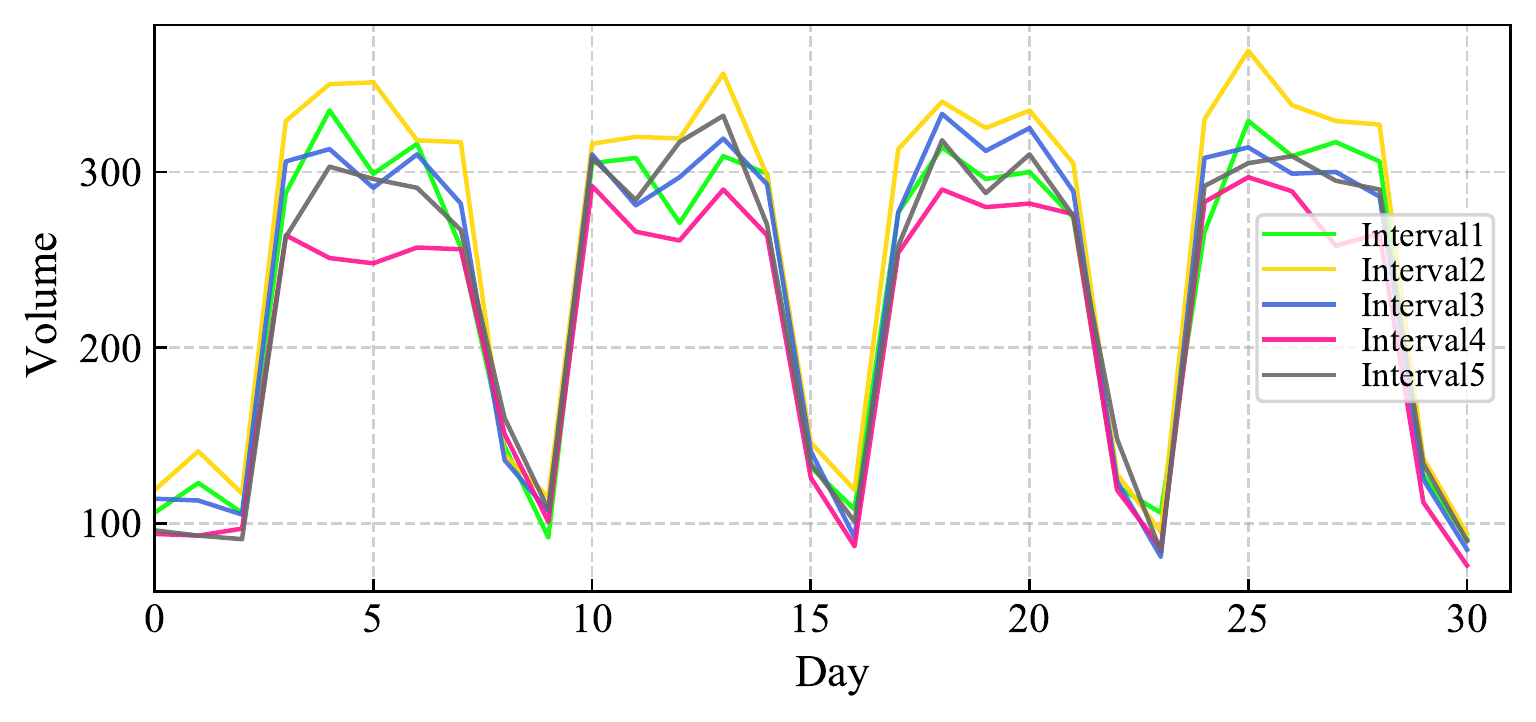}
}
\subfigure[Low-rank properties]{
\centering
\includegraphics[scale=0.53]{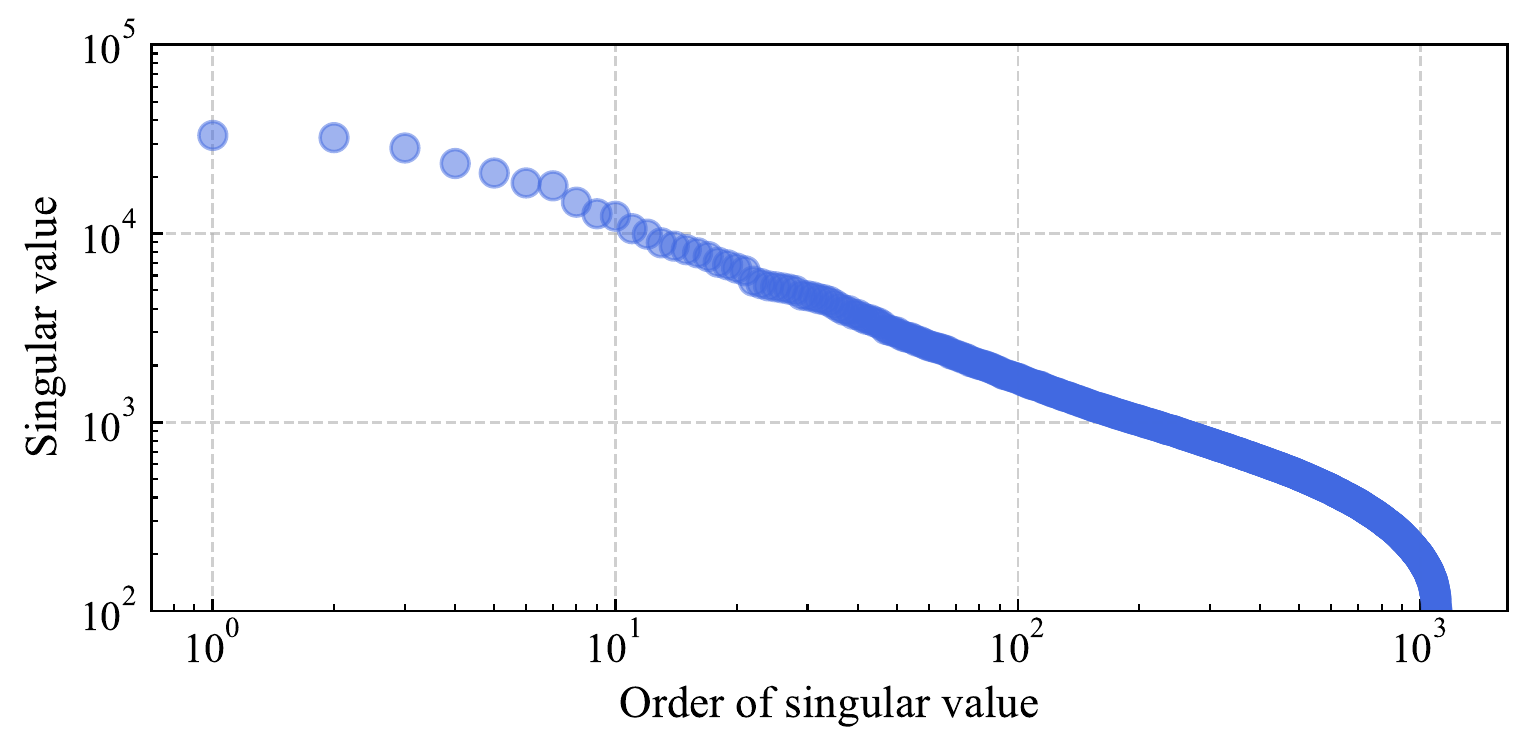}
}
\caption{Visualization of spatiotemporal correlations and low-rank properties (a) Traffic volume for a week of a randomly selected sensor. (b) Traffic volume within a day of four randomly selected sensors. (c) Traffic volume for a month of five consecutive time intervals. (d) Log-log plot of singular values.}
\label{traffic data}
\end{figure}
As shown in \minew{Figure} \ref{traffic data} (a), (b) and (c), spatiotemporal traffic data has remarkable multi-dimension correlations, e.g., day-to-day recurrence, neighborhood similarity, within-day regularity. 
Apart from the multi-mode similarities and correlations, another noticeable information in \minew{Figure} \ref{traffic data} (d) is that the singular values are dominated by only a few large ones, which is consistent with our 'low-rank' assumptions. Information loss along each mode of 'days', 'links' and 'intervals' can be possibly recovered by exploiting hidden correlations of other modes as well as imposing a low-rank restriction.

These aforementioned attributes of spatiotemporal traffic data \minew{enable} us to leverage low-rank tensor models to restore missing data effectively. In this paper, we organize the multivariate traffic time series as a third-order tensor structure, i.e. \textit{time intervals}$\times$\textit{locations (sensors)}$\times$\textit{days}. This three-dimensional data structure simultaneously captures the integrated spatial-temporal information, thus making it more efficient to impute missing values. 

Low-rank tensor recovery problem has been widely studied recently in machine learning, computer vision, and transportation data science communities to cope with image denoising, signal processing, and estimating missing values in visual and traffic data, etc \citep{tan2013tensor,gu2014weighted,chen2018human}. This tensor learning model can be stated as : given a linear map $\mathcal{G}:\mathbb{R}^{I_1\times I_2\times\cdots\times I_N}\rightarrow\mathbb{R}^q\:\text{with}\:q\le\prod_{i=1}^{N}I_i$ and given $\mathbf{b}\in\mathbb{R}^q$, determine the tensor $\mathcal{X}$ that fulfills the linear measurements $\mathcal{G}(\mathcal{X})=\mathbf{b}$ and minimizes a predefined function of the tensor rank.
Specially, when a subset of the entries of the tensor is given as the partially observed data, and the rest entries are treated as missing values to be imputed, this model becomes the standard tensor completion problem:
\minew{
\begin{equation}
    \begin{aligned}
    \min _{\mathcal{X}}~&\operatorname{f_{rank}}(\mathcal{X}) \\ \text { s.t.}~& \mathcal{P}\odot\mathcal{X}=\mathcal{P}\odot\mathcal{T}, \\
    \end{aligned}
    \label{equ:TC}
\end{equation}
}
where $f_{rank}$ is the rank function defined for tensor, $\mathcal{X}$ is the target spatiotemporal traffic tensor needed to be obtained, $\mathcal{T}$ is the observed tensor with missing entries, \minew{$ \odot$ denotes the Hadamard product, $\mathcal{P}$ is the mask tensor indicating the observed entries, given by:

\begin{equation}
    \mathcal{P}=\left\{\begin{array}{l}1,~\text{if } t_{ijk} \text{ is observed}, \\
    0,~\text{otherwise}. \\    
    \end{array}\right. \\
\end{equation}}

% $\Omega=\{(i,j,k)|\text{when} ~t_{ijk} ~\text{is observed}\}$ is the set of observed entries.

\minew{And we let $\mathcal{P}^{\perp}$ denote the complementary tensor of $\mathcal{P}$, satisfying $\mathcal{P}\odot\mathcal{X}+\mathcal{P}^{\perp}\odot\mathcal{X}=\mathcal{X}$}. Our goal is to design a more universal and flexible rank function for tensors in \minew{Eq.} \eqref{equ:TC} to overcome the above deficiencies of NN/TNN, so that we can better exploit this low-rank model to impute missing data in traffic systems more precisely and cope with various data loss cases.

\subsection{Complicated missing patterns}
In this work, we consider four kinds of data missing patterns, i.e., element-wise random missing (RM) and structured $n$-th mode fiber-like missing (FM-$n$). The majority of experiments have conducted mainly on random missing pattern in previous studies using either matrix or tensor based approach \citep{qu2009ppca,chen2017ensemble,asif2016matrix,li2013efficient,duan2016efficient,zhou2015spatio}. However, in real-world ITS, this case is not common to see because the input data is statistically aggregated within a specific time interval, and it is hard to detect and identify occurrence of this instantaneous loss (e.g., transmission failure) under the circumstances \citep{zhang2019missing}. More recently, researchers have tried some different non-random missing or so-called clustered missing patterns that are more difficult to perform imputation \citep{tan2013tensor,ran2016tensor,li2020spatiotemporal,yang2021real,chen2021Autoregressive}. Most of their works mainly discuss these more intricate modes in temporal domain, e.g., continuous and complete (or partially complete) missing within several days (or intervals) for a specific location, which can be successfully recovered by their time series attributes. Apart from this 'temporal missing', there also exits 'spatial missing' pattern, e.g., all detectors are unavailable for a given time span due to the power outage.

\begin{figure}[!htb]
\centering
\subfigure[Schematic diagram of random missing and fiber mode-$n$ missing]{
\centering
\includegraphics[scale=0.6]{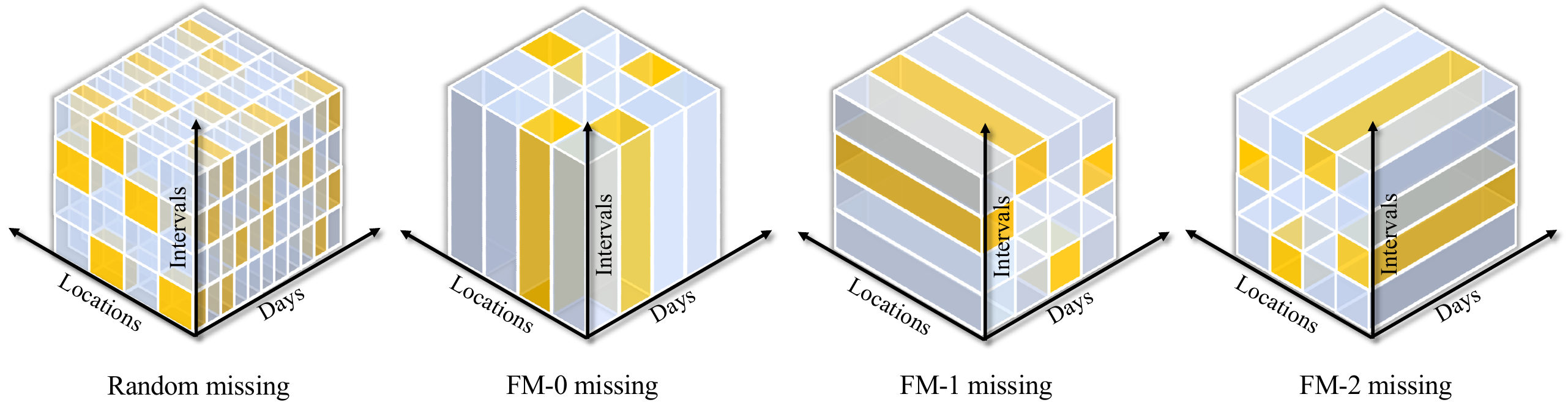}
}
\subfigure[Example time-space matrices]{
\centering
\includegraphics[scale=0.6]{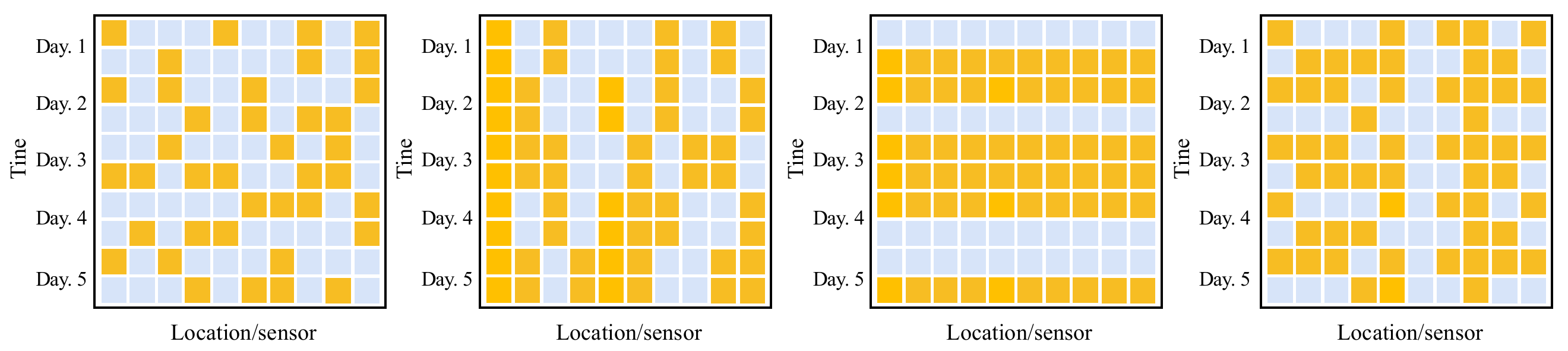}
}
\caption{Complicated missing patterns. (a) The first tensor denotes all of its channels (entries) and each of the last three tensors is displayed by its mode-$n$ fibers. (b) Example matrices of each missing scenario generating by converting third-order tensors into spatiotemporal traffic matrices of $time \times location$. Note that blue channels (fibers) are observations, while yellow ones denote missing values.}
\label{FM}
\end{figure}

To summarized these cases, we give a general definition of complicated missing patterns according to the structure of third-order tensor described in \minew{Section} \ref{Traffic Data Imputation and Tensor Completion}, as can been seen in \minew{Figure} \ref{FM}. In the light of these schematics, random missing case is completely element-wise loss as defined in previous works, and structured fiber mode-$n$ missing is generated through the two-by-two combinations of tensor mode-$n$ fibers. The latter is divided into three conditions that can be described as : (1) 'Intervals' mode fiber-like missing (FM-0), which illustrates a temporal missing pattern, is caused by adverse weather, breakdown of wireless connections or apparatus maintenance; (2) 'Locations' mode fiber-like missing (FM-1), which denotes a spatial missing pattern, can be explained by lack of electricity for successive sensors or malfunction of Internet Data Center; (3) 'Days' mode fiber-like missing (FM-2) illuminates a spatial-temporal mixture missing situation that they are offline (do not operate) at regular time intervals everyday for specific sensors. Note that in order to test our models' performance without loss of generality, we randomly select some fibers and artificially generate these missing scenarios, while ones can gather these blank entries (partially) together to simulate more realistic scenes.

\section{Methodology}\label{methodology}

In this section, we first define a new tensor truncated Schatten $p$-norm and give some explanations about this generalized norm. Since then, we incorporate the TSpN into the LRTC function to establish traffic data imputation model. Despite the minimization is a hard nonconvex problem, we individually solve each subproblem under the framework of ADMM and theoretically prove its global optimality and convergence.
Specifically, the $\mathbf{M}$-Subproblem is solved by GST, and the closed-form solution to $\mathcal{X}$-Subproblem is obtained by gradient method, respectively.

\subsection{Model formulation}
Schatten $p$-norm for matrix and tensor have been studied in prior research. We first briefly review two basic definitions of the Schatten $p$-norm in \minew{Definition} \ref{definition1} and \minew{Definition} \ref{definition2}. Then in \minew{Definition} \ref{definition3}, we introduce the formulation of proposed TSpN in detail.
\begin{definition}
\label{definition1}
(Matrix Schatten p-norm \citep{nie2012low}). The extended Schatten p-norm $(0<p<\infty)$ of a matrix $\mathbf{X}\in\mathbb{R}^{n\times m}$ is defined as :
\begin{equation}
    \Vert \mathbf{X}\Vert_{S_p}=\left(\sum_{i=1}^{\min\{n,m\}}\sigma_i^p \right)^{\frac{1}{p}}=\left(\Tr\left(\left(\mathbf{X}^\mathsf{T}\mathbf{X}\right)^{\frac{p}{2}}\right)\right)^{\frac{1}{p}},
    \label{sp_matrix}
\end{equation}
where $\sigma^i$ is the $i$-th singular value of $\mathbf{X}$, $\Tr(\cdot)$ signifies the trace operation.
\end{definition}
When $p=1$, the Schatten 1-norm becomes the well-known trace norm(a.k.a, nuclear norm); When p$\rightarrow$0, the Schatten 0-norm is exactly the rank of matrix.

\begin{definition}
\label{definition2}
(Tensor Schatten p-norm \citep{gao2020robust}). For a $N$-th mode tensor $\mathcal{X}\in\mathbb{R}^{I_1\times I_2\times\cdots\times I_N}$, its Schatten p-norm is defined as follow :
\begin{equation}
    \Vert \mathcal{X}\Vert_{S_p}=\left(\sum_{k=1}^{N}\alpha_k\Vert \mathbf{X}_{\left(k\right)}\Vert_{S_p}^p\right)^{\frac{1}{p}},
    \label{sp_tensor}
\end{equation}
where $\alpha_k$ satisfies $\alpha_k\geq0$ and $\Sigma_{k=1}^N\alpha_k=1, 0<p<1$.
\end{definition}
In fact, just like the other tensor norm, e.g., NN for tensor in \citep{liu2013tensor}, this definition is the weighted sum of the matrix Schatten $p$-norm of each unfolding matrix, which in essence is a pseudo-norm.

\begin{definition}
\label{definition3}
(Truncated Tensor Schatten p-norm) : Motivated by the truncated tensor nuclear norm proposed in \cite{chen2020nonconvex}, we give the definition of Truncated tensor Schatten p-norm:
\begin{equation}
    \Vert \mathcal{X}\Vert_{\theta,S_p}=\left(\sum_{k=1}^{N}\alpha_k\Vert \mathbf{X}_{\left(k\right)}\Vert_{\theta,S_p}^p\right)^{\frac{1}{p}}=
    \left(\sum_{k=1}^N\alpha_k\left(\sum_{i=r_k+1}^{\min\{I_k,\prod_{i\neq k} I_i\}}\sigma_i^p\left(\mathbf{X}_{\left(k\right)}\right)\right)\right)^\frac{1}{p},
    \label{t_sp_tensor}
\end{equation}
the truncation for each unfolding is assigned separately by a global truncation rate:
\begin{equation}\label{truncation rate}
    r_k=\lceil\theta\cdot\min\{I_k,\prod_{i\neq k} I_i\}\rceil,\quad\forall k \in \{1,2,\ldots,N\},
\end{equation}
where $0\leq\theta\leq1$, is the truncation rate, $\lceil\cdot\rceil$ means the smallest integer that is no less than the target value, and the singular values $\sigma_i$ are in non-ascending order.
\end{definition}
Under this definition, only the \textbf{smallest} $(\min\{I_k,\prod_{i\neq k} I_i\}-r)$ singular values contribute to the Schatten $p$-norm. When optimizing this tensor rank function, the largest $r$ singular values remain the same and merely minimizing those smaller ones. The reason for truncation operation is that the larger singular values convey the primary information, like the periodic and major trend components of traffic flow, which represent the low-rank parts; and the smaller ones can be treated as the irregular fluctuations caused by local stochastic events, which make the data not strictly satisfies the low-rank pattern. Minimizing the residual parts instead of treating all values equally resembles the idea of local non-parametric regression which has been widely adopted in traffic data imputation and prediction applications, and it shows competitive performance in comparison with global regression methods \citep{habtemichael2016short,chen2017ensemble,zhang2009spatio}. From another point of view, \cite{li2020spatiotemporal} justifies the integration of recurrent temporal parts and residual spatial parts to fill missing values. The key idea of the truncated tensor Schatten $p$-norm is that the $p$ value in the norm can be fine-tuned to deal with different rank components, which enables this model to be applied to various types of data sources flexibly. 

Moreover, the SpN serves as a tighter rank surrogate than NN when $p\to0$ due to its non-convex property. To give a straightforward illustration of the SpN's superiority over NN, we compare the matrix rank ($\mathcal{L}_0$ norm) with different rank surrogates. As shown in \minew{Figure} \ref{norm_comparison}, the SpN is always closer to the $\mathcal{L}_0$ norm (rank function) than the $\mathcal{L}_1$ norm (NN) so that it becomes a better approximation for rank. This illustrates that by setting different $p$ values, we are capable of balancing the models' accuracy and efficiency. 
\begin{figure}[!ht]
  \centering
  \includegraphics[scale=0.7]{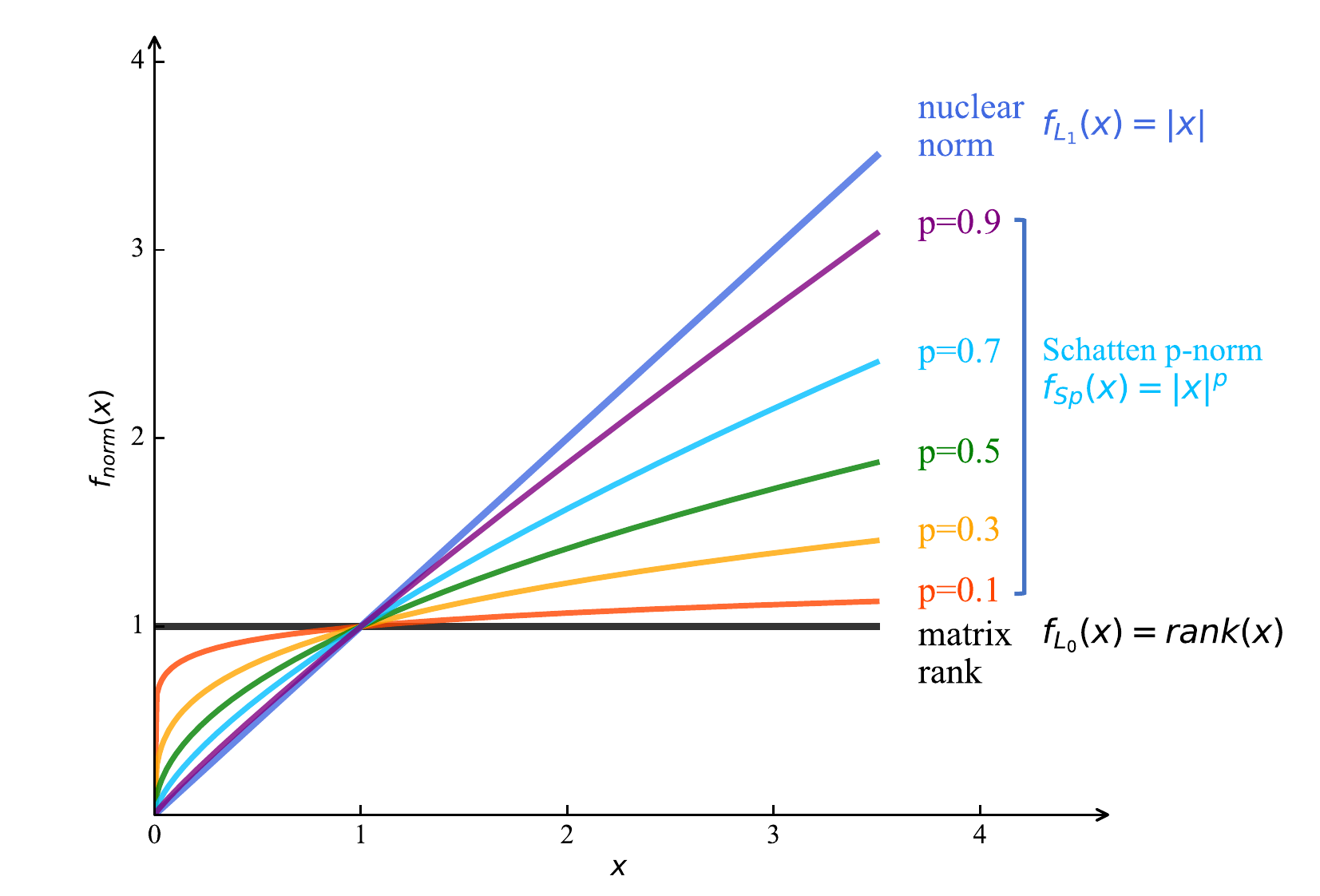}
  \caption{Comparison of the matrix rank function, nuclear norm and Schatten $p$-norm with different $p$ values. $x$ substitutes for singular value variables, $f_{norm}(x)$ denotes the continuous norm functions satisfying $norm(X) = \sum f_{norm}(x)$.}
  \label{norm_comparison}
\end{figure}

Using the truncated Schatten $p$-norm as the rank function in \minew{Eq.} \eqref{equ:TC}, the proposed LRTC-TSpN minimization model is formulated as: 
\minew{
\begin{equation} 
    \begin{aligned}
    \min_{\mathcal{X}}\Vert\mathcal{X}\Vert_{\theta,S_p}^p &=\min_{\mathcal{X}}\sum_{k=1}^{3}\alpha_k\Vert\mathbf{X}_{\left(k\right)}\Vert_{\theta,S_p}^p\\ 
    \text { s.t.}~ \mathcal{P}\odot\mathcal{X}&=\mathcal{P}\odot\mathcal{T}. \\
    \end{aligned}
    \label{equ:LRTC-TSpN}
\end{equation}
}
Note that we only concern about third-order tensor as mentioned above. Although \minew{Eq.} \eqref{equ:LRTC-TSpN} is a concise expression, this problem is challenging to solve because the objective function in \eqref{equ:LRTC-TSpN} is nonconvex due to the Schatten $p$-norm and truncation, however, we develop an efficient algorithm which combines the ADMM with GST to obtain the optimal solution.

\subsection{Iterative solving framework via ADMM}
\label{admm_solve}
Before solving this model, a problem can not be negligible is the interdependence among the unfolding matrices $\mathbf{X}_{(k)}$ due to the sharing of the same entries. In order to optimize these variables independently, we introduce a series of auxiliary matrix variables $\mathbf{M}_{(k)}$ to convert the coupling to equality constraints, and the model in \minew{Eq.} \eqref{equ:LRTC-TSpN} changes to the following form:
\minew{
\begin{equation}
    \begin{aligned}
    \min_{\mathcal{X}}&\sum_{k=1}^{3}\alpha_k\Vert\mathbf{M}_{k}\Vert_{\theta,S_p}^p\\
    \text { s.t.}~&\left\{\begin{array}{l}\mathbf{X}_{(k)}=\mathbf{M}_{k},~k=1,2,3, \\
    \mathcal{P}\odot\mathcal{X}=\mathcal{P}\odot\mathcal{T}. \\    
    \end{array}\right. \\
    \end{aligned}
    \label{equ:LRTC-TSpN_cvrt}
\end{equation}}
\minew{Eq.} \eqref{equ:LRTC-TSpN_cvrt} leads to a typical constrained minimization problem. We adopt the Augmented Lagrange Multiplier (ALM) \citep{powell1969method} method to relax the equality constrains, and the corresponding augmented Lagrange function is defined as:
\begin{equation}
    \mathcal{L}_\mu(\mathcal{X},\mathbf{M}_{k},\mathbf{Z}_{k})=
    \sum_{k=1}^3\left(\alpha_k\Vert\mathbf{M}_{\left(k\right)}\Vert_{\theta,S_p}^p+\big\langle\mathbf{Z}_{k},\mathbf{X}_{(k)}-\mathbf{M}_{k}\big\rangle+\frac{\mu_k}{2}\Vert\mathbf{X}_{(k)}-\mathbf{M}_{k}\Vert_F^2\right),
    \label{equ:ALM}
\end{equation}
where $\mathbf{Z}_{k}$ is the Lagrange multiplier (dual variable), $\mu_k$ is a positive penalty scalar, and $\langle\cdot\rangle$ is the inner product of matrix. Then, the ADMM \citep{boyd2011distributed} iterative framework is utilized to alternately minimize each of the optimal primal and dual variable until convergence.The iterative scheme of ADMM can formulate as follows:
\begin{equation}
    \begin{aligned} 
    \left\{\begin{array}{l}\mathbf{M}_k^{j+1}=\operatorname{arg}\min_{\mathbf{M}_{k}}\mathcal{L}_\mu(\mathcal{X}^j,\mathbf{M}_{k},\mathbf{Z}_{k}^j) \\
    \mathcal{X}^{j+1}=\operatorname{arg}\min_{\mathcal{X}}\mathcal{L}_\mu(\mathcal{X},\mathbf{M}_{k}^{j+1},\mathbf{Z}_{k}^j) \\    
    \mathbf{Z}_k^{j+1}=\mathbf{Z}_k^{j}+\mu_k(\mathbf{X}_{(k)}^{j+1}-\mathbf{M}_{k}^{j+1}).
    \end{array}\right. \\
    \end{aligned}
    \label{equ:ALM_update}
\end{equation}
For $j$-th iteration, we optimize each subproblem to find their optimal solution, then the ADMM scheme guarantees the final solution to converge to global optimum. From \minew{Eqs.} \eqref{equ:ALM} and \eqref{equ:ALM_update} we see that the penalty parameter $\mu_k$ serves as the learning rate for Lagrange dual variable as well as the penalty in \minew{the} objective function. In theory, there needs to be a sufficiently large $\mu_k$ to control the degree of constraint violation in ALM. However, when $\mu_k$ is getting larger, the condition number of Hessian matrix also becomes larger, which will result in the difficulties of calculation; when $\mu_k$ grows too slowly, the convergence speed may be affected significantly. Therefore, it is rational to set a small $\mu_k$ at first and increase \minew{it} with moderate steps to balance the convergence and accuracy.

\subsubsection{M-subproblem}
According to the optimization subproblem in \minew{Eqs.} \eqref{equ:ALM} and \eqref{equ:ALM_update}, $\mathbf{M}_k\,(k=1,2,3)$ are independent. For term $\mathbf{M}_k^{j+1}$, the original problem is converted to:
\begin{equation}
\begin{aligned}
    \mathbf{M}_k^{j+1}=&\operatorname{arg}\min_{\mathbf{M}_k}~\alpha_k\Vert\mathbf{M}_{\left(k\right)}\Vert_{\theta,S_p}^p+\big\langle\mathbf{Z}_{k}^j,\mathbf{X}_{(k)}^j-\mathbf{M}_{k}\big\rangle+\frac{\mu_k}{2}\Vert\mathbf{X}_{(k)}^j-\mathbf{M}_{k}\Vert_F^2,\\
    =&\operatorname{arg}\min_{\mathbf{M}_k}~\alpha_k\Vert\mathbf{M}_{\left(k\right)}\Vert_{\theta,S_p}^p -\big\langle\mathbf{Z}_{k}^j,\mathbf{M}_{k}\big\rangle +
    \frac{\mu_k}{2}\big\langle\mathbf{X}_{(k)}^j-\mathbf{M}_{k},\mathbf{X}_{(k)}^j-\mathbf{M}_{k}\big\rangle,\\
    =&\operatorname{arg}\min_{\mathbf{M}_k}~\alpha_k\Vert\mathbf{M}_{\left(k\right)}\Vert_{\theta,S_p}^p-\big\langle\mathbf{Z}_{k}^j,\mathbf{M}_{k}\big\rangle +
    \frac{\mu_k}{2}\big\langle\mathbf{M}_{k},\mathbf{M}_{k}\big\rangle -\mu_k\big\langle\mathbf{X}_{(k)}^j,\mathbf{M}_{k}\big\rangle,\\
    =&\operatorname{arg}\min_{\mathbf{M}_k}~\alpha_k\Vert\mathbf{M}_{\left(k\right)}\Vert_{\theta,S_p}^p+\frac{\mu_k}{2}\big\langle\mathbf{M}_{k},\mathbf{M}_{k}\big\rangle -\mu_k\big\langle\mathbf{X}_{(k)}^j+\frac{\mathbf{Z}_{k}^j}{\mu_k},\mathbf{M}_{k}\big\rangle,\\
    =&\operatorname{arg}\min_{\mathbf{M}_k}~\alpha_k\Vert\mathbf{M}_{\left(k\right)}\Vert_{\theta,S_p}^p+\frac{\mu_k}{2}\Vert\mathbf{M}_{k}\Vert_F^2 -\mu_k\big\langle\mathbf{X}_{(k)}^j+\frac{\mathbf{Z}_{k}^j}{\mu_k},\mathbf{M}_{k}\big\rangle+\frac{\mu_k}{2}\Vert\mathbf{X}_{(k)}^j+\frac{\mathbf{Z}_{k}^j}{\mu_k}\Vert_F^2,\\
    =&\operatorname{arg}\min_{\mathbf{M}_k}~\frac{\alpha_k}{\mu_k}\Vert\mathbf{M}_{\left(k\right)}\Vert_{\theta,S_p}^p+ \frac{1}{2}\Vert\mathbf{M}_{k}-(\mathbf{X}_{(k)}^j+\frac{\mathbf{Z}_{k}^j}{\mu_k})\Vert_F^2,\\
\end{aligned}
\label{equ:M-sub}
\end{equation}
However, the Schatten $p$-norm regularization in \minew{Eq.} \eqref{equ:M-sub} leads to a complicated nonconvex problem, whose global optimum is hard to obtain via general method (e.g., gradient based method). Fortunately, we derive the desired solution efficiently using the following theorems.
\begin{lemma}
\label{Lemma1}
(Von Neumann \citep{mirsky1975trace}) For any two matrices $\mathbf{A},\mathbf{B}\in\mathbb{R}^{m\times n}$, $\sigma(\mathbf{A})$ and $\sigma(\mathbf{B})$ are the \textbf{ordered} singular values of $\mathbf{A}$ and $\mathbf{B}$ with the \textbf{same order}, then $\Tr(\mathbf{A}^\mathsf{T}\mathbf{B})\leq\Tr(\sigma(\mathbf{A})^\mathsf{T}\sigma(\mathbf{B}))$ holds. The equality occurs only if there exist unitaries $\mathbf{U}$ and $\mathbf{V}$ that simultaneously singular value decompose $\mathbf{A}$ and $\mathbf{B}$, i.e., $\mathbf{A}=\mathbf{U}\diag(\sigma(\mathbf{A}))\mathbf{V}^\mathsf{T}$, $\mathbf{B}=\mathbf{U}\diag(\sigma(\mathbf{B}))\mathbf{V}^\mathsf{T}$.
\end{lemma}
According to Lemma \ref{Lemma1}, we convert the problem in \minew{Eq.} \eqref{equ:M-sub} to a tractable form based on the following theorem:
\begin{theorem}
For any $\lambda>0$, and $\mathbf{M},\mathbf{Y}\in\mathbb{R}^{m\times n}$, denote the full singular value decomposition of $\mathbf{Y}$ \minew{by} $\mathbf{Y}=\mathbf{U}\mathbf{\Sigma}\mathbf{V}^\mathsf{T}$, with $\Sigma = \diag(\sigma_1,\sigma_2,\ldots,\sigma_{\min\{m,n\}})$. The optimal solution of
\begin{equation}
    \min_{\mathbf{M}}~\lambda\Vert\mathbf{M}\Vert_{\theta,S_p}^p+\frac{1}{2}\Vert\mathbf{M}-\mathbf{Y}\Vert_F^2
\label{equ:thm1}
\end{equation}
is given by $\mathbf{M}_{opt}=\mathbf{U}\mathbf{\Delta}\mathbf{V}^\mathsf{T}$, where the diagonal matrix $\mathbf{\Delta}=\diag(\sigma_1,\ldots,\sigma_r,\delta_{r+1},\delta_{\min\{m,n\}})$ is arranged in the \textbf{same order} (e.g., non-ascending order) as $\mathbf{\Sigma}$, with the truncation $r$ calculated by \minew{Eq.} \eqref{truncation rate} and the $i$-th diagonal element $\delta_i~(i>r)$ given by the following problem: 
\begin{equation}
    \begin{aligned}
    \min_{\delta_i\geq 0}&\sum_{i=r+1}^{\min\{m,n\}}\left(\lambda\delta_i^p+\frac{1}{2}\left(\delta_i-\sigma_i\right)^2\right), \\
    \text { s.t.}~&\delta_i \geq \delta_j, \text{ for}~i\leq j, i=r+1,\ldots,\min\{m,n\}.
    \end{aligned}
    \label{equ:thm1_}
\end{equation}
\label{Theorem1}
\end{theorem}

\begin{proof}
Suppose the full SVD of $\mathbf{Y}$ is $\mathbf{Y}=\mathbf{U}\mathbf{\Sigma}\mathbf{V}^\mathsf{T}$. Given any two orthogonal matrices $\mathbf{Q}$ and $\mathbf{R}$, and any rectangular diagonal matrix $\mathbf{\Delta} = \diag(\delta_1,\ldots,\delta_{\min\{m,n\}})$ with $\delta_i \geq 0$ and $\mathbf{\Delta}$ is arranged by the same order (i.e., non-ascending order) as $\mathbf{\Sigma}$, then we construct a matrix $\mathbf{M}$ by $\mathbf{M}=\mathbf{Q}\mathbf{\Delta}\mathbf{R}^\mathsf{T}$. According to \minew{Lemma} \ref{Lemma1}, given an arbitrary $\mathbf{Y}$, $\forall \mathbf{M}$, $\Tr(\mathbf{M}^\mathsf{T}\mathbf{Y})\leq\Tr(\sigma(\mathbf{M})^\mathsf{T}\sigma(\mathbf{Y}))$ holds. Consider the arbitrariness of $\mathbf{M}$, we can first determine the singular value vectors $\mathbf{Q}, \mathbf{V}$ and then find the singular value matrix $\mathbf{\Delta}$, respectively.

For a fixed $\mathbf{\Delta}$, $\forall \mathbf{Q},\mathbf{V}$:
\begin{equation*}
    \begin{aligned}
    \lambda\Vert\mathbf{M}\Vert_{\theta,S_p}^p+&\frac{1}{2}\Vert\mathbf{M}-\mathbf{Y}\Vert_F^2= \lambda\sum_{i=r+1}^{\min\{m,n\}}\delta_i^p+\frac{1}{2}\Tr\left(\left(\mathbf{M-Y}\right)^\mathsf{T}\left(\mathbf{M-Y}\right)\right),\\
    &=\lambda\sum_{i=r+1}^{\min\{m,n\}}\delta_i^p+\frac{1}{2}\Tr\left(\left(\mathbf{Q}\mathbf{\Delta}\mathbf{R}^\mathsf{T}-\mathbf{U}\mathbf{\Sigma}\mathbf{V}^\mathsf{T}\right)^\mathsf{T}\left(\mathbf{Q}\mathbf{\Delta}\mathbf{R}^\mathsf{T}-\mathbf{U}\mathbf{\Sigma}\mathbf{V}^\mathsf{T}\right)\right),\\
    &=\lambda\sum_{i=r+1}^{\min\{m,n\}}\delta_i^p+\frac{1}{2}\left(\Tr\left(\mathbf{\Delta}^\mathsf{T}\mathbf{\Delta}\right)+\Tr\left(\mathbf{\Sigma}^\mathsf{T}\mathbf{\Sigma}\right)-2\Tr\left(\mathbf{M}^\mathsf{T}\mathbf{Y}\right)\right),\\
    &\geq\lambda\sum_{i=r+1}^{\min\{m,n\}}\delta_i^p+\frac{1}{2}\left(\Tr\left(\mathbf{\Delta}^\mathsf{T}\mathbf{\Delta}\right)+\Tr\left(\mathbf{\Sigma}^\mathsf{T}\mathbf{\Sigma}\right)-2\Tr\left(\mathbf{\Delta}^\mathsf{T}\mathbf{\Sigma}\right)\right),\\
    \end{aligned}
\end{equation*}

according to \minew{Lemma} \ref{Lemma1}, the equality holds only when $\mathbf{Q}=\mathbf{U}$, $\mathbf{R}=\mathbf{V}$. Then given fixed $\mathbf{Q}$ and $\mathbf{V}$, 
\begin{equation*}
    \begin{aligned}   
    &\min_{\mathbf{M}}~\lambda\Vert\mathbf{M}\Vert_{\theta,S_p}^p+\frac{1}{2}\Vert\mathbf{M}-\mathbf{Y}\Vert_F^2\iff \min_{\delta_i}~\lambda\sum_{i=r+1}^{\min\{m,n\}}\delta_i^p+\frac{1}{2}\Vert\mathbf{\Delta}-\mathbf{\Sigma}\Vert_F^2,\\
    &=\min_{\delta_i}~\lambda\sum_{i=r+1}^{\min\{m,n\}}\delta_i^p+\frac{1}{2}\left(\sum_{i=r+1}^{\min\{m,n\}}\left(\delta_i-\sigma_i\right)^2+\sum_{i=1}^r\left(\delta_i-\sigma_i\right)^2\right),\\
    &=\min_{\delta_i}~\underbrace{\sum_{i=r+1}^{\min\{m,n\}}\left(\lambda\delta_i^p+\frac{1}{2}\left(\delta_i-\sigma_i\right)^2\right)}_{\text{determine the latter } r+1\dots \min\{m,n\} ~\delta_i}+\underbrace{\frac{1}{2}\sum_{i=1}^r\left(\delta_i-\sigma_i\right)^2}_{\text{determine the first }r ~\delta_i}, \text{ ($\delta_i$ and $\delta_j$ are independent)}\\
    &\geq\min_{\delta_i}~\sum_{i=r+1}^{\min\{m,n\}}\left(\lambda\delta_i^p+\frac{1}{2}\left(\delta_i-\sigma_i\right)^2\right),\\
    \end{aligned}
\end{equation*}
and the equality holds only when $\delta_i=\sigma_i, i\in\{1,2,\ldots,r\}$. Therefore, given a specific $\mathbf{Y}$, minimizing \minew{Eq.} \eqref{equ:thm1} can be converted to minimizing \minew{Eq.} \eqref{equ:thm1_} with $\mathbf{M}=\mathbf{U}\mathbf{\Delta}\mathbf{V}^\mathsf{T}$ and $\delta_i=\sigma_i, i\in\{1,2,\ldots,r\}$.
\end{proof}

\begin{remark}
Similar conclusion can also be obtained via the optimization process of weighted matrix Schatten $p$-norm with the given weights $w_1,\ldots,w_r=0$ and $w_{r+1},\ldots,w_{\min\{m,n\}}=1$ in \citep{xie2016weighted}. Note that this derivation resembles the solution to the well-known matrix Schatten $p$-norm regularization proposed in \citep{nie2012robust}. However, in their work, they didn't consider the order constrains of the optimal solution in \minew{Eq.} \eqref{equ:thm1_}.
\end{remark}

By establishing \minew{Theorem} \ref{Theorem1}, we convert the variables of \minew{Eq.} \eqref{equ:thm1} from matrices to singular values. And we can construct the optimal $\mathbf{M}$ by calculating only the the latter $r+1\dots\min\{m,n\}$ singular values, given the SVD of $\mathbf{Y}$. However, solving the problem in \minew{Eq.} \eqref{equ:thm1_} is still tricky owing to the nonconvex objective function and the order constrains (i.e., $\delta_i \geq \delta_j, i\leq j$). To this end, we obtain the numerical solution with global optimum guarantees by introducing the following lemma.

\begin{lemma}
\label{Lemma2}
(Wangmeng Zuo \citep{zuo2013generalized}) For $y\in (0,+\infty)$, the $\mathcal{L}_p$-minimization problem
\begin{equation}
    \min_x~\frac{1}{2}(x-y)^2+\lambda x^p,~\forall\lambda>0,~1>p>0,
\label{equ:lp}
\end{equation}
can be solved by a generalized soft-thresholding function defined by:
\begin{equation}
    \begin{aligned}
    T_p^{GST}(y;\lambda)=\left\{\begin{array}{ll}0,&\text{if}~y\leq \tau_p^{GST}(\lambda), \\
    S_p^{GST}(y;\lambda),&\text{if}~y> \tau_p^{GST}(\lambda),\\     
    \end{array}\right. \\
    \end{aligned}
    \label{equ:gst}
\end{equation}
where $\tau_p^{GST}(\lambda)$ is the thresholding value defined as:
\begin{equation}
    \tau_p^{GST}(\lambda)=\left(2\lambda\left(1-p\right)\right)^{\frac{1}{2-p}}+\lambda p\left(2\lambda\left(1-p\right)\right)^{\frac{p-1}{2-p}},
    \label{eqn:taup}
\end{equation}
and $S_p^{GST}(y;\lambda)$ is in the range of $(\left(2\lambda\left(1-p\right)\right)^{\frac{1}{2-p}},+\infty)$, which can be obtained by solving the following equation:
\begin{equation}
    S_p^{GST}(y;\lambda)-y+\lambda p\left(S_p^{GST}(y;\lambda)\right)^{p-1}=0,
    \label{eqn:Sp}
\end{equation}
the numerical solution to this equation can be solved by an iterative manner. The GST algorithm is summarized in Algorithm \ref{algorithm1}.
\end{lemma}

\begin{algorithm}
\caption{Generalized soft-thresholding \citep{zuo2013generalized}}\label{algorithm1}
\KwIn{$y,\lambda,p,J$}
\KwOut{$T_p^{GST}(y;\lambda)$}
$\tau_p^{GST}(\lambda)\leftarrow\left(2\lambda\left(1-p\right)\right)^{\frac{1}{2-p}}+\lambda p\left(2\lambda\left(1-p\right)\right)^{\frac{p-1}{2-p}}$\;
\eIf{$y\leq\tau_p^{GST}(\lambda)$}{
    $T_p^{GST}(y;\lambda)\leftarrow 0$\;
    }{
    $k\leftarrow 0$, $x^{(k)}\leftarrow y$\;
    \For{$k=1$ \KwTo $J$}{$x^{(k+1)}\leftarrow y-\lambda p\left(x^{(k)}\right)^{p-1}$\;
    $k\leftarrow k+1$\;
    $T_p^{GST}(y;\lambda)\leftarrow x^{(k)}$\;}
    }
\end{algorithm}
In fact, we can separately solve the $i$-th subproblem in \minew{Eq.} \eqref{equ:thm1_} by introducing \minew{Lemma} \ref{Lemma2} if the problem can be decoupled (i.e., satisfying these order constrains). To this end, we establish the following theorem:
\begin{theorem}\label{Theorem2}
Supposing the global minimum of the $i$-th subproblem:
\begin{equation*}
    \min_{\delta_i}~\frac{1}{2}(\delta_i-\sigma_i)^2+\lambda \delta_i^p,~\lambda>0, 1>p>0
\end{equation*}
obtained by GST in Algorithm \ref{algorithm1} is $\delta_i^*$, then we have:
\begin{equation}\label{eqn:ineq}
    \delta_i^*\geq \delta_j^*,\forall i\leq j,~i,j=r+1,\ldots,\min\{m,n\}.
\end{equation}
\end{theorem}

\begin{proof} We prove this theorem based on the GST algorithm. Note that for a fixed $\lambda$, the thresholding value $\tau_p^{GST}$ is a positive constant. 
When $\sigma_i, \sigma_j \leq \tau_p^{GST}$, $\delta_i^*=\delta_j^*$, the inequality \eqref{eqn:ineq} holds; when $\sigma_i>\tau_p^{GST}$ and $\sigma_j\leq\tau_p^{GST}$, $\delta_j^*=0$, and $\delta_i^*$ can be obtained by \minew{Eq.} \eqref{eqn:Sp} with $\delta_i^*>\left(2\lambda\left(1-p\right)\right)^{\frac{1}{2-p}}>0$, then the proposition still holds; when $\sigma_i, \sigma_j > \tau_p^{GST}$, both of them satisfy \minew{Eq.} \eqref{eqn:Sp}, i.e.:
\begin{equation}\label{eqn:deltaij}
    \begin{aligned}
    \left\{\begin{array}{l}
    \delta_i^*+\lambda p\left(\delta_i^*\right)^{p-1}=\sigma_i,~\delta_i^*>\left(2\lambda\left(1-p\right)\right)^{\frac{1}{2-p}},\\
    \delta_j^*+\lambda p\left(\delta_j^*\right)^{p-1}=\sigma_j,~\delta_j^*>\left(2\lambda\left(1-p\right)\right)^{\frac{1}{2-p}},\\
    \end{array}\right.
    \end{aligned}
\end{equation}
denote $f(x)=x+\lambda p x^{p-1}, x\in (\left(2\lambda\left(1-p\right)\right)^{\frac{1}{2-p}},+\infty)$, set the first derivative of $f(x)$ to zero, we have $f'(x)=1+\lambda p(p-1)x^{p-2}=0$, thus the stationary point $x_0=\left(\lambda p\left(1-p\right)\right)^{\frac{1}{2-p}}$. It is not hard to prove that when $0<p<1$, $x_0<\left(2\lambda\left(1-p\right)\right)^{\frac{1}{2-p}}$, which means $f(x)$ is monotonically increasing in $(\left(2\lambda\left(1-p\right)\right)^{\frac{1}{2-p}},+\infty)$. Therefore, according to \minew{Eq.} \eqref{eqn:deltaij}, $f(\delta_i^*)=\sigma_i \geq f(\delta_j^*)=\sigma_j$, we have $\delta_i^* \geq \delta_j^*$, the theorem is proved.
\end{proof}
\begin{remark}
The solution to \minew{Eq.} \eqref{equ:lp} has closed-form solution when $p \in \{\frac{1}{2},\frac{2}{3},1\}$ \citep{cao2013fast}. Specially, when $p=1$, the GST becomes the singular value thresholding operator used in NN minimization \citep{cai2010singular}.
\end{remark}
By introducing \minew{Theorem} \ref{Theorem2}, we can divide the problem \eqref{equ:thm1_} to $\min\{I_k,\prod_{i\neq k} I_i\}$ independent subproblems that can be solved consequently using GST. So far, based on the above analysis, we can give the optimal solution to the $M$-subproblem by:
\begin{equation}\label{eqn:msolve}
    \mathbf{M}_{k}^{j+1}=\mathbf{U}\diag(\sigma_1,\ldots,\sigma_{r_k},\delta_{r_k+1},\ldots,\delta_{\min\{I_k,\prod_{i\neq k} I_i\}})\mathbf{V}^\mathsf{T},
\end{equation}
where $\mathbf{U}$ and $\mathbf{V}$ are the left and right singular value vector matrices of $\mathbf{X}_{(k)}^j+\frac{\mathbf{Z}_{k}^j}{\mu_k}$, respectively, $\sigma_1,\ldots,\sigma_{r_k}$ are the corresponding first $r_k$ singular values; $\delta_i = GST(\sigma_i,\frac{\alpha_k}{\mu_k},p)$ with $i=r_k+1,\ldots,\min\{I_k,\prod_{i\neq k} I_i\}$. 

From \minew{Eq.} \eqref{eqn:msolve} one can find that this equation holds for $n$ (in our case is three) $\mathbf{M}_k$ variables, and each $\mathbf{M}_k$ can fold into a tensor in its $k$-th dimension. In order to preserve useful information as much as possible, we develop a fourth-order tensor $\mathcal{M}\in \mathbb{R}^{n\times n_1 \times n_2 \times n_3}$ by adding an additional 'variable' dimension $n$ to update \minew{Eq.} \eqref{eqn:msolve} in the form of tensor at every turn. This can be formulated as:
\begin{equation}
    \mathcal{M}^{j+1}[k,:,:,:] = \operatorname{Fold}_k\left(\mathbf{U}\diag(\sigma_1,\ldots,\sigma_{r_k},\delta_{r_k+1},\ldots,\delta_{\min\{I_k,\prod_{i\neq k} I_i\}})\mathbf{V}^\mathsf{T}\right).
\end{equation}

\subsubsection{X-subproblem} To update $\mathcal{X}$, we solve the ADMM subproblem according to \minew{Eqs.} \eqref{equ:ALM} and \eqref{equ:ALM_update}:
\begin{equation}\label{equ:X-sub}
\begin{aligned}
    \mathcal{X}^{j+1}=&\operatorname{arg}\min_{\mathcal{X}}~\sum_{k=1}^3\left(\big\langle\mathbf{Z}_{k}^j,\mathbf{X}_{(k)}-\mathbf{M}_{k}^{j+1}\big\rangle+\frac{\mu_k}{2}\Vert\mathbf{X}_{(k)}-\mathbf{M}_{k}^{j+1}\Vert_F^2\right),\\
    =&\operatorname{arg}\min_{\mathcal{X}}~\sum_{k=1}^3\left(\big\langle\mathbf{Z}_{k}^j,\mathbf{X}_{(k)}\big\rangle+\frac{\mu_k}{2}\big\langle\mathbf{X}_{(k)}-\mathbf{M}_{k}^{j+1},\mathbf{X}_{(k)}-\mathbf{M}_{k}^{j+1}\big\rangle\right),\\
    =&\operatorname{arg}\min_{\mathcal{X}}~\sum_{k=1}^3\left(\big\langle\mathbf{Z}_{k}^j,\mathbf{X}_{(k)}\big\rangle+\frac{\mu_k}{2}\big\langle\mathbf{X}_{(k)},\mathbf{X}_{(k)}\big\rangle-\mu_k\big\langle\mathbf{X}_{(k)},\mathbf{M}_{k}^{j+1}\big\rangle\right),\\
    =&\operatorname{arg}\min_{\mathcal{X}}~\sum_{k=1}^3\left(\frac{\mu_k}{2}\Vert\mathbf{X}_{(k)}\Vert_F^2+\mu_k\big\langle\mathbf{X}_{(k)},\frac{\mathbf{Z}_{k}^j}{\mu_k}-\mathbf{M}_{k}^{j+1}\big\rangle+\frac{\mu_k}{2}\Vert\frac{\mathbf{Z}_{k}^j}{\mu_k}-\mathbf{M}_{k}^{j+1}\Vert_F^2\right),\\
    =&\operatorname{arg}\min_{\mathcal{X}}~\sum_{k=1}^3\left(\frac{\mu_k}{2}\Vert\mathbf{X}_{(k)}+\frac{\mathbf{Z}_{k}^j}{\mu_k}-\mathbf{M}_{k}^{j+1}\Vert_F^2\right),\\
    =&\operatorname{arg}\min_{\mathcal{X}}~\sum_{k=1}^3\left(\frac{\mu_k}{2}\Vert\mathcal{X}+\operatorname{Fold}_k(\frac{\mathbf{Z}_{k}^j}{\mu_k}-\mathbf{M}_{k}^{j+1})\Vert_F^2\right)=\operatorname{arg}\min_{\mathcal{X}}~\mathcal{F}\left(\mathcal{X},\mathbf{Z}_{k}^j,\mathbf{M}_{k}^{j+1}\right),
\end{aligned}
\end{equation}
where the last step denotes that we can update the variable $\mathcal{X}$ in the form of tensor efficiently. This subproblem is a trivial convex and smooth optimization that the closed-form solution can be derived by gradient method:
\begin{equation} \label{eqn:xsolve}
\begin{aligned}
    \frac{\partial\mathcal{F}}{\partial\mathcal{X}}&= \sum_{k=1}^3\mu_k\left(\mathcal{X}+\operatorname{Fold}_k(\frac{\mathbf{Z}_{k}^j}{\mu_k}-\mathbf{M}_{k}^{j+1})\right)=0,\\
    &\minew{\Rightarrow\left\{\begin{array}{ll}
    \mathcal{X}^{j+1}&=\left(\sum_{k=1}^3\mu_k\operatorname{Fold}_k(\mathbf{M}_{k}^{j+1}-\mathbf{Z}_{k}^j/\mu_k)\right)/\sum_{k=1}^3\mu_k,\\
    \mathcal{P}\odot\mathcal{X}^{j+1}&=\mathcal{P}\odot\mathcal{T},
    \end{array}\right.}
\end{aligned}
\end{equation}
where the last term in \minew{Eq.} \eqref{eqn:xsolve} serves as the information transmission from the observed tensor.

\subsubsection{Update dual variable}
The dual variable $\mathbf{Z}^{j+1}$ can also be updated in tensor the same as $\mathbf{M}^{j+1}$. By introducing  $\mathcal{Z}\in \mathbb{R}^{n\times n_1 \times n_2 \times n_3}$, we can calculate $\mathcal{Z}^{j+1}$ in one time by:
\begin{equation}
    \mathcal{Z}^{j+1} = \mathcal{Z}^{j}+\mu_k(\operatorname{Broadcast}(\mathcal{X}^{j+1})-\mathcal{M}^{j+1}),
\end{equation}
where $\operatorname{Broadcast}(\cdot)$ means duplication and extension the given tensor to higher-order tensor.

\subsubsection{Proposed algorithm}
Finally, the proposed LRTC-TSpN algorithm is summarized in Algorithm \ref{algorithm2}. In spite of the complex derivation and proof above, the proposed algorithm has concise form and efficient implementation.

\begin{algorithm}[!htb]
\caption{\minew{Low-Rank Tensor Completion based on Truncated Tensor Schatten $p-$Norm}}
\label{algorithm2}
\KwIn{$\mathcal{T}, \mathcal{P}, \mathcal{P}^{\perp},\bm{\alpha}, \bm{\mu}, \theta, p, J, \epsilon$}
\KwOut{$\hat{\mathcal{X}}$}
Initialize $\mathcal{Z}$ and $\mathcal{M}$ as zeros, $j\leftarrow0$, $\mathcal{P}\odot\mathcal{X}\leftarrow\mathcal{P}\odot\mathcal{T}$, and $\mathcal{P}^{\perp}\odot\mathcal{X}\leftarrow \operatorname{mean}(\mathcal{P}\odot\mathcal{T})$\footnotemark; \\
\While{not converged}{
\For{$k = 1$ \KwTo $3$}{
    $\mathbf{X}_{(k)}^j+\mathbf{Z}_{k}^j/\mu_k\leftarrow\operatorname{Unfold}_k\left(\mathcal{X}^j+\mathcal{Z}^{j}[k,:,:,:]/\mu_k\right)$;\\
    Compute SVD of matrix $\mathbf{X}_{(k)}^j+\mathbf{Z}_{k}^j/\mu_k$ as $\mathbf{U}\diag(\sigma_1,\ldots,\sigma_{\min\{I_k,\prod_{i\neq k} I_i\}})\mathbf{V}^\mathsf{T}$;\\
    Compute $r_k$ by $r_k\leftarrow\lceil\theta\cdot\min\{I_k,\prod_{i\neq k} I_i\}\rceil$;\\
    Solve problem \eqref{equ:thm1_} using Algorithm \ref{algorithm1} and compute $\delta_i = GST(\sigma_i,\frac{\alpha_k}{\mu_k},p)$ with $i>r_k$ ;\\
    Update $\mathcal{M}^{j+1}$ by $\mathcal{M}^{j+1}[k,:,:,:] \leftarrow \operatorname{Fold}_k\left(\mathbf{U}\diag(\sigma_1,\ldots,\sigma_{r_k},\delta_{r_k+1},\ldots,\delta_{\min\{I_k,\prod_{i\neq k} I_i\}})\mathbf{V}^\mathsf{T}\right)$;
    }
Update $\mathcal{X}^{j+1}$ by $\mathcal{X}^{j+1}\leftarrow\left(\sum_{k=1}^3\mu_k\mathcal{M}^{j+1}[k,:,:,:]-\mathcal{Z}^{j}[k,:,:,:]\right)/\sum_{k=1}^3\mu_k;$\\
Transmit observations by
$\mathcal{P}\odot\mathcal{X}^{j+1}\leftarrow\mathcal{P}\odot\mathcal{T}$;\\
Update $\mathcal{Z}^{j+1}$ by $\mathcal{Z}^{j+1} \leftarrow \mathcal{Z}^{j}+\mu_k(\operatorname{Broadcast}(\mathcal{X}^{j+1})-\mathcal{M}^{j+1})$;\\
$e^{j+1}\leftarrow\Vert\mathcal{X}^{j+1}-\mathcal{X}^j\Vert_F/\Vert\mathcal{X}^j\Vert_F$;\\
\If{$e^{j+1}<\epsilon$}{converge.}
$j\leftarrow j+1$;
}
\end{algorithm}
\footnotetext{\minew{Please note that $\operatorname{mean}(\mathcal{P}\odot\mathcal{T})$ denotes the mean of the observed elements of tensor $\mathcal{T}$ along all dimensions, which returns a scalar, and we use broadcast mechanism here for brevity.}}

Note that in \minew{Algorithm} \ref{algorithm2}, we apply the ADMM in which the primal variables $\{\mathbf{M}^{j}_k,\mathcal{X}^j,\mathbf{Z}_k^{j},k=1,2,3\}$ are sequentially updated in a fixed order: $\{\mathbf{M}^{j}_1\Rightarrow\mathbf{M}^{j}_2\Rightarrow\mathbf{M}^{j}_3\Rightarrow\mathcal{X}^j\Rightarrow\mathbf{Z}_1^{j}\Rightarrow\mathbf{Z}_2^{j}\Rightarrow\mathbf{Z}_3^{j}\}$. In general, due to the symmetry between $\mathbf{M}_k$ and $\mathbf{X}_{(k)}$, we can also treat $\mathbf{X}_{(k)}$ as auxiliary variables and permute them in the objective function. Although randomly permuted ADMM is claimed to be useful for avoiding low-quality local solutions \citep{sun2015expected}, the permutation including both $\mathbf{M}_k$ and $\mathbf{X}_{(k)}$ may cause ADMM to diverge in some nonconvex problems. In our case, with the convergence analysis (which will be discussed in \minew{Subsection} \ref{Convergence analysis}) and empirical experiments, the update order in \minew{Algorithm} \ref{algorithm2} always works in our testing scenarios. Meanwhile, ones can adjust the sequence order using random scheme or greedy scheme (even simply exchange the order of $\mathbf{M}_k$ and $\mathbf{X}_{(k)}$) according to their specific tasks.

\subsection{Convergence analysis}
\label{Convergence analysis}
The convergence characteristic of ADMM in the case of nonconvex problem is notoriously hard to analyse and still an open question in academic community \citep{chen2020robust}. With some valuable results in the pioneering work \citep{wang2019global}, we provide the theoretical global convergence analysis for \minew{Algorithm} \ref{algorithm2}. %by utilizing the theorems in \citep{wang2019global}. 
Besides, the good convergence performances of the proposed method in the following experiments in \minew{Section} \ref{experiments} also provide empirical evidences.

The convergence of our proposed method can be guaranteed by following corollary. 

\begin{corollary}
For a sufficiently large penalty parameter $\mu$, the iterative sequence $(\mathbf{M}_k^j,\mathcal{X}^j,\mathbf{Z}_k^j)$ generated by \minew{Algorithm} \ref{algorithm2} has unique limit point $(\mathbf{M}_k^*,\mathcal{X}^*,\mathbf{Z}_k^*)$ that satisfies $0 \in \partial\mathcal{L}_{\mu}(\mathbf{M}_k^*,\mathcal{X}^*,\mathbf{Z}_k^*)$, and $(\mathbf{M}_k^j,\mathcal{X}^j,\mathbf{Z}_k^j)$ converges globally to this unique stationary point.
\label{corollary1}
\end{corollary}

\minew{
\begin{proof} We give the proof of \minew{Corollary} \ref{corollary1} by introducing the following lemma.
\begin{lemma}
\label{Lemma3}
(Yu Wang \citep{wang2019global}) Consider a general optimization problem given by:
\begin{equation}
    \begin{aligned}
    \min_{\mathbf{X_0},\dots,\mathbf{X_p},\mathbf{Y}}&f(\mathbf{X_0},\mathbf{X_1},\dots,\mathbf{X_p})+h(\mathbf{Y})\\
    \text { s.t.}~&\mathbf{A_0X_0}+\cdots+\mathbf{A_pX_p}+\mathbf{BY}=\mathbf{C}, \\
    \end{aligned}
    \label{equ:lemma3}
\end{equation}
where $f:\mathbb{R}^{n\times m}\rightarrow\mathbb{R}$ is proper, continuous, and possibly nonsmooth, $h:\mathbb{R}^{q\times m}\rightarrow\mathbb{R}$ is proper and differentiable, both $f$ and $h$ could be nonconvex. Let $(\mathbf{X}^k,\mathbf{Y}^k,\mathbf{W}^k)$ be a sequence generated by the ADMM of problem \eqref{equ:lemma3}, $\mathbf{W}^k$ is the dual variable, and $\xi$ is a penalty parameter. Supposing that all the following assumptions hold:

\begin{assumption}
(coercivity) define the feasible set $\mathcal{F}:=\{(\mathbf{X},\mathbf{Y})\in\mathbb{R}^{(q+n)\times m}:\mathbf{AX}+\mathbf{BY}=\mathbf{C}\}$, if $(\mathbf{X},\mathbf{Y})\in\mathcal{F}$ and $\Vert(\mathbf{X},\mathbf{Y})\Vert\rightarrow\infty$, then the objective function $f+h$ is coercive over this set, i.e., $f+h\rightarrow\infty$;
\label{ass1}
\end{assumption}

\begin{assumption}
(feasibility) $\myIm(\mathbf{A})\subseteq\myIm(\mathbf{B})$, where $\myIm(\cdot)$ means the image of a matrix;
\label{ass2}
\end{assumption}

\begin{assumption}
(Lipschitz sub-minimization paths) for any fixed $\mathbf{X}, \operatorname{arg}\min_{\mathbf{Y}}\{f+h: \mathbf{BY}=\mathbf{U}\}$ has a unique minimizer $H$ and $H: \myIm(\mathbf{B})\rightarrow\mathbb{R}^{q\times m}$ defined by $H(\mathbf{U})\triangleq \operatorname{arg}\min_{\mathbf{Y}}\{f+h:\mathbf{BY}=\mathbf{U}\}$, is a Lipschitz continuous map; for $i=0,\dots,p, \operatorname{arg}\min_{\mathbf{X}_i}\{f+h:\mathbf{A}_i\mathbf{X}_i=\mathbf{U}\}$ has a unique minimizer $F_i$ and $F_i:\myIm(\mathbf{A})\rightarrow\mathbb{R}^{n_i\times m}$ defined by $F_i(\mathbf{U})\triangleq \operatorname{arg}\min_{\mathbf{X_i}}\{f+h:\mathbf{A}_i\mathbf{X}_i=\mathbf{U}\}$, is a Lipschitz continuous map;

\label{ass3}
\end{assumption}

\begin{assumption}
(objective-f regularity) the function $f$ has the form:
\begin{equation*}
    f(\mathbf{X}):=g(\mathbf{X})+\sum_{i=0}^{p}f_i(\mathbf{X_i}),
\end{equation*}
where: 
(1) $g(\mathbf{X})$ is differentiable and its gradient is Lipschitz continuous, with Lipschitz constant denoting $L_g$; (2) $f_0$ is lower semi-continuous and $f_i(\mathbf{X}_i)$ is restricted prox-regular for $i=1,\cdots,p$;
\label{ass4}
\end{assumption}

\begin{assumption}
(objective-h regularity) the function $h$ is differentiable and its gradient is Lipschitz continuous, with Lipschitz constant denoting $L_h$;
\label{ass5}
\end{assumption}
then, the ADMM algorithm converges subsequently for any sufficiently large $\xi$, i.e., it generates a sequence that is bounded with at least one limit point, and each limit point $(\mathbf{X}^*,\mathbf{Y}^*,\mathbf{W}^*)$ is a stationary point of $\mathcal{L}_{\xi}, i.e., 0 \in \partial\mathcal{L}_{\xi}(\mathbf{X}^*,\mathbf{Y}^*,\mathbf{W}^*)$. Additionally, if $\mathcal{L}_{\xi}$ has the Kurdyka-Łojasiewicz (KŁ) property, then $(\mathbf{X}^k,\mathbf{Y}^k,\mathbf{W}^k)$ converges globally to the unique limit point $(\mathbf{X}^*,\mathbf{Y}^*,\mathbf{W}^*)$.

\end{lemma}

Note that in the case of our nonconvex optimization problem as given in \minew{Eq.} \eqref{equ:LRTC-TSpN_cvrt}, $f(\mathbf{M})=\sum_{i=0}^{p}f_i(\mathbf{M}_i)=\sum_{k=1}^{3}\alpha_k\Vert\mathbf{M}_{k}\Vert_{\theta,S_p}^p,h(\mathbf{X}_{(k)})=0,\mathbf{C}=\mathbf{0}$, and the coefficient matrices are both identity matrices. The augmented Lagrange function $\mathcal{L}_{\mu}$ in \minew{Eq.} \eqref{equ:ALM} is a polynomial function, i.e., belonging to the category of semi-algebraic function, thus satisfying the KŁ inequality \citep{bolte2014proximal}. Then in order to guarantee the global convergence, we only need to verify whether the problem \eqref{equ:LRTC-TSpN_cvrt} holds \minew{Assumptions} \ref{ass1}-\ref{ass5}.

Consider the Schatten $p$-norm function is continuous and the linear constraints $\mathbf{M}_k=\mathbf{X}_{(k)}$ hold, when $\Vert(\mathbf{M}_k,\mathbf{X}_{(k)})\Vert \rightarrow \infty$, $f(\mathbf{M}) \rightarrow \infty$ obviously holds, thus \minew{Assumption} \ref{ass1} meets. \minew{Assumption} \ref{ass2} naturally holds \minew{because} $\mathbf{A}=\mathbf{I}$ and $\mathbf{B}=-\mathbf{I}$ in this case. 
As for \minew{Assumption} \ref{ass3},
we have derived the unique optimum for each ADMM subproblem in \minew{Subsection} \ref{admm_solve}, and both $\mathbf{I}$ and $-\mathbf{I}$ have full column rank with trivial null spaces, therefore, $F_i,H$ reduce to linear map operators. And for $\forall k_1,k_2\in\mathbb{N}$, we have:
\begin{equation*}
\begin{aligned}
    \Vert F_i(\mathbf{M}^{k_1})-F_i(\mathbf{M}^{k_2})\Vert=&\Vert F_i(\mathbf{M}^{k_1}-\mathbf{M}^{k_2})\Vert
    \leq\Vert F_i\Vert\Vert \mathbf{M}^{k_1}-\mathbf{M}^{k_2}\Vert,\\
    \Vert H(\mathbf{X}^{k_1})-H(\mathbf{X}^{k_2})\Vert=&\Vert H(\mathbf{X}^{k_1}-\mathbf{X}^{k_2})\Vert
    \leq\Vert H\Vert\Vert \mathbf{X}^{k_1}-\mathbf{X}^{k_2}\Vert,\\
\end{aligned}
\end{equation*}
thereby satisfying Lipschitz condition as well as \minew{Assumption} \ref{ass3}.
\minew{Assumption} \ref{ass5} holds because $h=0$. Hence, it remains to verify \minew{Assumption} \ref{ass4}.

For convenience, we ignore the truncation operation and simplify the objective function of problem \eqref{equ:LRTC-TSpN_cvrt} as $f(\mathbf{M})=\sum_{k=1}^{3}\Vert\mathbf{M}_{k}\Vert_{S_p}^p$.
Assumption \ref{ass4}.(1) can be satisfied by setting $g=0$. And $f_0=\Vert\mathbf{M}_{1}\Vert_{S_p}^p=\sum_i\sigma_i^p$, the $p$-th power of matirx Schatten $p$-norm, is lower bounded and continuous. Moreover, \cite{wang2019global} (Proposition 1, Section 2) proves that $\Vert \cdot\Vert_{Sp}^p$ is prox-regular with exclusion set being $S_m=\{\mathbf{M}|\forall \mathbf{N}\in\partial\Vert \mathbf{M}\Vert_{Sp}^p,\Vert \mathbf{N}\Vert_F>\rho\}$, for any positive parameter $\rho$, and the general sub-gradient of $\Vert \mathbf{M}\Vert_{Sp}^p$ can be given as \citep{watson1992characterization}:
\begin{equation*}
    \partial\Vert \mathbf{M}\Vert_{Sp}^p=\{\mathbf{U}^{(1)}\mathbf{DV}^{(1)\mathsf{T}}+\mathbf{U}^{(2)}\mathbf{TV}^{(2)\mathsf{T}}, \forall \mathbf{T} ~\textit{with} ~\sigma_1(\mathbf{T})\leq1\}.
\end{equation*}

Finally, according to \minew{Lemma} \ref{Lemma3}, we complete the proof.
\end{proof}
}

\subsection{Truncation rate decay strategy}
\label{decay strategy}
The truncation operator in our model plays a key role in capturing the low-rank pattern. In fact, the truncation rate $\theta$ controls the amount of information that we utilize to approximate the low-rank tensor. In practical applications, the singular values obtained from observations are used to estimate the true ones (from complete tensor). When the missing rate of observations is high, more information is needed to compute the Schatten $p$-norm. Intuitively, this global parameter should decline as the missing rate grows. Therefore, we design a truncation rate decay strategy as follows:
\begin{equation}\label{theta decay}
    \theta_{\psi} = \theta_0~\text{exp}(-\beta\times\psi),
\end{equation}
where $\theta_0$ is the initial truncation rate, $\beta$ is the decay rate, $\psi$ is the missing rate of observations defined as:
\minew{
\begin{equation}\label{missing rate}
    \psi = \frac{\Vert\mathcal{P}^{\perp}\Vert_0}{n_1\times n_2 \times n_3},
\end{equation}}
where $n_1,n_2,n_3$ are the dims of \minew{$\mathcal{P}^{\perp}$, $\Vert\cdot\Vert_0$ is the $\mathcal{L}_0$ norm}. When only a single set of data is given, let $\theta_{\psi}=\theta_0$; when a series of homogeneous datasets are given (this case often occurs when a large amount of data is divided into some relatively smaller ones along a given dim, e.g., partition along temporal dimension \citep{zhang2021customized} or according to different types of subnetworks \citep{asif2016matrix}), different $\theta_{\psi}$ are assigned to each LRTC-TSpN model accordingly. In practice, $\beta$ is recommended for $[1,3]$.

\section{Experiments}\label{experiments}

In this section, we conduct comprehensive experiments to evaluate the performance of proposed LRTC-TSpN model on four different types of spatiotemporal traffic datasets (i.e., origin-destination volume, traffic speed, parking occupancy, and link volume) under complicated missing patterns. First we give a brief description of datasets. Then we introduce four state-of-the-art imputation models as baselines and give some discussions about corresponding results in detail. \textbf{The Python source codes for this paper can be found in our GitHub repository.\footnote{from \url{https://github.com/tongnie/tensorlib}.} 
}

%\newpage
\subsection{Data preparation and experiment setup}
\subsubsection{Spatiotemporal traffic datasets}
In order to test the generalization ability of our model, we use four typical kinds of spatiotemporal traffic data from real-world transportation surveillance or measurement systems. All of the input data are arranged in a third-order tensor paradigm of \textit{time intervals}$\times$\textit{locations (sensors)}$\times$\textit{days}. The datasets are summarized as follows:
\begin{itemize}
\item \textbf{Xuhui urban arterial OD flow data}: This is an urban arterial origin-destination (OD) flow dataset which is aggregated by vehicles' travel path recorded at 15-minute interval using automatic vehicle identification (AVI) systems in Xuhui, Shanghai. This data consists of $17\times17$ OD pairs within two weeks (i.e., 14 days from January 6, 2021 to January 19, 2021). We only use the time interval from 6:00 a.m. to 22:00 p.m. of a day. The size of input tensor is: $(64\times289\times14)$. Note that we filter the OD pairs which the volume exceeds a specific threshold (i.e., 5 pairs/15 min) to avoid sparsity in the observation tensor.
\item \textbf{Guangzhou traffic speed data \footnote{from \url{https://doi.org/10.5281/zenodo.1205229}.}}: This is an urban traffic speed dataset which consists of 214 road segments within two months (i.e., 61 days from August 1, 2016 to September 30, 2016) at 10-minute interval, in Guangzhou, China. The size of input tensor is: $(144\times214\times61)$.
\item \textbf{Portland link volume data \citep{tufte2015evolution}}: This dataset consists of link volume collected from highways in Portland, which contains 1156 loop detectors within one month at 15-minute interval. The size of input tensor is: $(96\times1156\times31)$.
\item \textbf{Birmingham parking occupancy data \citep{stolfi2017predicting}}: This dataset collects the occupancy rates from 30 car parks in Birmingham from October 4, 2016 to December 19, 2016 at 30-minute interval (8:00 a.m. to 17:00 p.m.).  The size of input tensor is: $(18\times30\times77)$.
\end{itemize}

The speed and volume are typical large-scale traffic flow datasets with more fluctuation and randomness, while at the same time, appear distinct day-to-day and within-day recurrence. Imputing this kind of data \minew{entails} capturing details from time series. The occupancy data is a much simpler signal set with negligible intraday volatility, thus testing models' capability of perceiving low-rank characteristics. In terms of OD pair data, it shows less \minew{consistency} and \minew{time-dependency}, hence the recovery is more challenging without spatial topological relations. It can be used to quantify models' robustness when handling signals with a large variance.

\subsubsection{Baseline models and experiment setup}
One the one hand, previous studies have testified the tensor completion models are more effective than factorization-based methods \citep{ran2016tensor,chen2020nonconvex,chen2021Autoregressive}; on the other hand, the relatively small-scale datasets in our experiments, e.g., Xuhui and Birmingham data, are not sufficient enough to deploy a large-scale supervised leanrning model. Therefore, we compare our LRTC-TSpN model with several state-of-the-art tensor learning models as follows. All of them are based on tensor recovery with fine-designed rank approximation.
\begin{itemize}
\item \textbf{HaLRTC \citep{liu2013tensor}}: High accuracy Low Rank Tensor Completion. This is a typical benchmark used in many works, which minimizes the tensor nuclear norm via ADMM. In comparison with HaLRTC, we can see the advantages of SpN.
\item \textbf{LRTC-TNN \citep{chen2020nonconvex}}: Low Rank Tensor Completion with Truncated Nuclear Norm. This is a nonconvex minimization model based on the defined truncated tensor nuclear norm, which shows competitive performance in extensive imputation tasks.
\item \textbf{TNN-DCT \citep{lu2019low}}: Tensor Nuclear Norm minimization with Discrete Cosine Transform. This model aims to minimize a new tensor nuclear norm induced by the transform based t-product, which substitutes for a class of models using t-SVD. The discrete cosine transform used in their work is selected as the invertible linear transform operator in our experiment.
\item \textbf{SpBCD \citep{gao2020robust}}: Schatten $p$-norm minimization via Block Coordinate Descent. This model uses quadric penalty and BCD method to minimize the tensor Schatten $p$-norm. Compared with this model, we can directly show the strength of truncation operation and our iteration framework.
\end{itemize}

To measure the models' performance, we manually mask a few entries to substitute for the missing cases and compare the imputed values with masked ones. The evaluation metrics we use are mean absolute error (MAE) and root mean square error (RMSE):
\begin{equation}
    \begin{aligned}
    \text{MAE}=& \frac{1}{|\Omega_m|}\sum_{i\in\Omega_m} |x_i-\hat{x}_i|,\\
    \text{RMSE}=& \sqrt{\frac{1}{|\Omega_m|}\sum_{i\in\Omega_m} \left(x_i-\hat{x}_i\right)^2},\\
    \end{aligned}
\end{equation}
where $\Omega_m=\{(i,j,k)|~\text{when} ~x_{ijk} ~\text{is masked and not missing in original tensor}\}$.
Note that we only calculate imputation error on these masked entries whose ground-truth value is available in order to reflect the actual performance. For a fair comparison, we have carefully tuned the parameters required in the above baseline models. We set the missing rate ranging from $10\%$ to $90\%$ for all missing patterns. The convergence criterion $\epsilon$ for all models is set to $10^{-4}$. For the maximum number of iteration, we set 200 for LRTC-TSpN, LRTC-TNN, TNN-DCT and HaLRTC, 1000 for SpBCD. The weights of unfoldings $\alpha_i$ are simply set equally (i.e., $\alpha_1=\alpha_2=\alpha_3=\frac{1}{3}$) for HaLRTC, SpBCD, LRTC-TNN and LTRC-TSpN. And we set the regularization parameter $\mu_1=\mu_2=\mu_3=\mu$ with the initial value being $10^{-5}$ and updating in each iteration with $\mu=\min\{1.05\times\mu,10^5\}$ as in \citep{chen2020nonconvex}.

The two most important parameters that dominate the imputation and convergence performance for our model are the truncation rate $\theta_{\psi}$ and the Schatten $p$ value. To be fair, we calculate $\theta$ using \minew{Eq.} \eqref{theta decay} for both LRTC-TNN and LRTC-TSpN under different missing cases. The set of $\theta_0, \beta$ and $p$ are obtained by cross-validation and summarized in \minew{Table} \ref{truncation parameters}. More discussions about hyper-parameters tuning are presented in \ref{hpt}.

\begin{table}[!ht]
  \caption{Parameters setting for $\theta_0$, $\beta$ and $p$.}
  \label{truncation parameters}
  \centering
  \footnotesize
  %\scriptsize
  \begin{tabular}{ccccccc}
    \toprule
    \multicolumn{2}{c}{Datasets}
    & Xuhui & Guangzhou & Portland & Birmingham \\
    \midrule 
    \multirow{4}{*}{$\theta_0$ (LRTC-TSpN/LRTC-TNN)} 
    & RM   & 0.1/0.15 & 0.15/0.3 & 0.2/0.1 & 0.15/0.15\\
    & FM-0 & 0.1/0.1 & 0.05/0.05 & 0.1/0.1 & 0.05/0.05\\
    & FM-1 & 0.08/0.1 & 0.05/0.05 & 0.1/0.1 & 0.1/0.1\\
    & FM-2 & 0.1/0.1 & 0.05/6.04 & 0.1/0.05 & 0.05/0.05\\
    \midrule
    \multirow{4}{*}{$\beta$ (LRTC-TSpN/LRTC-TNN)}
    & RM   & 2/2 & 1.5/2 & 2.5/2 & 3/3\\
    & FM-0 & 2/2 & 2/2 & 3/3 & 2/2\\
    & FM-1 & 2.5/2.5 & 2/2 & 3/3 & 3/3\\
    & FM-2 & 2.5/2.5 & 2/2 & 3/2 & 2/2\\
    \midrule
    \multirow{4}{*}{$p$ (LRTC-TSpN/SpBCD)}
    & RM   & 0.3/0.5 & 0.3/0.3 & 0.5/0.5 & 0.7/0.7\\
    & FM-0 & 0.3/0.5 & 0.5/0.5 & 0.5/0.5 & 0.75/0.75\\
    & FM-1 & 0.35/0.3 & 0.3/0.4 & 0.5/0.5& 0.7/0.7\\
    & FM-2 & 0.3/0.3 & 0.3/0.3 & 0.5/0.5 & 0.75/0.7 \\
    \bottomrule
  \end{tabular}
\end{table}

\subsection{Algorithmic performance}
\subsubsection{Imputation results}
The imputation results are displayed in \minew{Table}~\ref{xuhui and guangzhou results} and \minew{Table}~\ref{portland and Birmingham results}. It is clear to see that the proposed LRTC-TSpN model outperforms other baseline models in most missing cases on four different spatiotemporal traffic datasets. What is noteworthy is that we pay little attention to horizontal comparisons among the imputation error of different datasets, as the value scale of four datasets vary significantly. With different \minew{settings} of parameters (e.g., $p$ and $\theta$), our model can suit various scenarios with different types of datasets. 

\begin{table}[!htbp]
\centering
  \caption{Imputation performance on Xuhui and Guangzhou dataset (MAE/RMSE).}
  \label{xuhui and guangzhou results}
  \setlength\tabcolsep{4pt}
  \centering
  %\footnotesize
  \scriptsize
  \begin{tabular}{c|ccccc|ccccc}
    \toprule
    & \multicolumn{5}{c|}{Xuhui} & \multicolumn{5}{c}{Guangzhou} \\
    \cmidrule(lr){2-6}
    \cmidrule(lr){7-11}
    Models & HaLRTC & LRTC-TNN & TNN-DCT & SpBCD & LRTC-TSpN & HaLRTC & LRTC-TNN & TNN-DCT & SpBCD & LRTC-TSpN\\
    \midrule
    10\%, RM  & 3.66/6.84 & 3.15/5.43 & 3.90/6.59 & \textbf{2.96}/\textbf{4.86} & \textbf{2.96}/4.97 & 2.15/3.04 & 1.88/2.71 & 2.21/3.05 & \textbf{1.84/2.61} & \textbf{1.84}/\textbf{2.60} \\
    30\%, RM  & 4.02/7.66 & 3.59/6.38 & 4.18/7.15 & 3.28/5.49 & \textbf{3.24}/\textbf{5.41} & 2.33/3.34 & 2.03/2.98 & 2.39/3.33 & 1.99/2.84 & \textbf{1.98}/\textbf{2.83} \\
    50\%, RM  & 4.75/9.11 & 4.02/7.17 & 4.70/8.06 & 3.48/6.09 & \textbf{3.46}/\textbf{6.08} & 2.57/3.68 & 2.21/3.29 & 2.59/3.65 & 2.15/3.10 & \textbf{2.14}\textbf{/3.09} \\
    70\%, RM  & 6.23/12.05 & 5.00/8.79 & 5.75/9.95 & 4.09/7.43 & \textbf{4.06}/\textbf{7.33} & 2.92/4.11 & 2.42/3.64 & 2.84/4.05 & 2.35/3.46 & \textbf{2.34}/\textbf{3.45} \\
    90\%, RM  & 11.01/20.68 & 7.40/12.55 & 8.60/15.28 & \textbf{5.34}/10.14 & 5.39/\textbf{10.09} & 38.65/40.04 & 2.78/4.18 & 3.28/4.73 & \textbf{2.67}/\textbf{3.99} & \textbf{2.67}/4.00 \\
    \midrule
    10\%, FM-0  & 3.77/6.83  & 3.36/5.74 & 3.90/6.42 & 3.00/4.96 & \textbf{2.99}/\textbf{4.95} & 2.93/4.09 & \textbf{2.55}/\textbf{3.90} & 3.00/4.33 & 2.63/3.97 & 2.61/3.97 \\
    30\%, FM-0  & 4.53/8.62  & 3.74/6.52 & 4.28/7.28 & \textbf{3.30}/\textbf{5.58} & \textbf{3.30}/\textbf{5.58} & 3.16/4.41 & \textbf{2.65}/4.03 & 3.21/4.69 & 2.69/4.03 & \textbf{2.65}/\textbf{4.02} \\
    50\%, FM-0  & 5.63/11.10  & 4.27/7.60 & 4.57/7.92 & 3.63/6.49 & \textbf{3.58}/\textbf{6.34} & 3.50/4.73 & 2.82/4.27 & 3.34/4.88 & 2.80/4.21 & \textbf{2.76}/\textbf{4.18} \\
    70\%, FM-0  & 7.76/15.44  & 4.99/9.07 & 5.42/9.97 & 4.34/8.06 & \textbf{4.27}/\textbf{7.82} & 4.29/5.43 & 3.05/4.56 & 3.57/5.12 & 2.95/\textbf{4.38} & \textbf{2.91}/\textbf{4.38} \\
    90\%, FM-0  & 11.47/21.63  & 7.97/15.73 & 8.58/16.49 & 6.64/13.35 & \textbf{6.12}/\textbf{11.92} & 23.53/26.09 & 5.10/9.93 & 5.27/8.61 & 4.02/6.10 & \textbf{3.44}/\textbf{5.22} \\
    \midrule
    10\%, FM-1  & 4.97/9.73 & 6.51/11.86 & \textbf{4.73}/\textbf{8.12} & 5.98/10.80 & 6.07/10.99 & 2.28/3.19 & 2.00/2.91 & 2.35/3.27 & 1.81/2.57 & \textbf{1.80}/\textbf{2.56} \\
    30\%, FM-1  & 6.35/14.00 & 5.68/11.01 & \textbf{5.38}/\textbf{9.82} & 5.53/10.67 & 5.49/10.61 & 2.63/3.62 & 2.25/3.33 & 2.62/3.64 & 1.96/2.82 & \textbf{1.95}/\textbf{2.80} \\
    50\%, FM-1  & 8.01/17.59 & 6.12/11.38 & 6.26/12.19 & \textbf{5.58}/\textbf{10.32} & 5.64/10.55 & 3.02/4.07 & 2.46/3.65 & 2.84/3.98 & 2.12/3.08 & \textbf{2.11}/\textbf{3.06} \\
    70\%, FM-1  & 8.54/18.81 & 6.85/14.18  & 6.84/13.75 & 7.53/15.96 & \textbf{6.78}/\textbf{13.17} & 4.01/5.12 & 2.83/4.22 & 3.28/4.64 & 2.44/3.62 & \textbf{2.42}/\textbf{3.60} \\
    90\%, FM-1  & 11.68/22.08 & 10.12/20.61 & 10.93/21.11 & 9.09/17.35 & \textbf{8.88}/\textbf{16.76} & 22.38/24.86 & 4.90/9.24 & 5.24/8.74 & 3.13/\textbf{4.72} & \textbf{3.10}/4.74 \\
    \midrule
    10\%, FM-2  & 3.74/6.69 & 3.34/5.77 & 3.95/6.35 & 3.12/5.16 & \textbf{3.10}/\textbf{5.11} & 2.22/3.12 & 2.00/2.90 & 2.22/3.06 & \textbf{1.88}/2.66 & \textbf{1.88}/\textbf{2.65} \\
    30\%, FM-2  & 4.57/9.59 & 3.42/5.78 & 4.12/7.12 & 3.10/5.07 & \textbf{3.08}/\textbf{5.06} & 2.45/3.44 & 2.23/3.29 & 2.41/3.34 & 2.00/2.83 & \textbf{1.99}/\textbf{2.82} \\
    50\%, FM-2  & 5.78/12.94 & 4.06/7.24 & 4.67/8.18 & 3.48/5.95 & \textbf{3.47}/\textbf{5.94} & 2.80/3.90 & 2.45/3.64 & 2.67/3.74 & 2.18/\textbf{3.13} & \textbf{2.17}/\textbf{3.13} \\
    70\%, FM-2  & 6.72/15.41 & 4.68/8.56 & 5.12/9.45 & 3.96/7.07 & \textbf{3.92}/\textbf{6.98} & 3.41/4.60 & 2.80/4.21 & 3.07/4.31 & 2.50/3.67 & \textbf{2.49}/\textbf{3.66} \\
    90\%, FM-2  & 10.21/20.92 & 8.87/24.77 & 7.07/14.73 & 5.72/11.54 & \textbf{5.40}/\textbf{10.47} & 6.51/7.90 & 3.57/5.35 & 4.30/5.81 & \textbf{3.13}/\textbf{4.68} & 3.15/4.72 \\
    \bottomrule
    \multicolumn{5}{l}{\scriptsize{Best results are bold marked.}}
  \end{tabular}
\end{table}

\begin{table}[!htbp]
  \caption{Imputation performance on Portland and Birmingham dataset (MAE/RMSE).}
  \label{portland and Birmingham results}
  \setlength\tabcolsep{2pt}
  \centering
  %\footnotesize
  \scriptsize
  \begin{tabular}{c|ccccc|ccccc}
    \toprule
    & \multicolumn{5}{c|}{Portland} & \multicolumn{5}{c}{Birmingham} \\
    \cmidrule(lr){2-6}
    \cmidrule(lr){7-11}
    Models & HaLRTC & LRTC-TNN & TNN-DCT & SpBCD & LRTC-TSpN & HaLRTC & LRTC-TNN & TNN-DCT & SpBCD & LRTC-TSpN\\
    \midrule
    10\%, RM & 13.04/18.96 & 11.94/16.96 & 14.03/19.39 & 12.10/17.39 & \textbf{11.81}/\textbf{16.73} & 10.91/17.18 & 8.48/\textbf{13.11} & \textbf{8.41}/13.37 & 11.56/16.58 & 9.78/14.23 \\
    30\%, RM & 14.07/20.56 & 12.52/18.15 & 14.93/20.62 & 12.65/18.26 & \textbf{12.45}/\textbf{17.72} & 16.73/26.63 & 13.33/19.75 & 15.30/26.79 & 13.83/19.92 & \textbf{12.55}/\textbf{18.20} \\
    50\%, RM & 15.25/22.45 & 13.52/20.08 & 15.92/22.33 & 13.43/19.64 & \textbf{13.23}/\textbf{19.29} & 24.56/39.85 & 20.07/31.07 & 25.05/45.03 & 17.13/25.08 & \textbf{16.83}/\textbf{24.48}\\
    70\%, RM & 17.59/25.87 & 15.12/22.85 & 17.79/25.20 & 14.40/21.28 & \textbf{14.29}/\textbf{21.06} & 41.90/73.43 & 32.61/53.44 & 42.93/76.90 & 23.83/37.35 & \textbf{23.65}/\textbf{36.82} \\
    90\%, RM & 37.64/50.94 & 18.90/27.96 & 23.68/33.78 & 16.22/\textbf{24.52} & \textbf{16.21}/24.58 & 648.63/925.40 & 63.69/108.02 & 90.13/158.86 & 50.24/90.32 & \textbf{45.26}/\textbf{74.61} \\
    \midrule
    10\%, FM-0 & 16.51/24.15 & 14.13/22.83 & 21.07/29.73 & 14.21/22.66 & \textbf{13.93}/\textbf{22.30} & 32.83/74.95 & 17.13/28.33 & 63.43/109.75 & 21.40/32.01 & \textbf{15.51}/\textbf{23.87} \\
    30\%, FM-0 & 17.74/26.68 & 14.92/22.97 & 22.01/31.94 & 14.93/22.51 & \textbf{14.59}/\textbf{22.02} & 57.02/123.02 & 37.04/97.14 & 65.07/107.78 & 31.72/66.69 & \textbf{30.11}/\textbf{69.09} \\
    50\%, FM-0 & 22.04/32.11 & 16.74/24.40 & 25.75/36.61 & 15.88/23.27 & \textbf{15.56}/\textbf{22.85} & 131.44/306.59 & 47.05/92.46 & 90.68/161.17 & 52.54/102.86 & \textbf{40.01}/\textbf{79.63} \\
    70\%, FM-0 & 59.44/81.86 & 20.07/31.49 & 31.40/44.27 & 18.35/27.71 & \textbf{17.84}/\textbf{26.95} & 188.41/401.44 & 69.46/135.68 & 97.99/168.26 & 98.47/229.90 & \textbf{60.46}/\textbf{116.85} \\
    90\%, FM-0 & 144.83/172.32 & 36.42/58.82 & 47.38/67.38 & 38.95/54.35 & \textbf{35.75}/\textbf{49.58} & 507.28/814.29 & 176.07/343.93 & 652.78/932.49 & 259.18/511.12 & \textbf{168.87}/\textbf{307.12} \\
    \midrule
    10\%, FM-1 & 15.86/22.96 & 13.77/19.77 & 19.71/27.07 & 14.09/20.31 & \textbf{13.62}/\textbf{19.45} & 16.76/25.26 & 15.54/23.35 & 19.98/32.78 & 15.05/22.28 & \textbf{14.62}/\textbf{21.51} \\
    30\%, FM-1 & 17.12/25.04 & 14.44/21.43 & 21.56/30.07 & 14.37/20.99 & \textbf{13.98}/\textbf{20.39} & 22.76/35.57 & 20.34/32.33 & 29.05/48.40 & 17.95/27.43 & \textbf{17.42}/\textbf{26.52} \\
    50\%, FM-1 & 21.57/30.65 & 16.15/23.82 & 24.64/34.56 & 15.26/22.18 & \textbf{14.97}/\textbf{21.85} & 37.85/64.32 & 27.43/44.16 & 45.98/79.18 & 21.91/36.53 & \textbf{21.36}/\textbf{33.55} \\
    70\%, FM-1 & 54.12/72.71 & 19.65/28.74 & 30.33/42.48 & 17.69/26.04 & \textbf{16.95}/\textbf{25.12} & 88.54/154.09 & 46.60/79.63  & 69.67/118.62 & 37.31/71.33 & \textbf{32.35}/\textbf{55.23} \\
    90\%, FM-1 & 146.05/174.28 & 42.75/58.13 & 49.05/75.46 & 35.85/49.69 & \textbf{32.58}/\textbf{45.41} & 529.80/790.23 & 219.00/448.25 & 647.47/924.04 & 228.90/392.54 & \textbf{118.11}/\textbf{215.04} \\
    \midrule
    10\%, FM-2 & 14.15/20.76 & 12.90/18.74 & 14.45/20.11 & 12.77/18.61 & \textbf{12.50}/\textbf{17.98} & 18.17/33.87 & 16.87/27.49 & 10.41/26.41 & 16.42/23.95 & \textbf{15.77}/\textbf{22.67} \\
    30\%, FM-2 & 15.94/23.35 & 13.45/19.79 & 15.69/21.92 & 13.04/18.77 & \textbf{12.89}/\textbf{18.53} & 36.55/69.71 & 19.15/28.02 & 14.96/28.01 & 17.08/24.58 & \textbf{16.50}/\textbf{23.39} \\
    50\%, FM-2 & 17.62/25.55 & 15.27/22.66 & 16.75/23.56 & 14.26/20.65 & \textbf{14.13}/\textbf{20.48} & 117.89/259.50 & 46.11/82.12 & 48.18/99.19 & 64.06/186.37 & \textbf{27.33}/\textbf{45.83} \\
    70\%, FM-2 & 34.79/51.09 & 18.43/27.25 & 19.96/28.57 & 16.38/24.07 & \textbf{16.05}/\textbf{23.57} & 262.56/492.35 & 71.98/123.41 & 119.40/216.55 & 156.16/368.95 & \textbf{57.15}/\textbf{101.35} \\
    90\%, FM-2 & 121.90/150.87 & 32.63/51.16 & 36.37/54.46 & 35.00/48.08 & \textbf{32.36}/\textbf{44.07} & 566.07/846.91 & 364.16/660.04 & 497.14/777.94 & 393.02/699.45 & \textbf{262.99}/\textbf{499.51} \\
    \bottomrule
    \multicolumn{5}{l}{\scriptsize{Best results are bold marked.}}
  \end{tabular}
\end{table}

In order to directly presents the effect of proposed model, we give some visualization examples of the imputation performance with different missing scenarios ($50\%$ missing rate) from four datasets in \minew{Figure} \ref{X-EX}, \ref{G-EX}, \ref{P-EX} and \ref{B-EX}. From these figures we can see that our model achieve high accuracy in imputing plausible values from partially observed data, even if the signals are fluctuating and the missing cases are continuous and non-random.

\begin{figure}[!htbp]
\centering
\subfigure[Xuhui, $50\%$ RM]{
\centering
\includegraphics[scale=0.45]{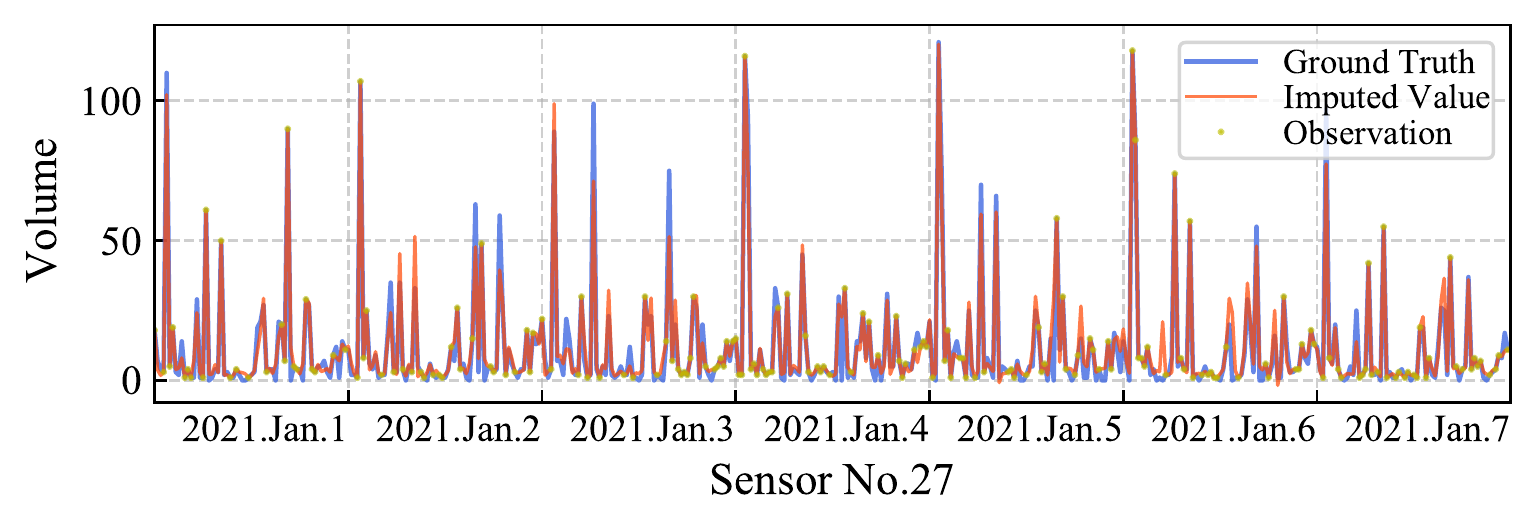}
}
\subfigure[Xuhui, $50\%$ FM-0]{
\centering
\includegraphics[scale=0.45]{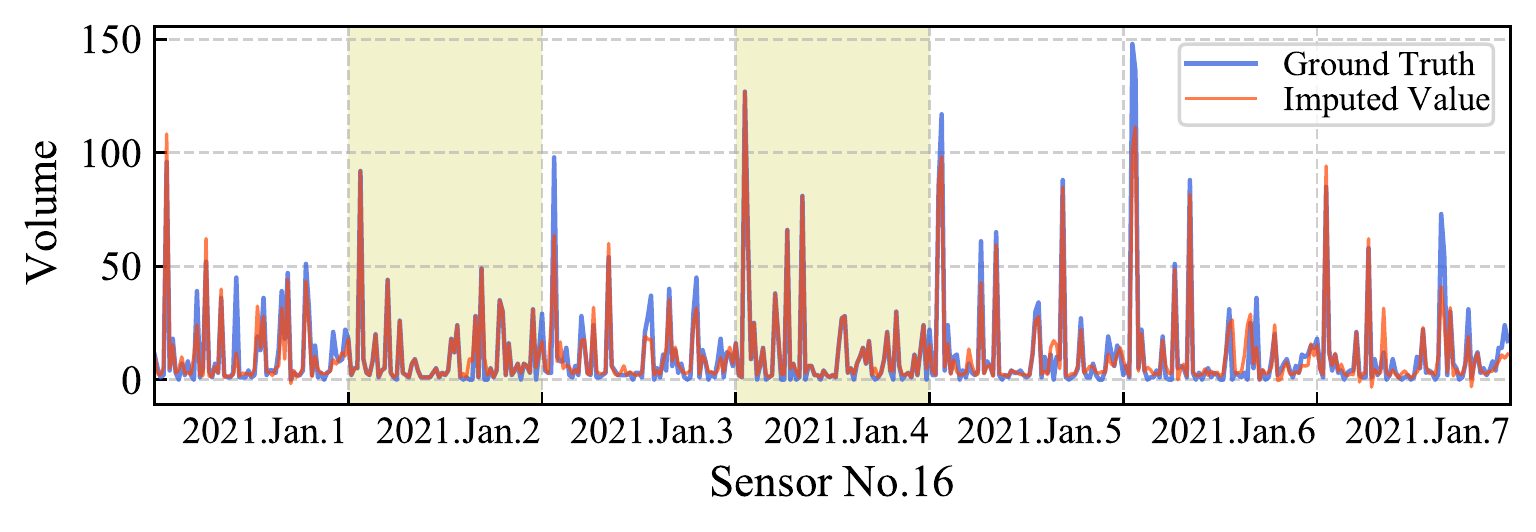}
}
\subfigure[Xuhui, $50\%$ FM-1]{
\centering
\includegraphics[scale=0.45]{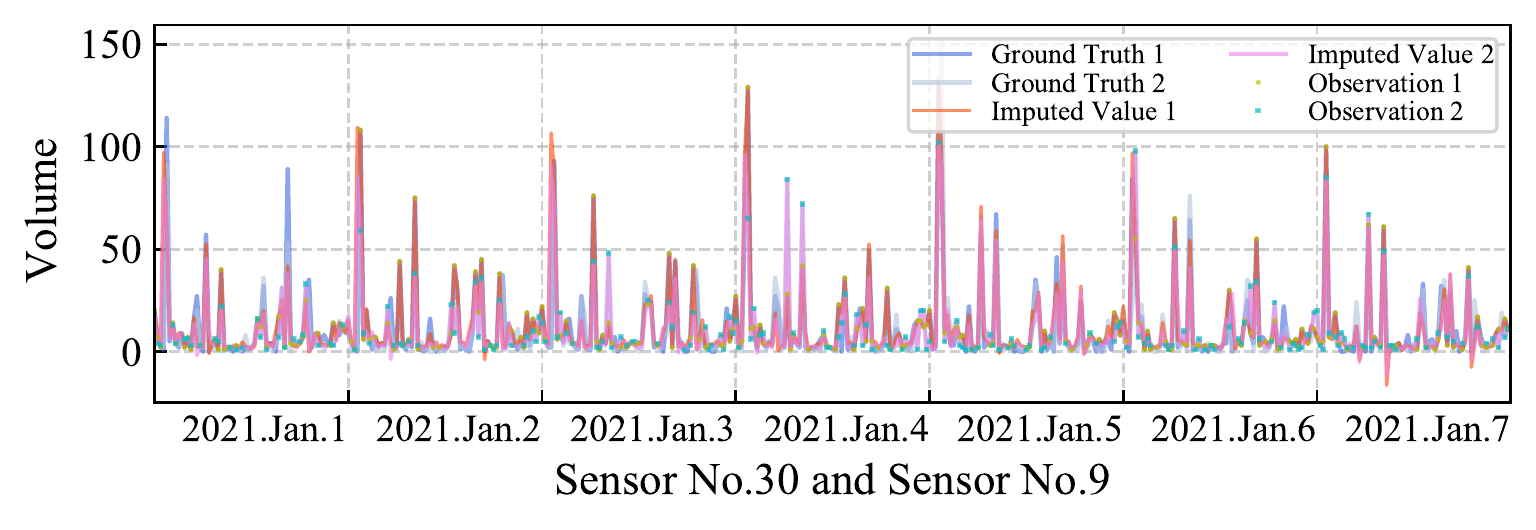}
}
\subfigure[Xuhui, $50\%$ FM-2]{
\centering
\includegraphics[scale=0.45]{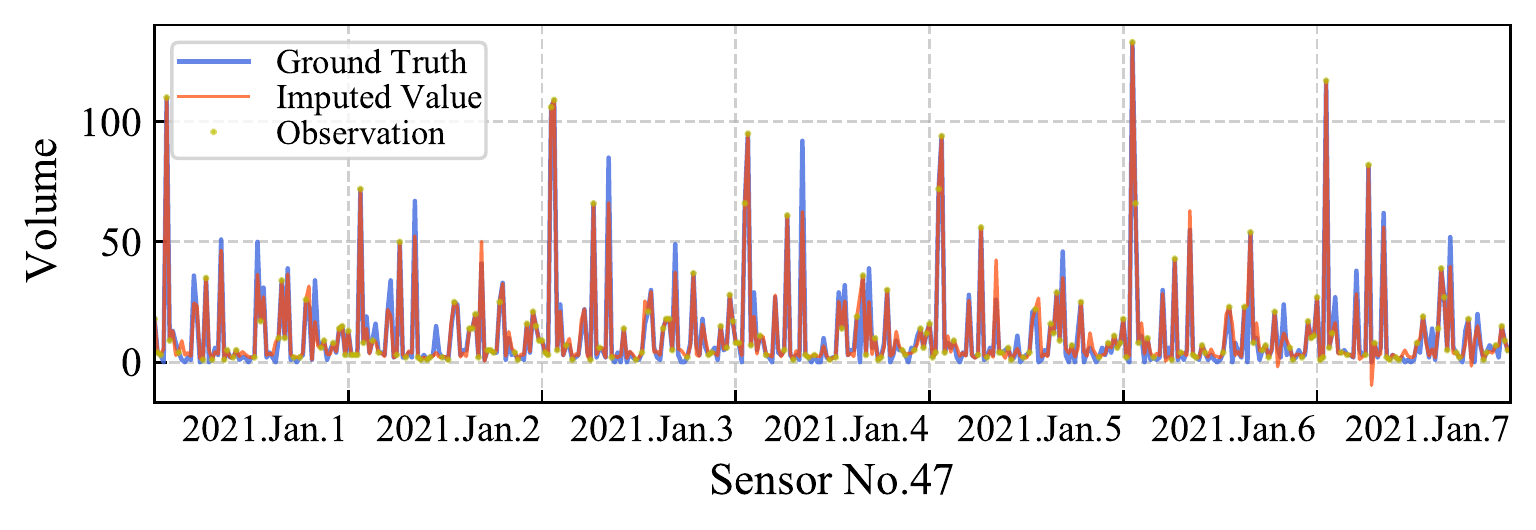}
}
\caption{Visualization examples of imputation for Xuhui ODs data. (a) $50\%$ random missing. (b) $50\%$ fiber mode-0 missing. The green window denotes the partial observations in this missing case. (c) $50\%$ mode-1 missing. Two different sensors who share the same missing time points are plotted for comparison(d) $50\%$ mode-2 missing.}
\label{X-EX}
\end{figure}

\begin{figure}[!htbp]
\centering
\subfigure[Guangzhou, $50\%$ RM]{
\centering
\includegraphics[scale=0.45]{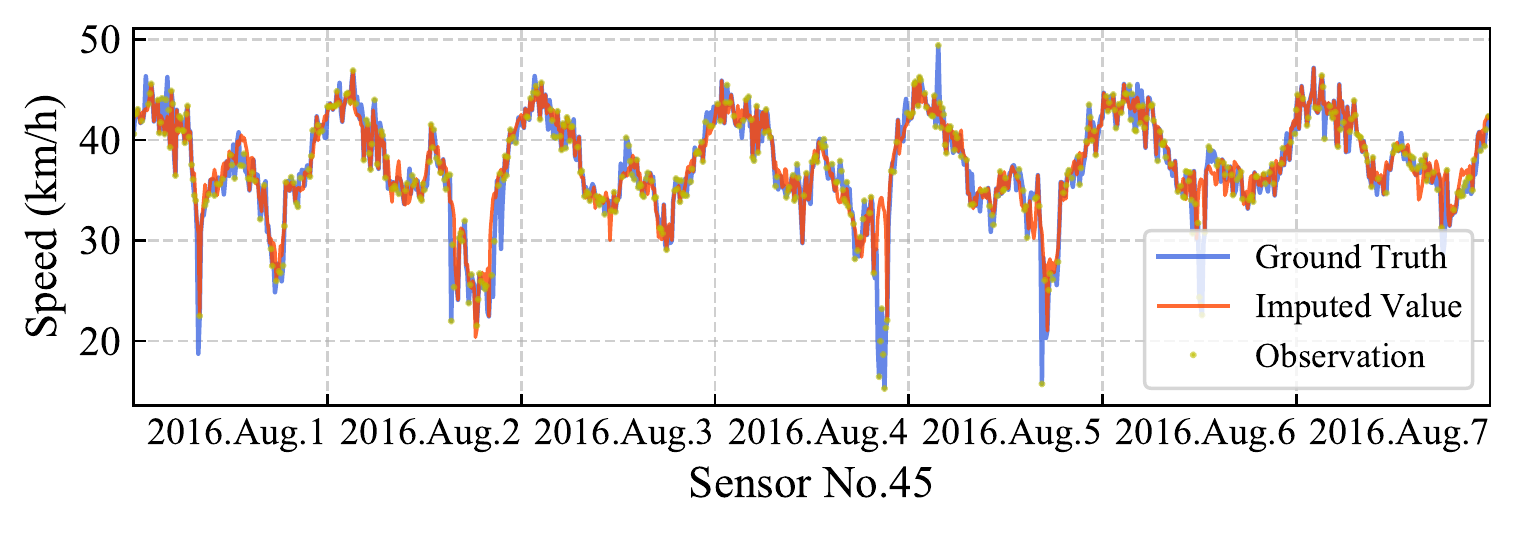}
}
\subfigure[Guangzhou, $50\%$ FM-0]{
\centering
\includegraphics[scale=0.45]{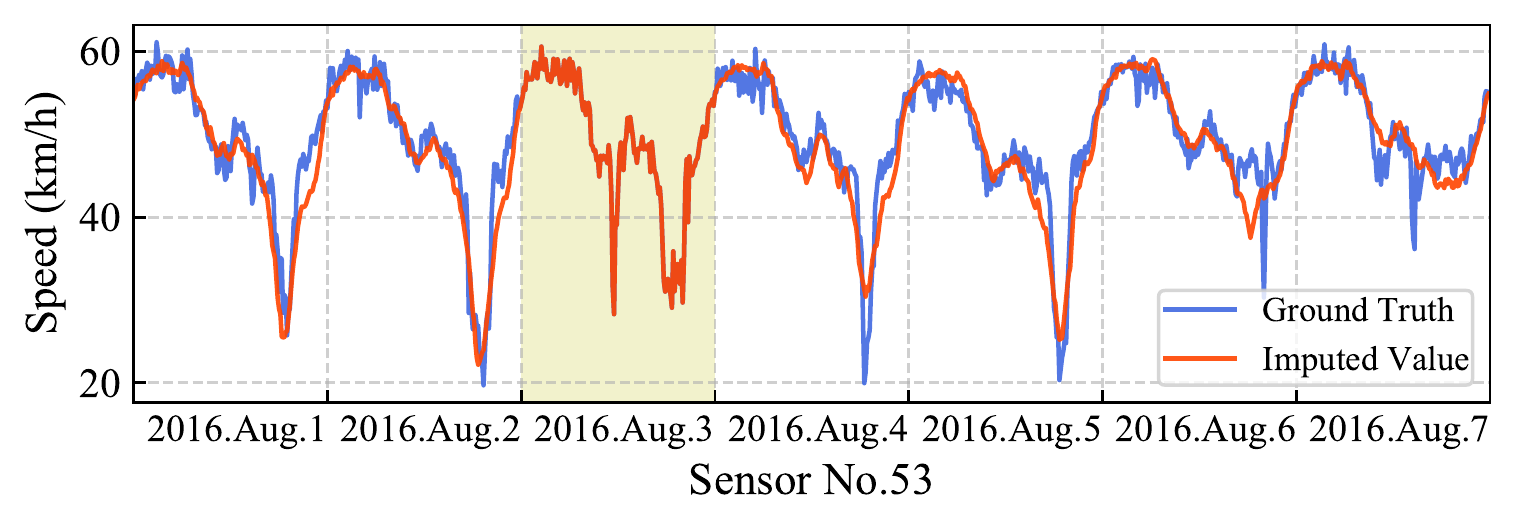}
}
\subfigure[Guangzhou, $50\%$ FM-1]{
\centering
\includegraphics[scale=0.45]{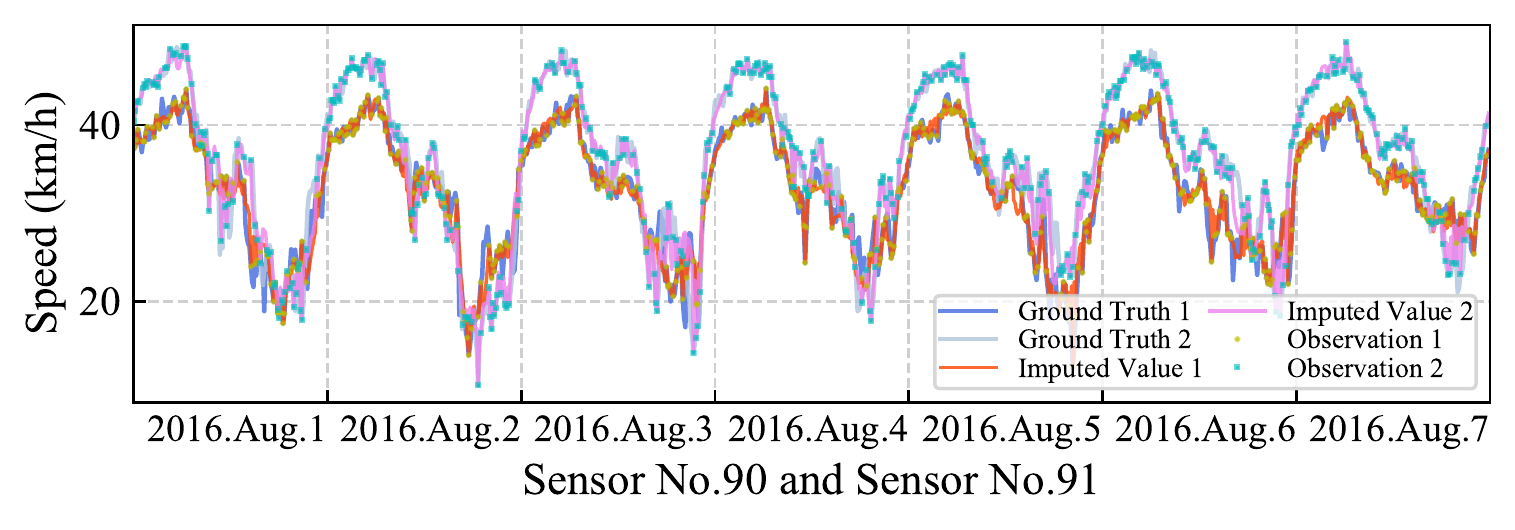}
}
\subfigure[Guangzhou, $50\%$ FM-2]{
\centering
\includegraphics[scale=0.45]{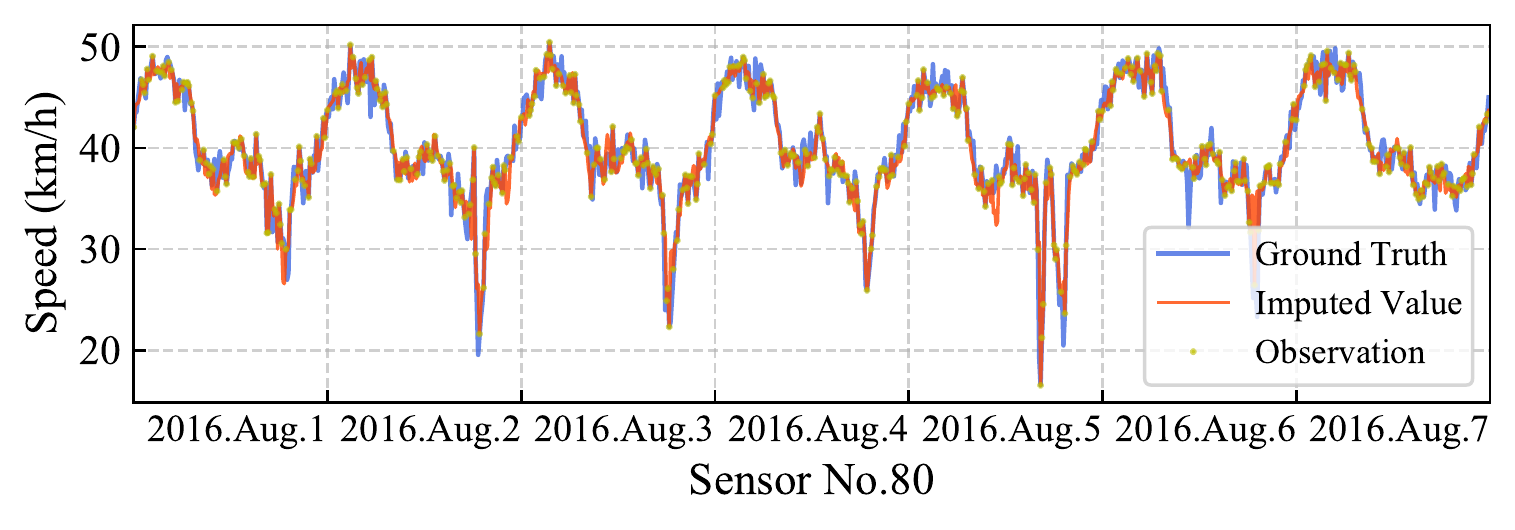}
}
\caption{Visualization examples of imputation for Guangzhou speed data. (a) $50\%$ random missing. (b) $50\%$ fiber mode-0 missing. The green window denotes the partial observations in this missing case. (c) $50\%$ mode-1 missing. Two different sensors who share the same missing time points are plotted for comparison(d) $50\%$ mode-2 missing.}
\label{G-EX}
\end{figure}

\begin{figure}[!htbp]
\centering
\subfigure[Portland, $50\%$ RM]{
\centering
\includegraphics[scale=0.45]{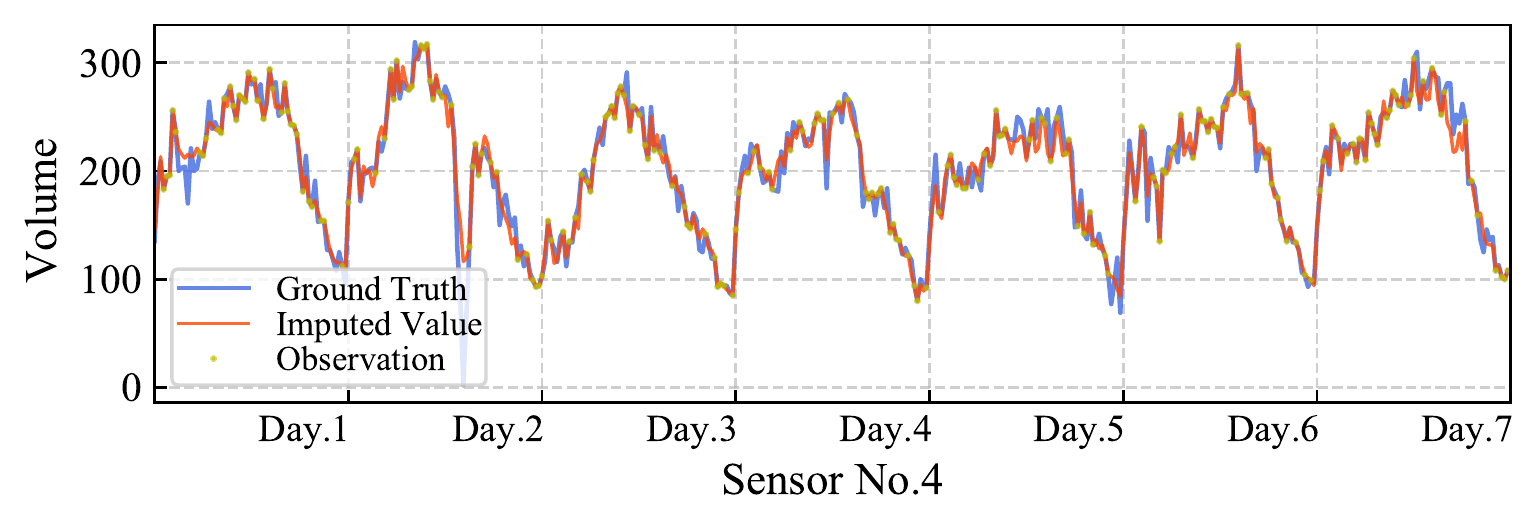}
}
\subfigure[Portland, $50\%$ FM-0]{
\centering
\includegraphics[scale=0.45]{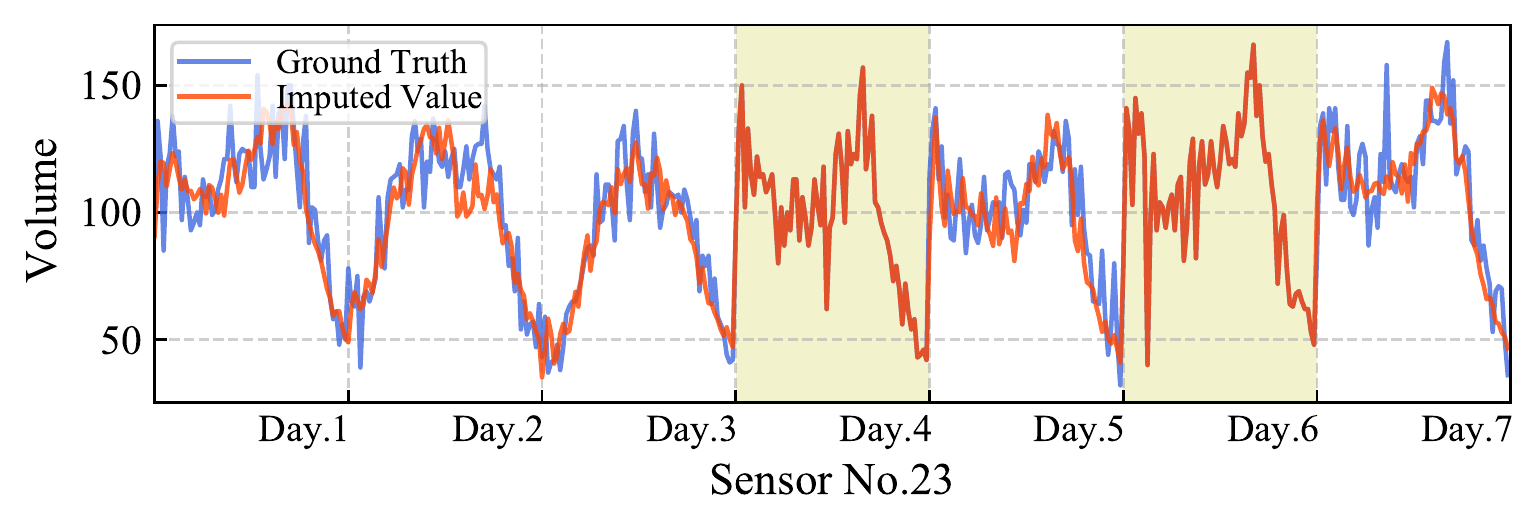}
}
\subfigure[Portland, $50\%$ FM-1]{
\centering
\includegraphics[scale=0.45]{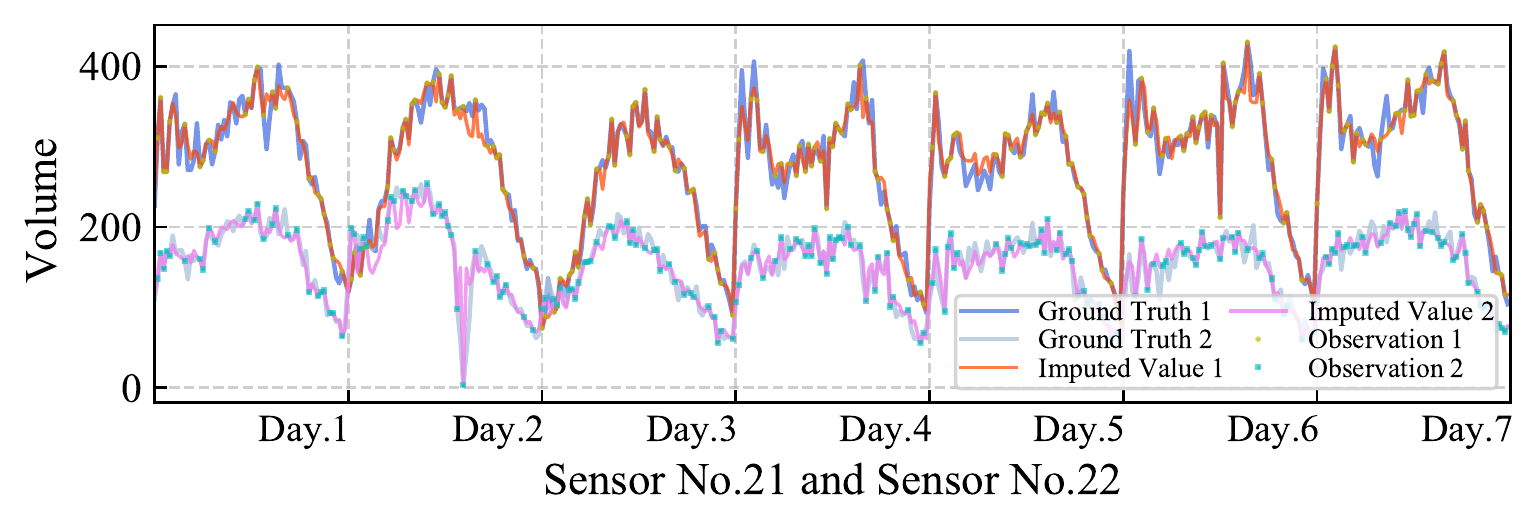}
}
\subfigure[Portland, $50\%$ FM-2]{
\centering
\includegraphics[scale=0.45]{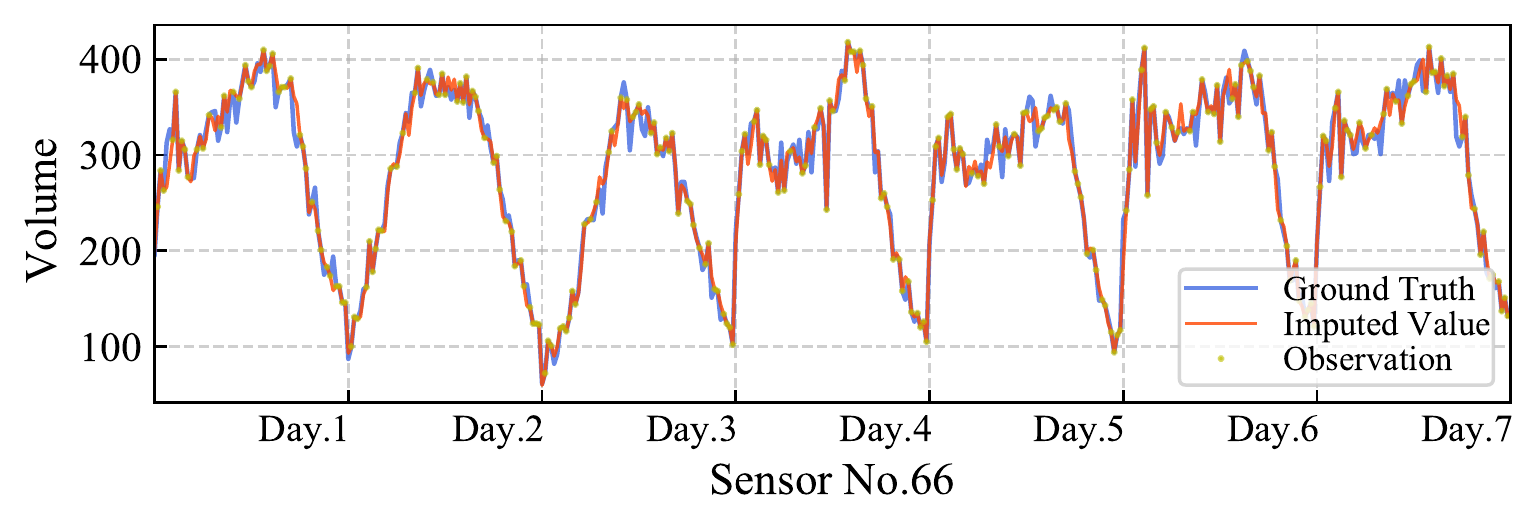}
}
\caption{Visualization examples of imputation for Portland volume data. (a) $50\%$ random missing. (b) $50\%$ fiber mode-0 missing. The green window denotes the partial observations in this missing case. (c) $50\%$ mode-1 missing. Two different sensors who share the same missing time points are plotted for comparison(d) $50\%$ mode-2 missing.}
\label{P-EX}
\end{figure}

\begin{figure}[!htbp]
\centering
\subfigure[Birmingham, $50\%$ RM]{
\centering
\includegraphics[scale=0.43]{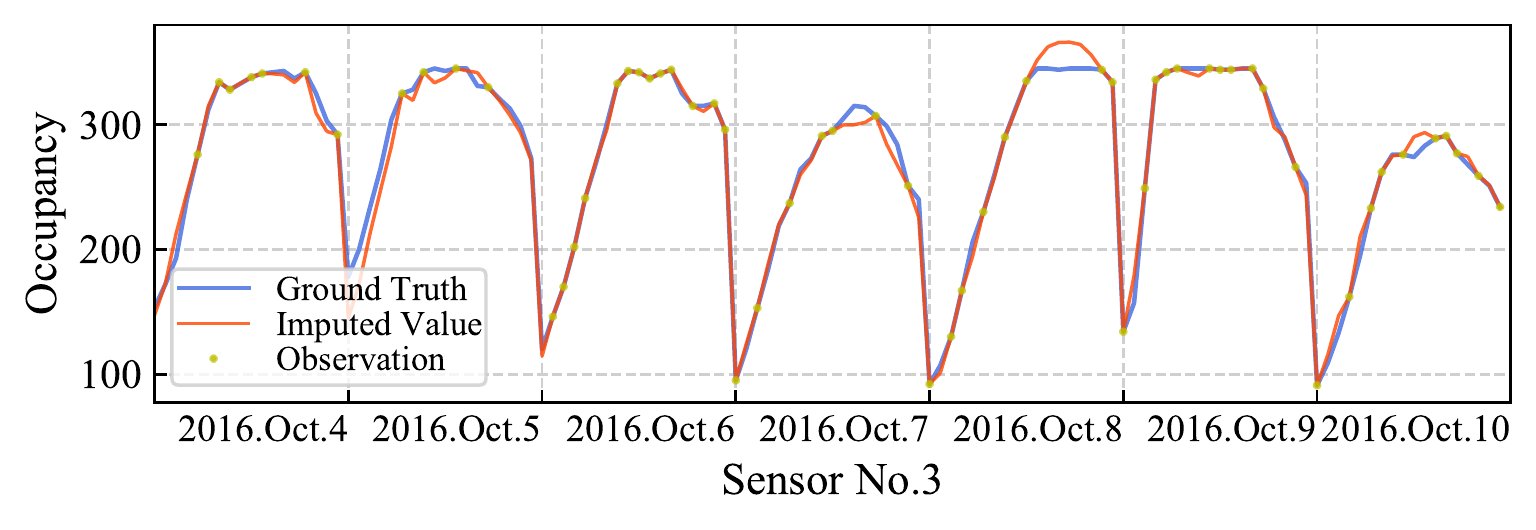}
}
\subfigure[Birmingham, $50\%$ FM-0]{
\centering
\includegraphics[scale=0.43]{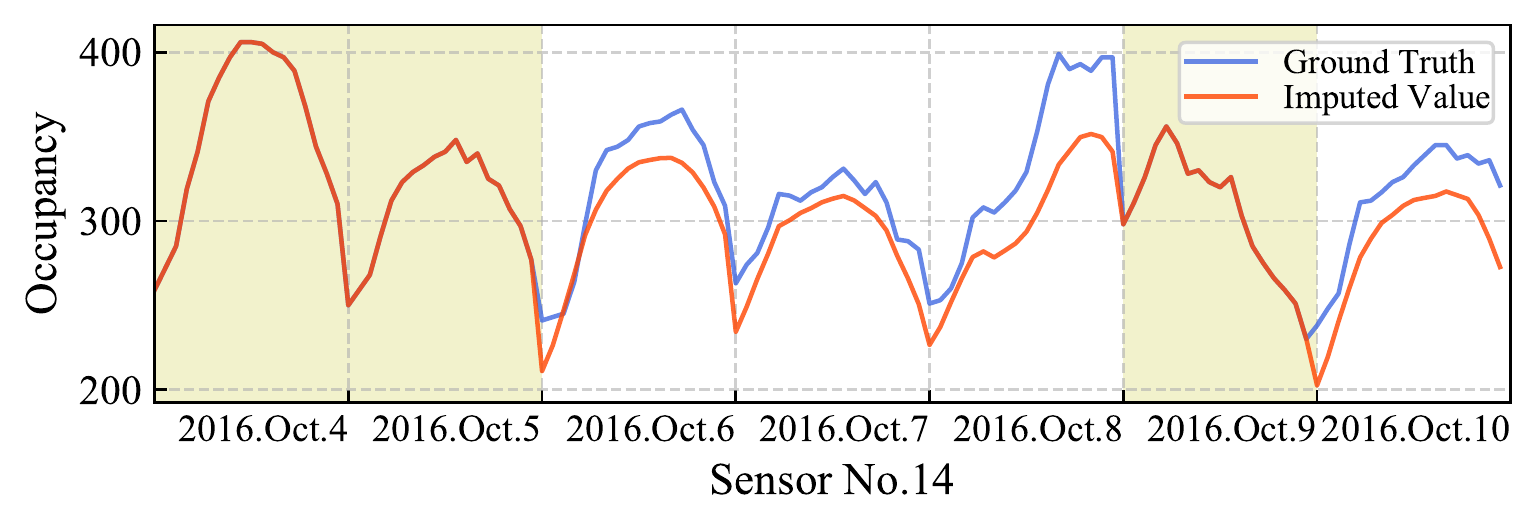}
}
\subfigure[Birmingham, $50\%$ FM-1]{
\centering
\includegraphics[scale=0.43]{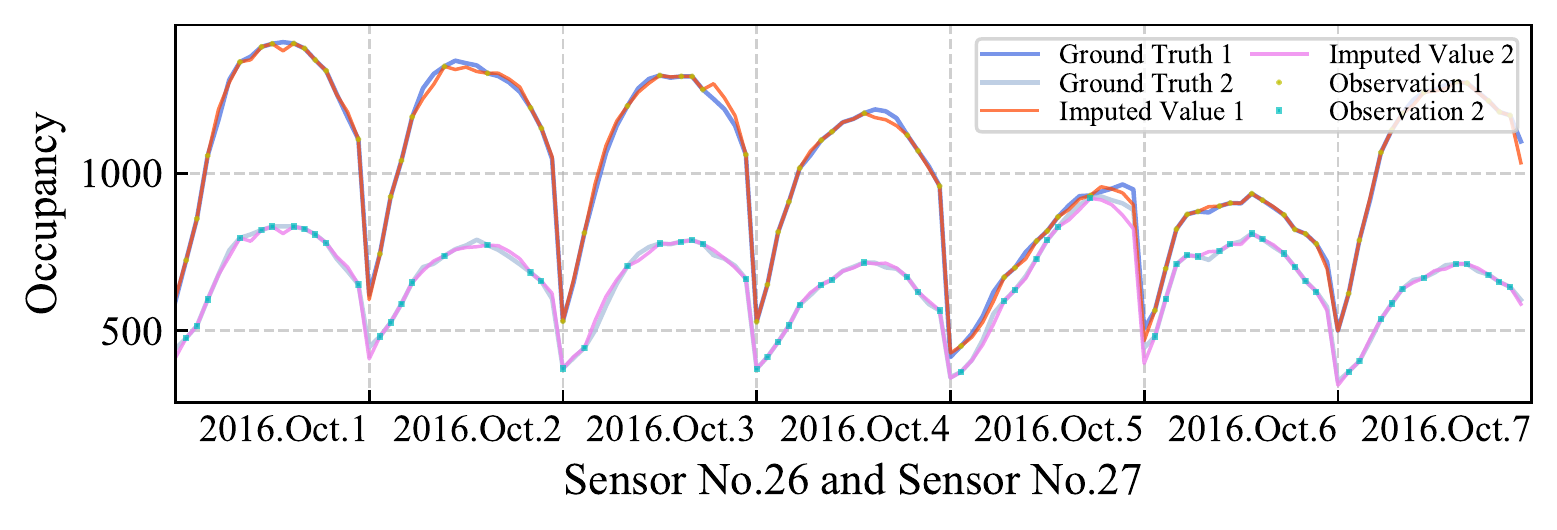}
}
\subfigure[Birmingham, $50\%$ FM-2]{
\centering
\includegraphics[scale=0.43]{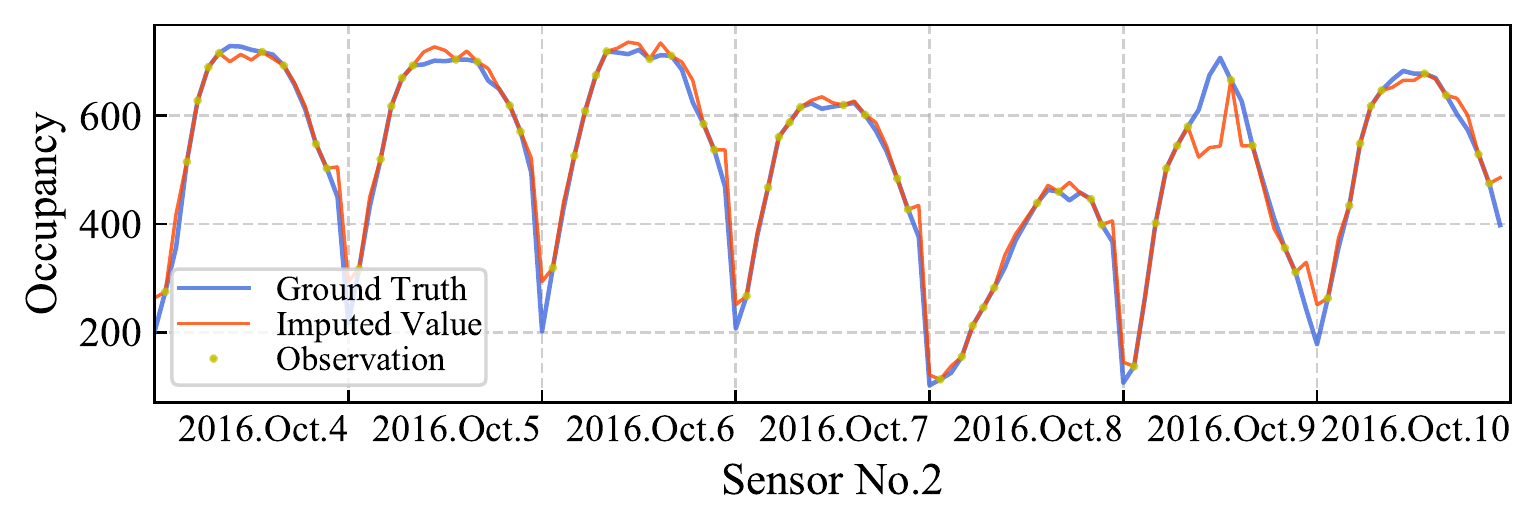}
}
\caption{Visualization examples of imputation for Birmingham occupancy data. (a) $50\%$ random missing. (b) $50\%$ fiber mode-0 missing. The green window denotes the partial observations in this missing case. (c) $50\%$ mode-1 missing. Two different sensors who share the same missing time points are plotted for comparison(d) $50\%$ mode-2 missing.}
\label{B-EX}
\end{figure}

For all these missing modes, we can clearly see that the structured fiber-like missing cases are much more difficult to impute than simple random missing. All models become less effective with the increasing of missing rate, but our model is more robust than other baselines generally. For Xuhui data, all models perform worst in 'spatial' FM-1 case. This is because without data of their neighbor traffic analysis zones, imputing OD flows has to exploit the time dimension, where shows slight within-day consistency \citep{cao2021day}, thus the information loss along this axis has a significant impact. One possible response is to establish the space topology as a additional prior knowledge, e.g., Graph Laplacian regularization in \citep{yang2021real}. 
As for Guangzhou and Portland data, since traffic flow dynamics like speed and volume have strong temporal dependencies, this FM-0 missing impedes accurate completion because there is no referable information on the whole missing days. The imputation results can be improved for this scenario by incorporating time series method, e.g. autoregressive model \citep{chen2021Autoregressive}. In the case of Birmingham data, both the 'temporal' FM-0 and FM-2 missing cases are more challenging because this parking data highly depends on day-to-day recurrence and consistency. This indicates that in element-wise and spatial missing cases, purely low-rank models can achieve adequate results for similar kinds of not complex signals; while in extreme temporal loss, LRTC models may generate biased turning points as can be seen in \minew{Figure} \ref{B-EX}. Therefore, structured data loss along each mode of tensor has different degrees of impacts on diverse kinds of data. This means that when dealing with practical missing traffic data, one should not treat different missing patterns equally and deploy a single model, instead, combining with prior knowledge and setting targeted parameters (e.g., $p$ values) will make more sense.

Moreover, the superiority of LRTC-TSpN over SpBCD denotes the effect of truncation and the decay strategy. When missing rate is high, only a \minew{few} \minew{proportions} of biggest singular values are truncated, but our model still achieves considerable improvement over the full Schatten $p$-norm model. Meanwhile, although the quadric penalty in \citep{gao2020robust} promotes its convergence, it also deteriorates the imputation precision, compared with our ADMM-GST algorithm.
The advantages of truncation and nonconvex approximation can also be found by comparison among LRTC-TNN, HaLTRC and TNN-DCT. When the missing rate is low, HaLRTC and TNN-DCT also show good performance, however, the two convex models become unstable as there is higher proportion of data unavailability. 
In general, both LRTC-TSpN and SpBCD have better performance than LRTC-TNN. This result clearly implies the dominant advantages of the generalized Schatten $p$-norm over the nuclear norm for capturing the underlying low-rank pattern, especially for these extreme missing cases. For example, according to \minew{Table} \ref{portland and Birmingham results} the proposed method achieves up to a 52\% MAE improvement compared to LRTC-TNN in Birmingham data even when the FM missing ratio is as high as 90\%.

\minew{Another common measurement is the multi-linear rank \citep{kolda2009tensor}. While in real-world traffic datasets, there exists noises that make the rank of data not numerically low (this will be discussed in \ref{Effect of truncation}). In other words, a few singular values are relatively very small but not zeros, and these residual values are normally included in the procedure of rank calculation. This impedes direct utilization of multi-linear rank to reflect true distribution of singular values. In this work, the 'low-rankness' is established in a sense that a low number of large singular values dominate the pattern of the matrix. In this context, comparing multi-linear ranks might not shed light on the evaluation of the imputed results.}

\begin{table}[htbp]
  \centering
  \caption{\minew{The Schatten $p$-norm of each unfolding matrix achieved by LRTC-TSpN on four datasets}}
  \footnotesize
    \begin{tabular}{cccc|cccc}
    \toprule
    \multicolumn{4}{c|}{Guangzhou} & \multicolumn{4}{c}{Portland} \\
    \midrule
    Missing pattern & Mode-0 & Mode-1 & Mode-2 & Missing pattern & Mode-0 & Mode-1 & Mode-2 \\
    \midrule
    \textbf{Complete} & \textbf{257.67}  & \textbf{386.59}  & \textbf{119.38}  & \textbf{Complete} & \textbf{217.76}  & \textbf{2096.79}  & \textbf{79.19}  \\
    RM    & 252.62  & 373.59  & 115.95  & RM    & 215.54  & 1821.11  & 77.08  \\
    FM-0   & 249.43  & 371.75  & 115.48  & FM-0   & 215.35  & 1834.63  & 78.59  \\
    FM-1   & 252.52  & 373.85  & 115.68  & FM-1   & 215.72  & 1826.31  & 77.10  \\
    FM-2   & 252.62  & 373.59  & 115.95  & FM-2   & 215.51  & 1822.70  & 77.29  \\
    \midrule
    \multicolumn{4}{c|}{Birmingham} & \multicolumn{4}{c}{Xuhui} \\
    \midrule
    Missing pattern & Mode-0 & Mode-1 & Mode-2 & Missing pattern & Mode-0 & Mode-1 & Mode-2 \\
    \midrule
    \textbf{Complete} & \textbf{40.68}  & \textbf{62.62}  & \textbf{147.73}  & \textbf{Complete} & \textbf{103.37}  & \textbf{86.09}  & \textbf{26.86}  \\
    RM    & 37.58  & 60.46  & 137.06  & RM    & 99.18  & 82.87  & 26.32  \\
    FM-0   & 37.29  & 59.58  & 131.75  & FM-0   & 99.94  & 80.38  & 26.40  \\
    FM-1   & 37.87  & 55.96  & 131.48  & FM-1   & 97.49  & 82.85  & 25.94  \\
    FM-2   & 37.32  & 59.73  & 137.59  & FM-2   & 98.99  & 82.69  & 26.37  \\
    \bottomrule
    \multicolumn{5}{l}{\scriptsize{Results of the complete tensor are bold marked.}}
    \end{tabular}%
  \label{norm}%
\end{table}%

\minew{
Instead, we give the achieved Schatten $p$-norm calculated by Eq. \eqref{sp_matrix} of each mode-$n$ unfolding matrix in Table \ref{norm}, taking the $50\%$ missing rate as an example. As our objective is to minimize the sum of unfolding-based norm, the norm of all unfoldings in all instances reported in Table \ref{norm} are lower but close to the complete ones, demonstrating the effectiveness of our method.}

\subsubsection{Model comparisons in typical scenarios}
The Guangzhou and Portland traffic flow datasets are both million-level, which are relatively large-scale for many transportation applications. Therefore, the advantages of our method over other baselines are not remarkable due to the size of data. To further evaluate the performance in more intractable situations, in this subsection we conduct imputation tasks on small subsets of the two traffic flow datasets mentioned above with typical missing scenarios (i.e., fiber-like missing cases with loss rate higher than $75\%$), and they are briefly described as:
\begin{itemize}
    \item \textbf{Guangzhou-small}: Speed data with the first 50 locations and the first 15 days. The size is $(144\times50\times15)$.
    \item \textbf{Portland-small}: Volume data with the first 80 locations and the first 15 days. The size is $(96\times80\times15)$.
\end{itemize}

In these tests, we set the missing rate as $75\%, 85\%$, and $95\%$ for three FM-$n$ missing situations. For the two small-scale traffic flow datasets, the latent patterns in the origin data are more prominent and more susceptible to noise and missing values, so that the variance may become larger than the results of whole datasets. We conduct these experiments in five times and report the average results, see \minew{Figure} \ref{model compare}. The parameter settings for this subsection are given in \minew{Table} \ref{params subsec}.

\begin{table}[!htb]
  \centering
  \caption{Parameter settings for Guangzhou-small and Portland-small data}
  \footnotesize
    \begin{tabular}{c|ccc|ccc|ccc}
    \toprule
          & \multicolumn{3}{c|}{FM-0} & \multicolumn{3}{c|}{FM-1} & \multicolumn{3}{c}{FM-2} \\
    \midrule
    Dataset & $p$     & $\theta_0$ & $\beta$  & $p$     & $\theta_0$ & $\beta$  & $p$     & $\theta_0$ & $\beta$ \\
    \midrule
    Guangzhou-small & 0.6   & 0.05  & 2     & 0.7   & 0.05  & 2     & 0.8   & 0.05  & 2 \\
    \midrule
    Portland-small & 0.8   & 0.05  & 4     & 0.8   & 0.05  & 2     & 0.6   & 0.05  & 4 \\
    \bottomrule
    \end{tabular}%
    \label{params subsec}
\end{table}%

% \begin{figure}[!htb]
% \centering
% \subfigure[Guangzhou-small, FM-0]{
% \centering
% \includegraphics[scale=0.7]{model compare-G0.pdf}
% }
% \subfigure[Guangzhou-small, FM-2]{
% \centering
% \includegraphics[scale=0.7]{model compare-G2.pdf}
% }
% \subfigure[Portland-small, FM-0]{
% \centering
% \includegraphics[scale=0.7]{model compare-P0.pdf}
% }
% \subfigure[Portland-small, FM-1]{
% \centering
% \includegraphics[scale=0.7]{model compare-P1.pdf}
% }

% \caption{\minew{Model comparison in Guangzhou-small and Portland-small data with different fiber-like missing situations.}}
% \label{model compare}
% \end{figure}

\begin{figure}[!ht]
  \centering
  \includegraphics[scale=0.7]{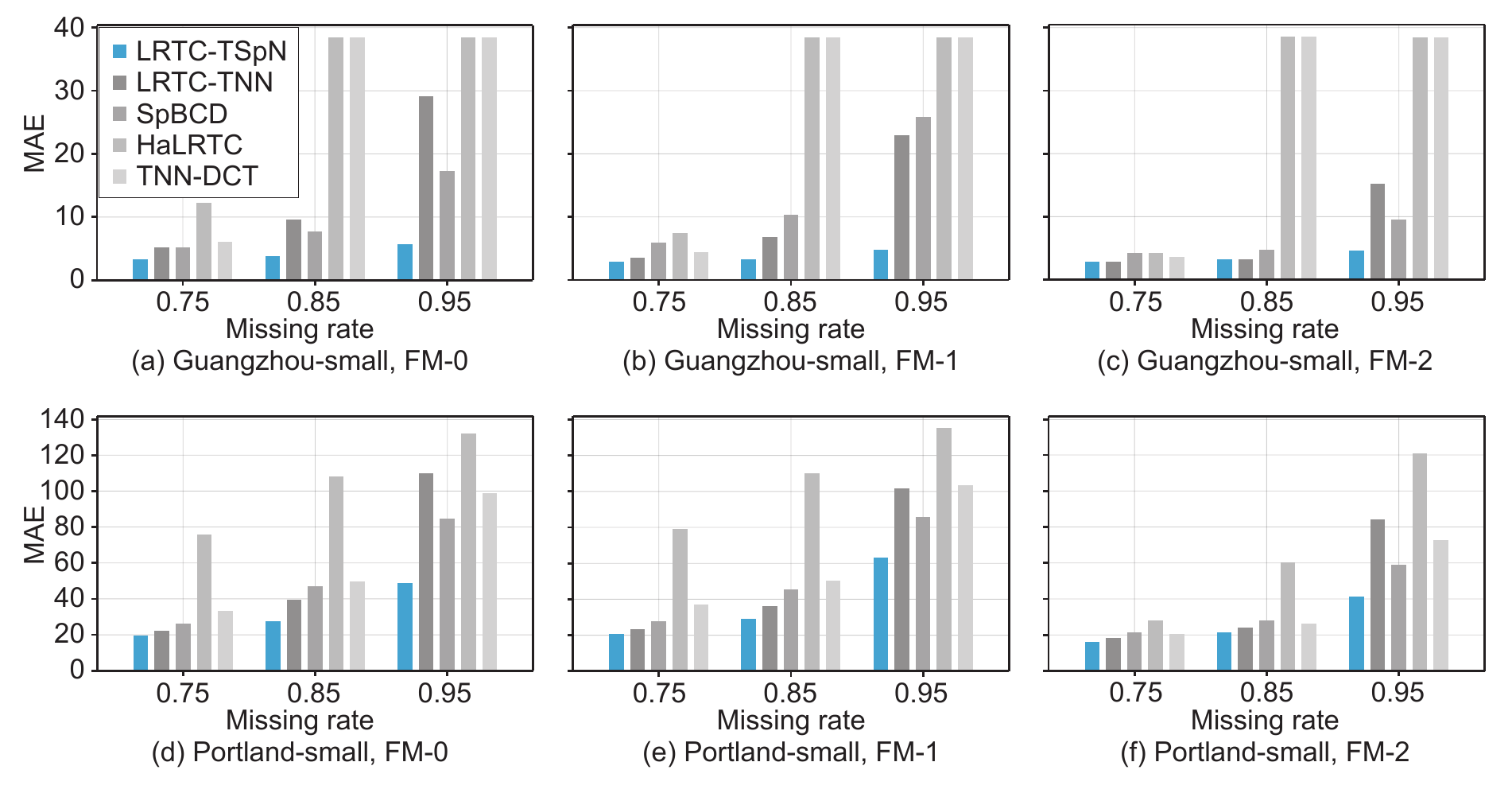}
  \caption{Model comparison in Guangzhou-small and Portland-small data with different fiber-like missing situations.}
  \label{model compare}
\end{figure}

From \minew{Figure} \ref{model compare}, it is obvious to find that the proposed method tends to be robust and much more stable than other approaches. The two convex model HaLRTC and TNN-DCT appear to have numerical instability when the missing ratio is as high as $85\%$. In this context, the superiority of LRTC-TSpN with respect to the accuracy is more distinct than in the large-scale datasets. The notable performance in relative small-size data with high proportion of data loss can be mainly ascribed to the truncation decay scheme and the strictly non-convexity \minew{characteristic} of TSpN that substitutes for a better representation ability.

%\subsubsection{Computational performance}
Apart from imputation precision, we also compare the computational performance with other baselines by recording the CPU running time, total iteration times and the convergence curves. The recorded results on Guangzhou-small data are given in \minew{Table} \ref{computational performances1} and \minew{Table} \ref{computational performances2}. Note that some results are invalid for HaLRTC and TNN-DCT since they are actually disabled to work in these extreme scenarios (see \minew{Figure} \ref{model compare}(a) and (b)), so we do not take them into considerations.

\begin{table}[!htbp]
  \centering
  \caption{Computational performances on Guangzhou-small data}
  \label{computational performances1}
  \setlength\tabcolsep{2pt}
  \footnotesize
    \begin{tabular}{c|c|cccccccccc}
    \toprule
    \multicolumn{2}{c|}{Models} & \multicolumn{2}{c}{LRTC-TNN} & \multicolumn{2}{c}{LRTC-TSpN} & \multicolumn{2}{c}{SpBCD} & \multicolumn{2}{c}{HaLRTC} & \multicolumn{2}{c}{TNN-DCT} \\
\cmidrule{3-12}    Missing cases & Missing rates & CPU time (s)& Epoch & CPU time (s)& Epoch & CPU time (s)& Epoch & CPU time (s)& Epoch & CPU time (s)& Epoch \\
    \midrule
    \multirow{3}[2]{*}{FM-0} & 75\%  & 12.4  & 185   & 7     & \textbf{56}    & 14.7  & 199   & \textbf{4.6}   & 76    & 6.6   & 119 \\
          & 85\%  & 12.5  & 184   & \textbf{7.1}   & \textbf{57}    & 15.5  & 199   & \st{0.1}   & \st{1}     & \st{0.2}   & \st{3} \\
          & 95\%  & 12.3  & 183   & \textbf{6.8}   & \textbf{56}    & 15.5  & 199   & \st{0.1}   & \st{1}     & \st{0.2}   & \st{3} \\
    \midrule
    \multirow{3}[2]{*}{FM-1} & 75\%  & 10.9  & 155   & 7.8   & \textbf{63}    & 14.3  & 199   & \textbf{3.9}   & 66    & 6.6   & 120 \\
          & 85\%  & 11.4  & 160   & \textbf{8.3}   & \textbf{66}    & 14.6  & 199   & \st{0.1}   & \st{1}     & \st{0.2}   & \st{3} \\
          & 95\%  & 11.8  & 171   & \textbf{7.8}   & \textbf{64}    & 14.8  & 199   & \st{0.1}   & \st{1}     & \st{0.2}   & \st{3} \\
    \midrule
    \multirow{3}[2]{*}{FM-2} & 75\%  & 10.6  & 156   & 8.2   & \textbf{64}    & 13.6  & 199   & \textbf{3.9}   & 65    & 6.6   & 120 \\
          & 85\%  & 12.1  & 163   & \textbf{8.6}   & \textbf{65}    & 15.9  & 199   & \st{0.1}   & \st{1}     & \st{0.2}   & \st{3} \\
          & 95\%  & 13.8  & 185   & \textbf{9.1}   & \textbf{68}    & 16.5  & 199   & \st{0.1}   & \st{1}     & \st{0.2}   & \st{3} \\
    \bottomrule
    \multicolumn{5}{l}{\scriptsize{Best results are bold marked and invalid results are marked with a strikeout.}}
    \end{tabular}%
\end{table}%

\begin{table}[!htbp]
  \centering
  \caption{Computational performances on Portland-small data}
  \label{computational performances2}
  \setlength\tabcolsep{2pt}
  \footnotesize
    \begin{tabular}{c|c|cccccccccc}
    \toprule
    \multicolumn{2}{c|}{Models} & \multicolumn{2}{c}{LRTC-TNN} & \multicolumn{2}{c}{LRTC-TSpN} & \multicolumn{2}{c}{SpBCD} & \multicolumn{2}{c}{HaLRTC} & \multicolumn{2}{c}{TNN-DCT} \\
\cmidrule{3-12}    Missing cases & Missing rates & CPU time (s)& Epoch & CPU time (s)& Epoch & CPU time (s)& Epoch & CPU time (s)& Epoch & CPU time (s)& Epoch \\
    \midrule
    \multirow{3}[2]{*}{FM-0} & 75\%  & 10    & 131   & 9.2    & \textbf{82}    & 15.9  & 199   & \textbf{7.6}   & 108   & 8.6   & 102 \\
          & 85\%  & 10.4  & 117   & 9.3    & \textbf{82}    & 18.3  & 199   & \textbf{7.2}   & 98    & 8.3   & 111 \\
          & 95\%  & 10    & 123   & 8.5  & 76    & 17.1  & 199   & \textbf{5.3}   & \textbf{73}    & 6.7   & 117 \\
    \midrule
    \multirow{3}[2]{*}{FM-1} & 75\%  & 10.8  & 134   & 9.5  & \textbf{86}    & 16.3  & 199   & 7.7   & 110   & \textbf{7.1}   & 103 \\
          & 85\%  & 10.5  & 128   & 9.6  & \textbf{88}    & 17.3  & 199   & \textbf{7.2}   & 101   & 7.6   & 112 \\
          & 95\%  & 12.9  & 165   & 9.3  & 83    & 17    & 199   & \textbf{5.5}   & \textbf{70}    & 8.1   & 124 \\
    \midrule
    \multirow{3}[2]{*}{FM-2} & 75\%  & 10    & 128   & 10.1  & \textbf{93}    & 15.2  & 199   & \textbf{8}     & 109   & 8.3   & 101 \\
          & 85\%  & 10.6  & 133   & 10.2  & \textbf{96}    & 17.7  & 199   & 8.5   & 119   & \textbf{7.8}   & 112 \\
          & 95\%  & 12.1  & 145   & 10.4  & 97    & 16.3  & 199   & \textbf{7.5}   & \textbf{96}    & 9.8   & 138 \\
    \bottomrule
    \multicolumn{5}{l}{\scriptsize{Best results are bold marked and invalid results are marked with a strikeout.}}
    \end{tabular}%
\end{table}%

From the two tables we can find that our method has good computational performance: the proposed LRTC-TSpN can converge with fewer iterations in most testing situations and is competitive in running time consumption, so that it is cost-effective for saving computational sources, although the GST algorithm could lead to some more computing burden. At the meantime, SpBCD reaches the maximum iterations in all cases because the block coordinate descend optimization framework becomes hard to calculate the next proximal variable when facing such a nonconvex problem.
HaLRTC also displays promising performances in some cases owing to the simplified convex model, however, both TNN-DCT and HaLRTC even do not operate when the fiber-missing rate is over $75\%$ in the speed dataset.
This improvement over other baselines mainly attributes to the flexibility of model parameters and appropriate iteration strategy with theoretical optimum and convergence guaranteed.

We then take the $95\%$ missing rate as an example to plot the curve of MAE versus iterations, which gives a straightforward reflect of the convergence capability of our method, as shown in \minew{Figure} \ref{convergence curve}. %(see more in \nameref{Appendix B.} ).

\begin{figure}[!htbp]
\centering
\subfigure[Guangzhou-small, $95\%$ FM-0]{
\centering
\includegraphics[scale=0.45]{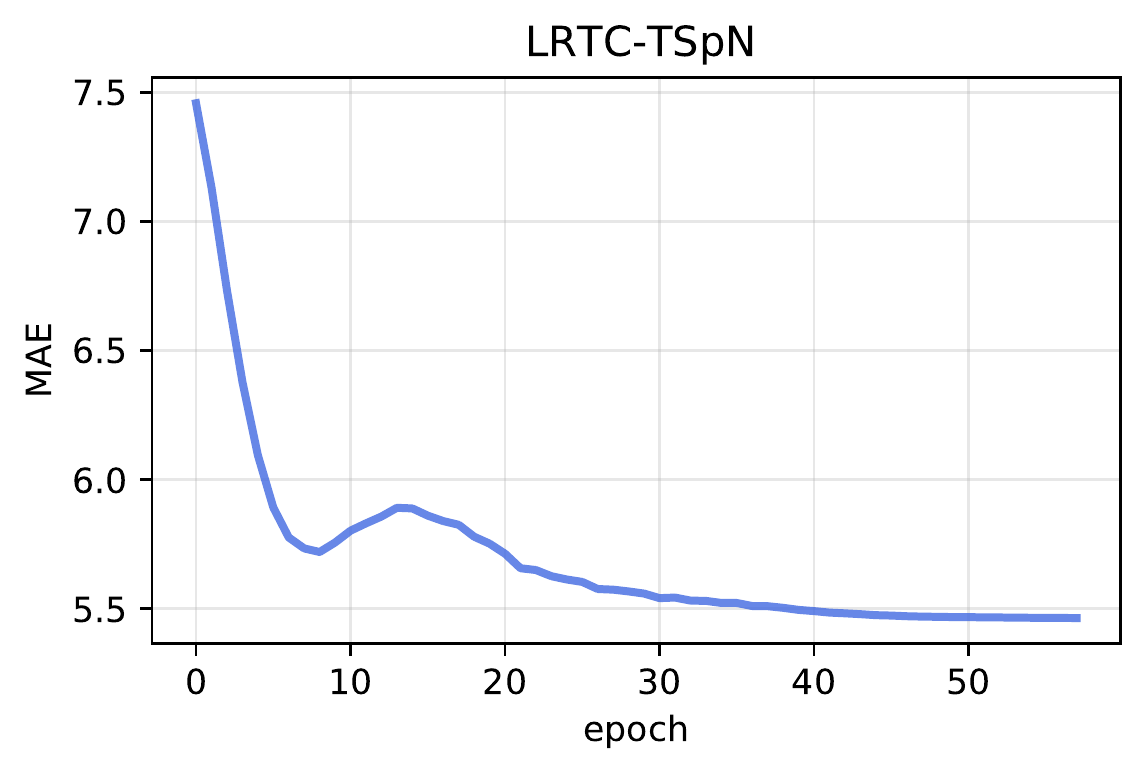}
}
\subfigure[Guangzhou-small, $95\%$ FM-1]{
\centering
\includegraphics[scale=0.45]{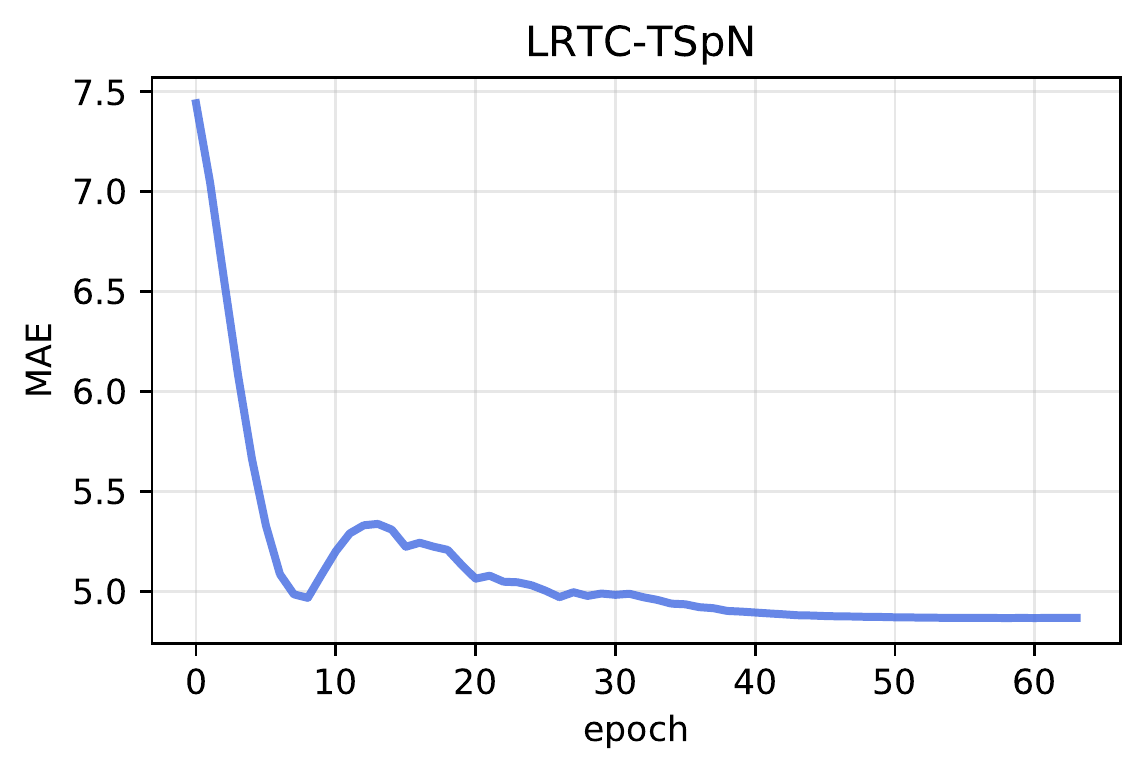}
}
\subfigure[Guangzhou-small, $95\%$ FM-2]{
\centering
\includegraphics[scale=0.45]{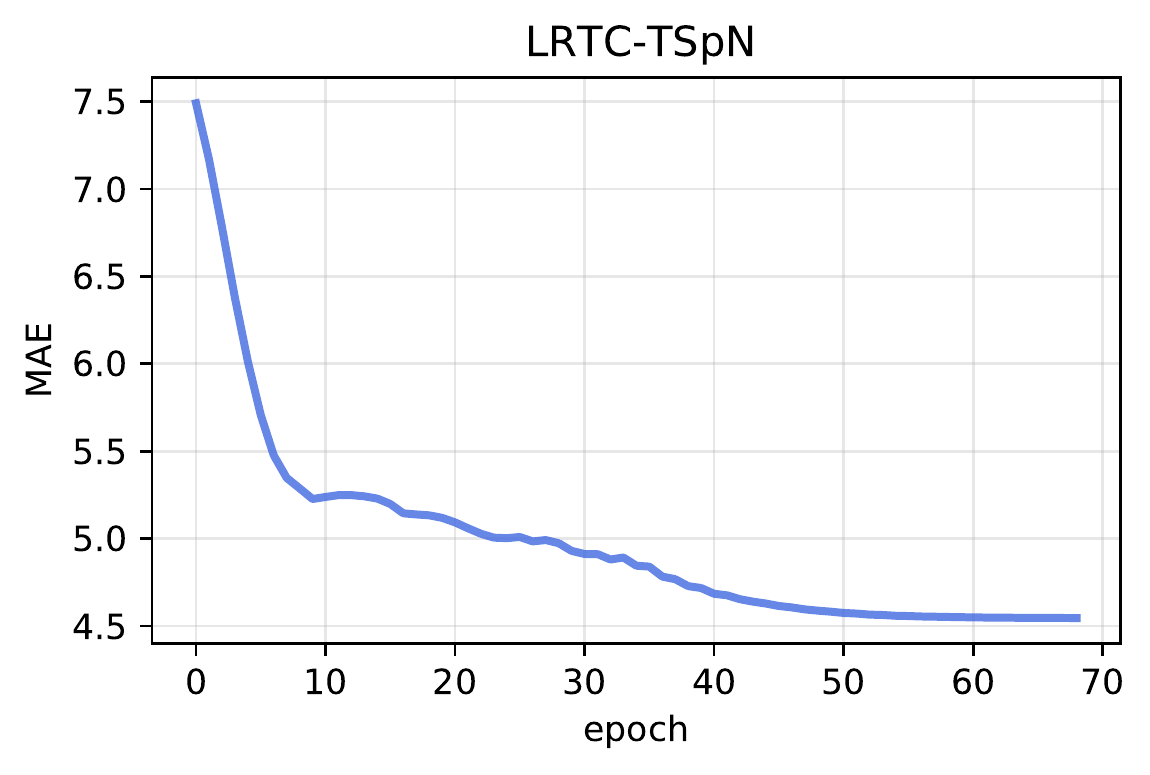}
}
\subfigure[Portland-small, $95\%$ FM-0]{
\centering
\includegraphics[scale=0.45]{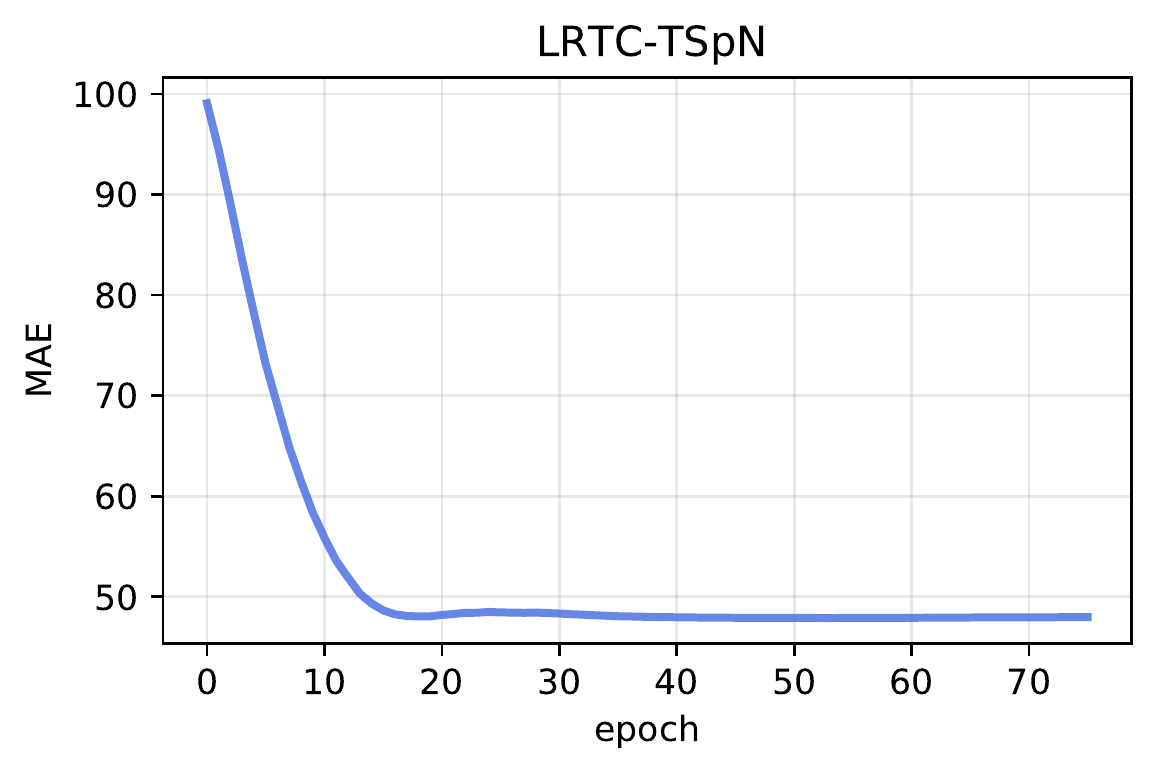}
}
\subfigure[Portland-small, $95\%$ FM-1]{
\centering
\includegraphics[scale=0.45]{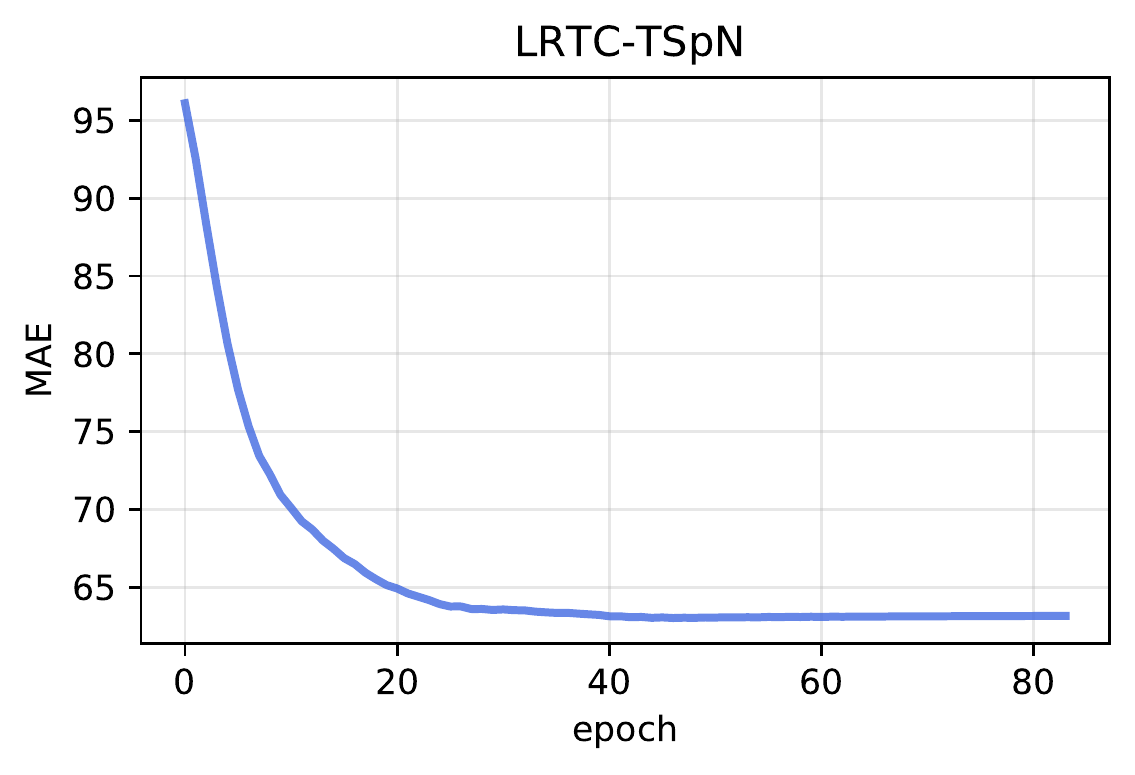}
}
\subfigure[Portland-small, $95\%$ FM-2]{
\centering
\includegraphics[scale=0.45]{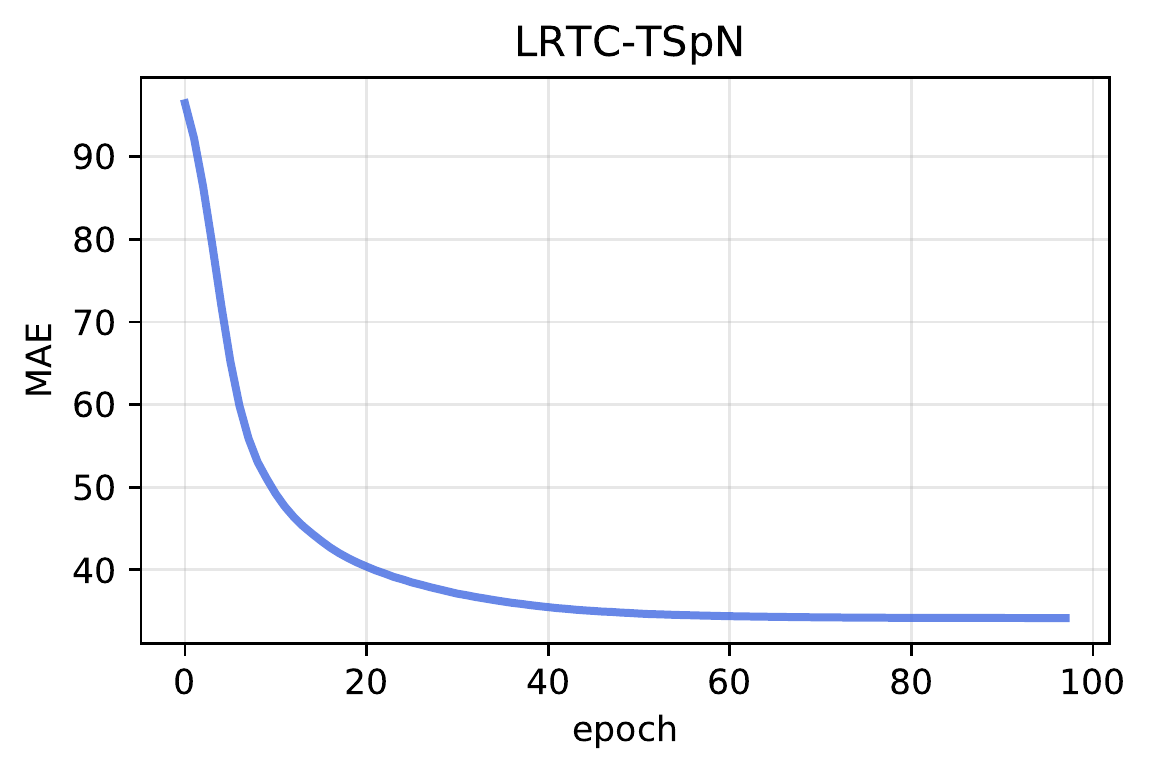}
}
\caption{The MAE variation curve with the epoch increasing until convergence.}
\label{convergence curve}
\end{figure}

These figures denote that the proposed \minew{Algorithm} \ref{algorithm2} generally decreases the objective function (or error measurements) stably and shows trivial fluctuation, even in such extreme cases. This means that the ADMM scheme can still be effective, in spite of a hard nonconvex problem, which is consistent with the conclusion in \minew{Subsection} \ref{Convergence analysis}. Therefore, our LRTC-TSpN method not only aims to improve the overall imputation accuracy \minew{to varying degrees}, but also provides a viable option for robust and efficient implementations of flexible scenarios.

\subsubsection{Hyper-parameter tuning and sensitivity analysis}
\label{hpt}
Compared with previous works, the proposed method has three additional parameters, i.e., truncation $\theta_0, \beta$ and Schatten $p$ value. To survey the impact of these hyper-parameters on model and give some insights into the hyper-parameter tuning procedures, in this subsection we investigate the influences of parameters on LRTC-TSpN on Birmingham data.

%\subsubsection{Sensitivity analysis of p value}
We keep the truncation parameters unchanged and set $p$ varying from 0.1 to 0.9 for four missing scenarios. Figure \ref{p influence} reports the MAE of LRTC-TSpN with different $p$ values under different missing rates. %(only parts of the results are plotted here to save space, see more in \nameref{Appendix B.} ). 

\begin{figure}[!htb]
\centering
\subfigure[Birmingham, RM]{
\centering
\includegraphics[scale=0.5]{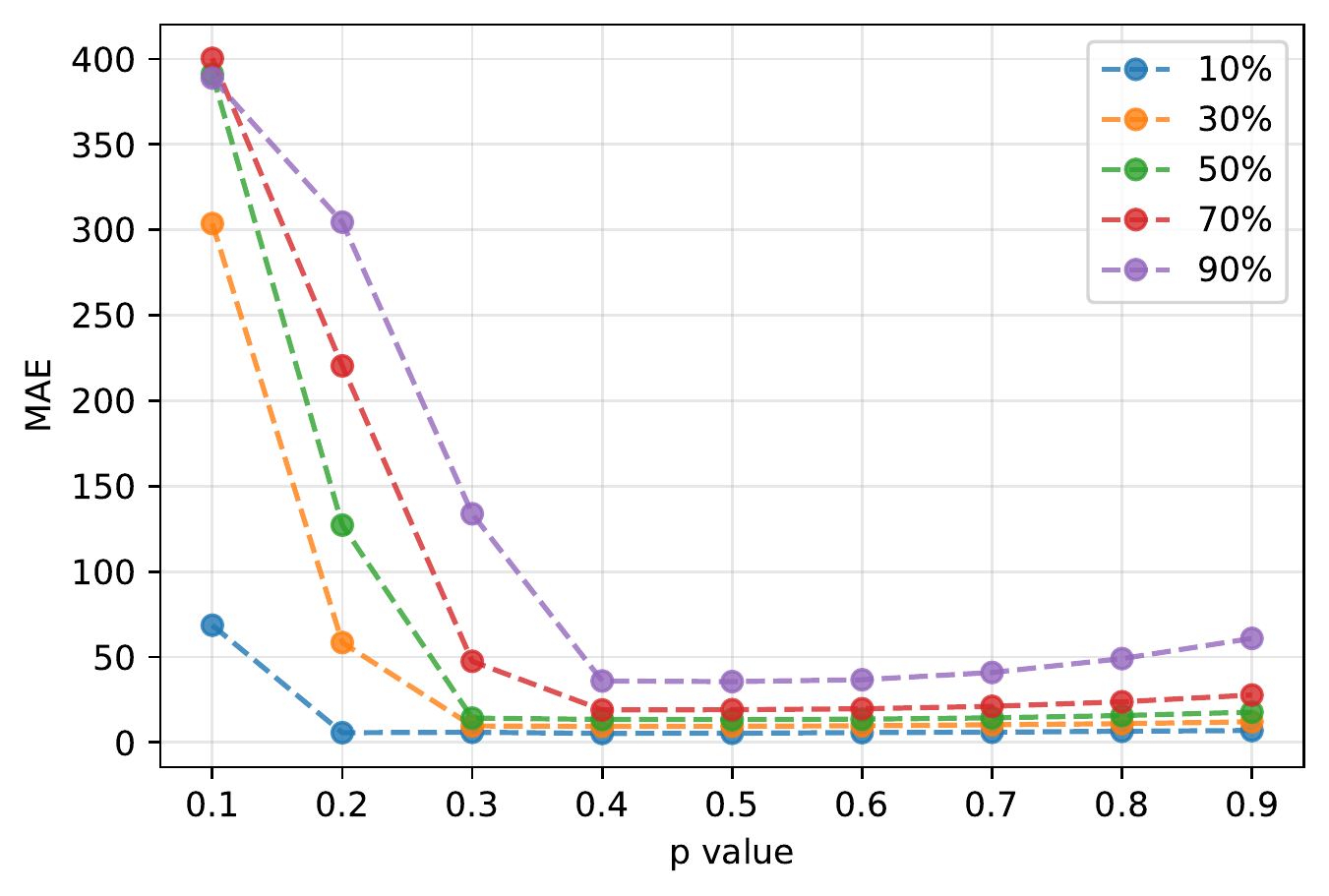}
}
\subfigure[Birmingham, FM-0]{
\centering
\includegraphics[scale=0.5]{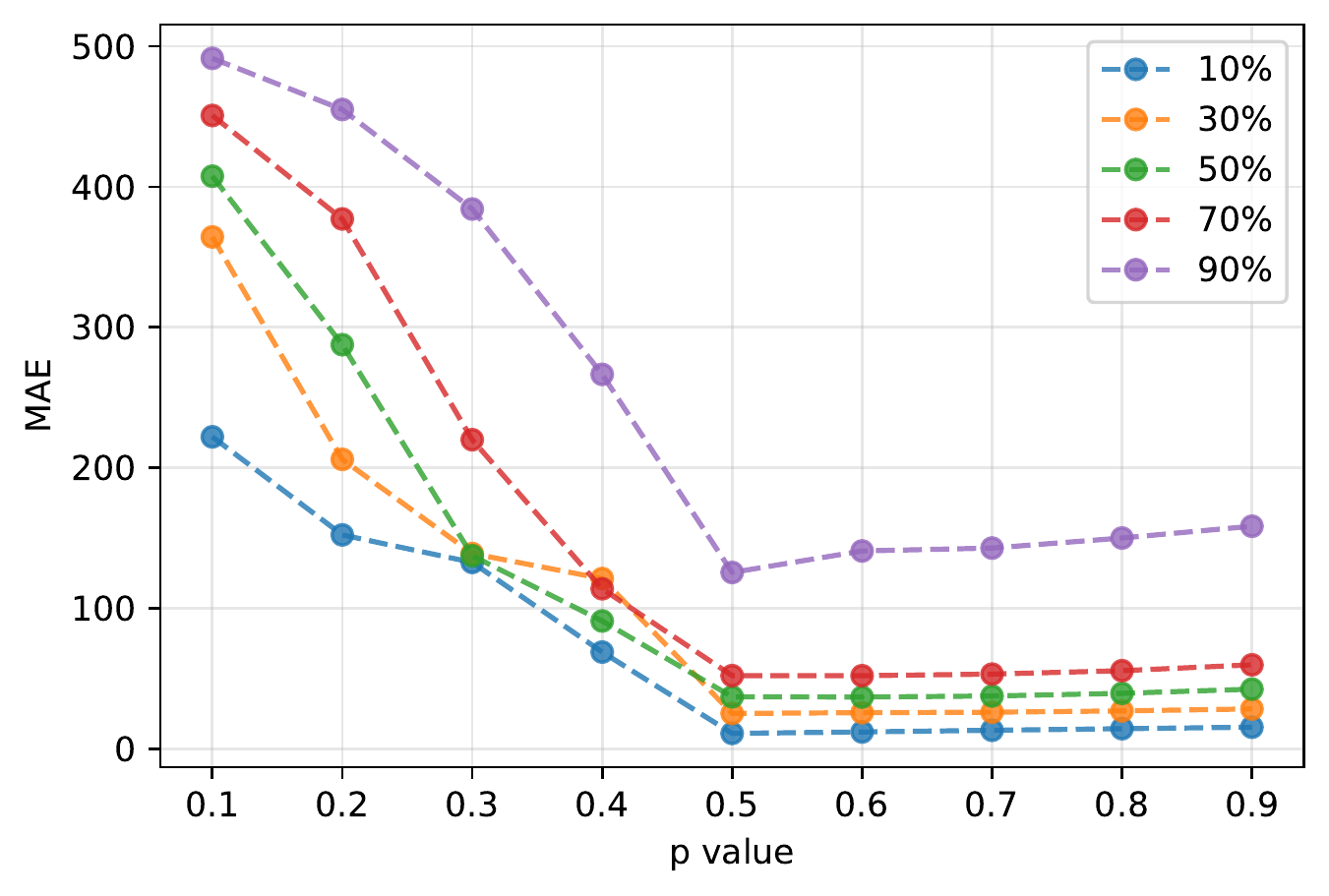}
}
\subfigure[Birmingham, FM-1]{
\centering
\includegraphics[scale=0.5]{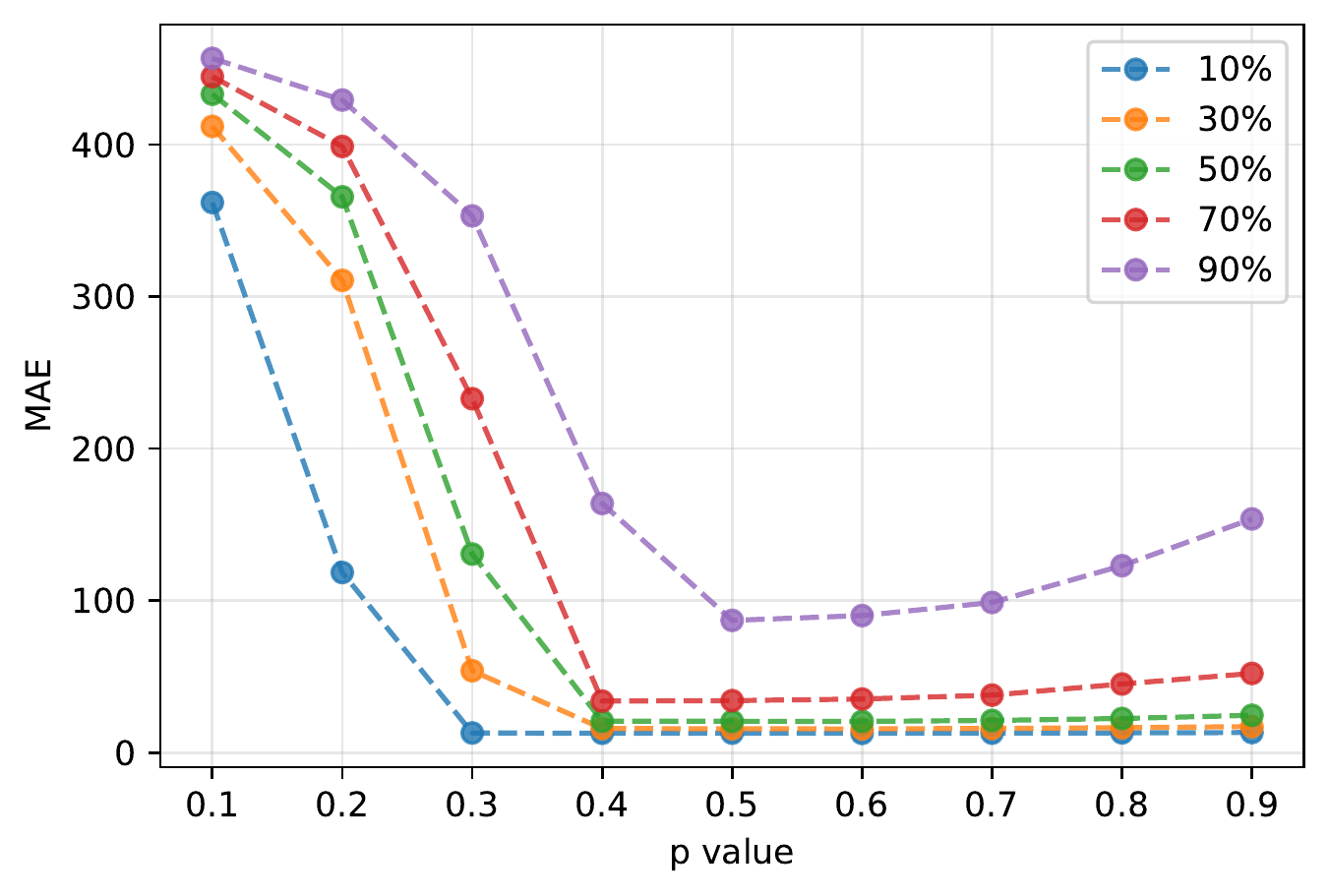}
}
\subfigure[Birmingham, FM-2]{
\centering
\includegraphics[scale=0.5]{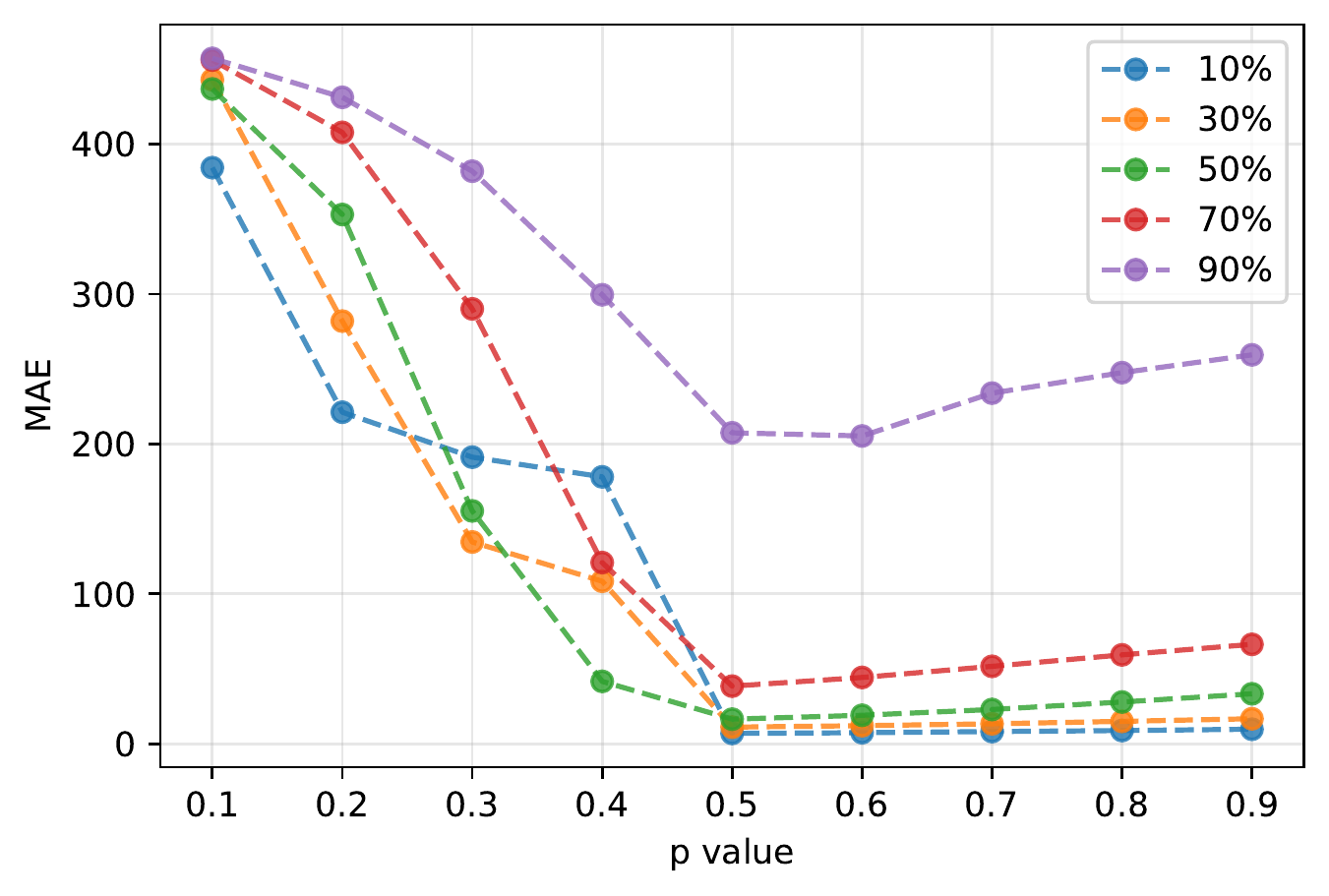}
}
\caption{Model performance under varying $p$ values on Birmingham dataset with different missing rates.}
\label{p influence}
\end{figure}

From this figure we observe that curves of different degrees of missing have similar trends with respect to $p$ values, and in this dataset the best performance occurs around $p=0.5$. When $p$ is less than this turning point, the MAE decreases dramatically; while there are slight impacts when $p$ exceeds it. At the same time, $p$ controls the level of nonconvexity of model in some extent. This enlightens us that we can find a proper $p$ value in the range after the inflection point to balance the accuracy and complexity. Look at another facet of this figure, when curves of different missing rates approximately have the same inflection point like \ref{p influence}(a), one can set a certain $p$ for all imputation tasks. On the contrary, one can improve the performance by dynamically increasing it like \ref{p influence}(b).

Figure \ref{p influence2} \minew{shows} the results cross different missing patterns by selecting a fixed missing rate (e.g., $50\%$ and $80\%$ are used herein). Similar conclusions can be drawn as in \minew{Figure} \ref{p influence}. For a simple realization, one can set a fixed $p$ value for all missing patterns.

\begin{figure}[!htb]
%\begin{minipage}[t]{0.48\textwidth}
\centering
\subfigure[Birmingham, $50\%$ missing]{
\centering
\includegraphics[scale=0.5]{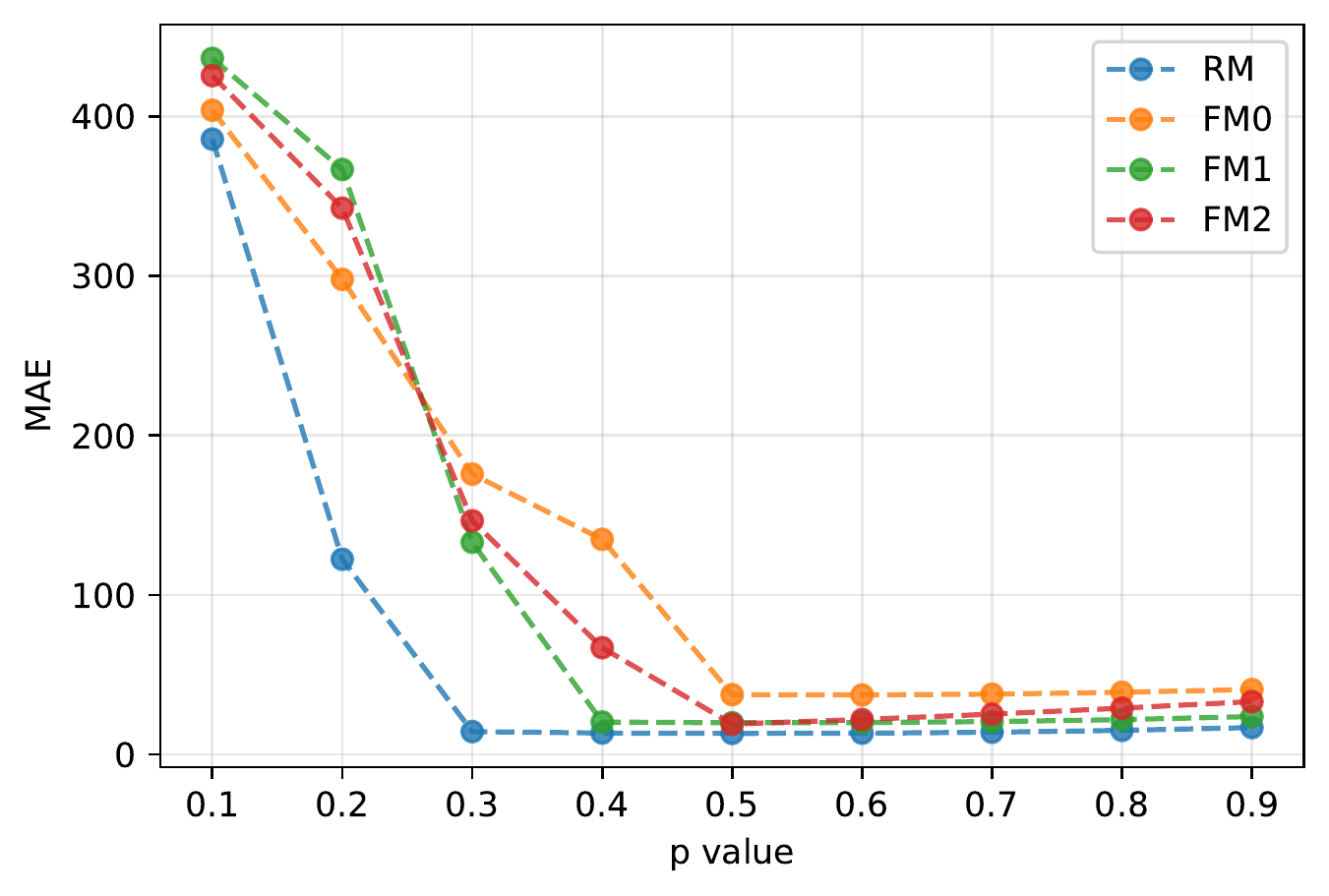}
}
%\end{minipage}
%\begin{minipage}[t]{0.48\textwidth}
\centering
\subfigure[Birmingham, $80\%$ missing]{
\centering
\includegraphics[scale=0.5]{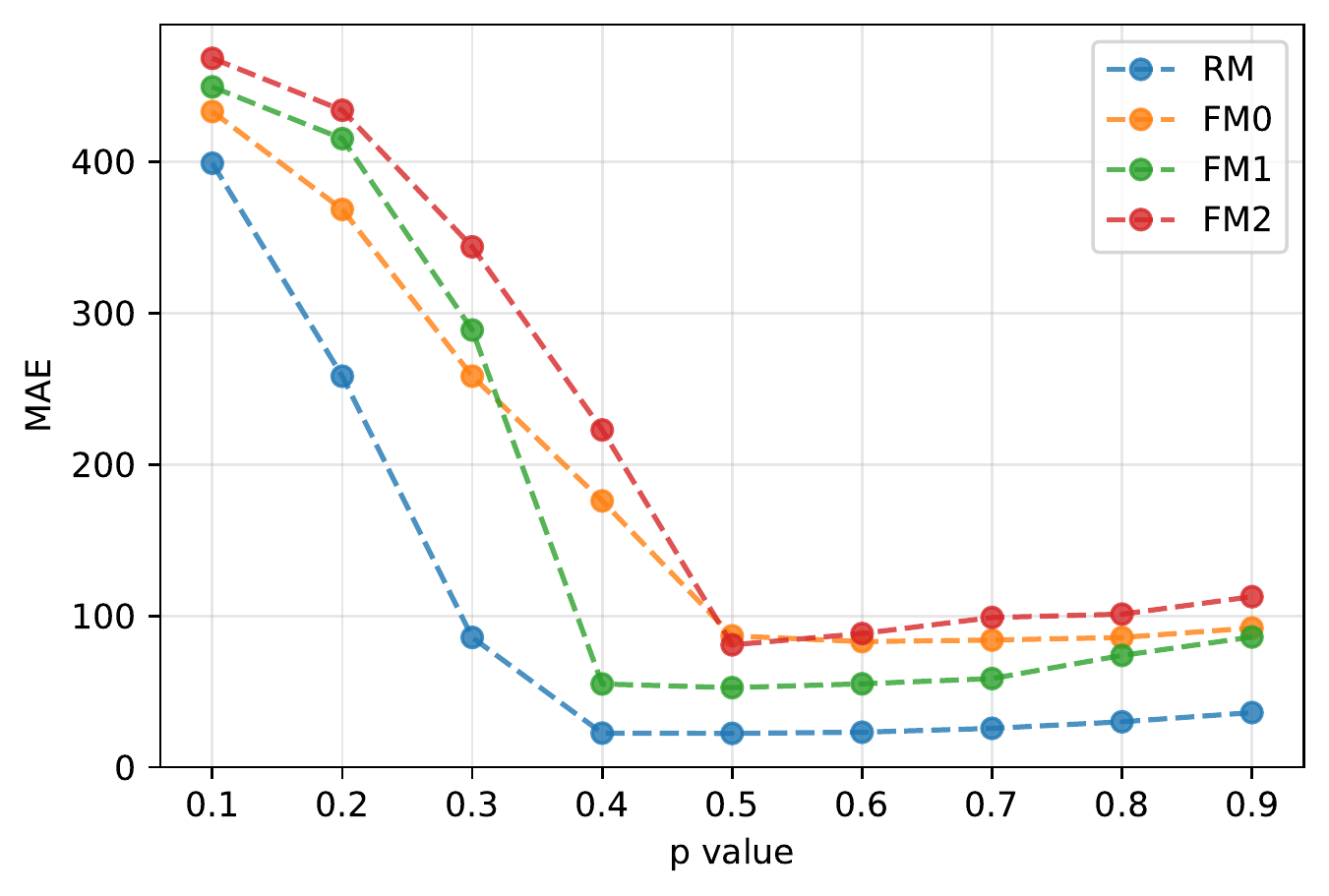}
}
%\end{minipage}
\caption{Model performance under different $p$ values on Birmingham dataset with different missing patterns.}
\label{p influence2}
\end{figure}

%\subsubsection{Sensitivity analysis of truncation parameters}
We survey the truncation $\theta$ parameter settings by testing the parameter combinations s1-s8 given in \minew{Table} \ref{theta setting}. In each combination we conduct imputation tasks with missing rates ranging from $10\%$ to $80\%$ ($10\%$ as step) on four missing patterns.
Figure \ref{boxplot} shows the box plot for the performances of 
each parameter setting group under 8 missing scenarios of each missing pattern. Each error box in \minew{Figure} \ref{boxplot} reports the imputation MAE of 8 missing rate scenarios with a specific $\theta$ setting.

\begin{table}[htbp]
  \centering
  \caption{Parameter settings for $\theta$}
  \label{theta setting}%
  \footnotesize
    \begin{tabular}{c|c|cccccccc}
    \toprule
    \multicolumn{2}{c|}{Parameter setting groups} & s1    & s2    & s3    & s4    & s5    & s6    & s7    & s8 \\
    \midrule
    \multirow{2}[2]{*}{RM} & $\theta_0$ & 0.2   & 0.2   & 0.2   & 0.2   & 0.1   & 0.3   & 0.4   & 0.5 \\
          & $\beta$  & 1     & 2     & 3     & 4     & 4     & 4     & 4     & 4 \\
    \midrule
    \multirow{2}[2]{*}{FM-0} & $\theta_0$ & 0.2   & 0.2   & 0.2   & 0.2   & 0.1   & 0.3   & 0.4   & 0.5 \\
          & $\beta$  & 2.5   & 3     & 3.5   & 4     & 4     & 4     & 4     & 4 \\
    \midrule
    \multirow{2}[2]{*}{FM-1} & $\theta_0$ & 0.2   & 0.2   & 0.2   & 0.2   & 0.1   & 0.3   & 0.4   & 0.5 \\
          & $\beta$  & 2.5   & 3     & 3.5   & 4     & 4     & 4     & 4     & 4 \\
    \midrule
    \multirow{2}[2]{*}{FM-2} & $\theta_0$ & 0.1   & 0.1   & 0.1   & 0.1   & 0.05  & 0.15  & 0.2   & 0.25 \\
          & $\beta$  & 1     & 1.5   & 2     & 2.5   & 2     & 2     & 2     & 2 \\
    \bottomrule
    \end{tabular}%
\end{table}%

\begin{figure}[!htb]
\centering
\subfigure[Birmingham, RM]{
\centering
\includegraphics[scale=0.6]{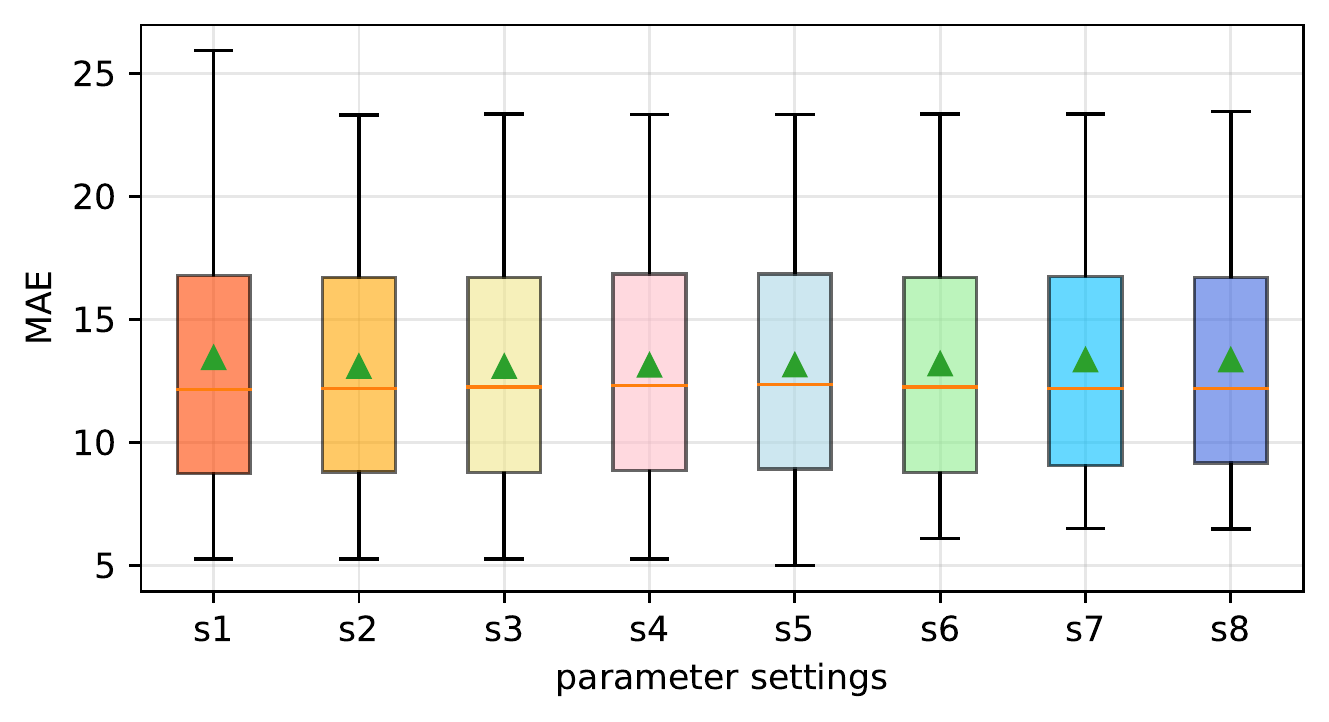}
}
\subfigure[Birmingham, FM-0]{
\centering
\includegraphics[scale=0.6]{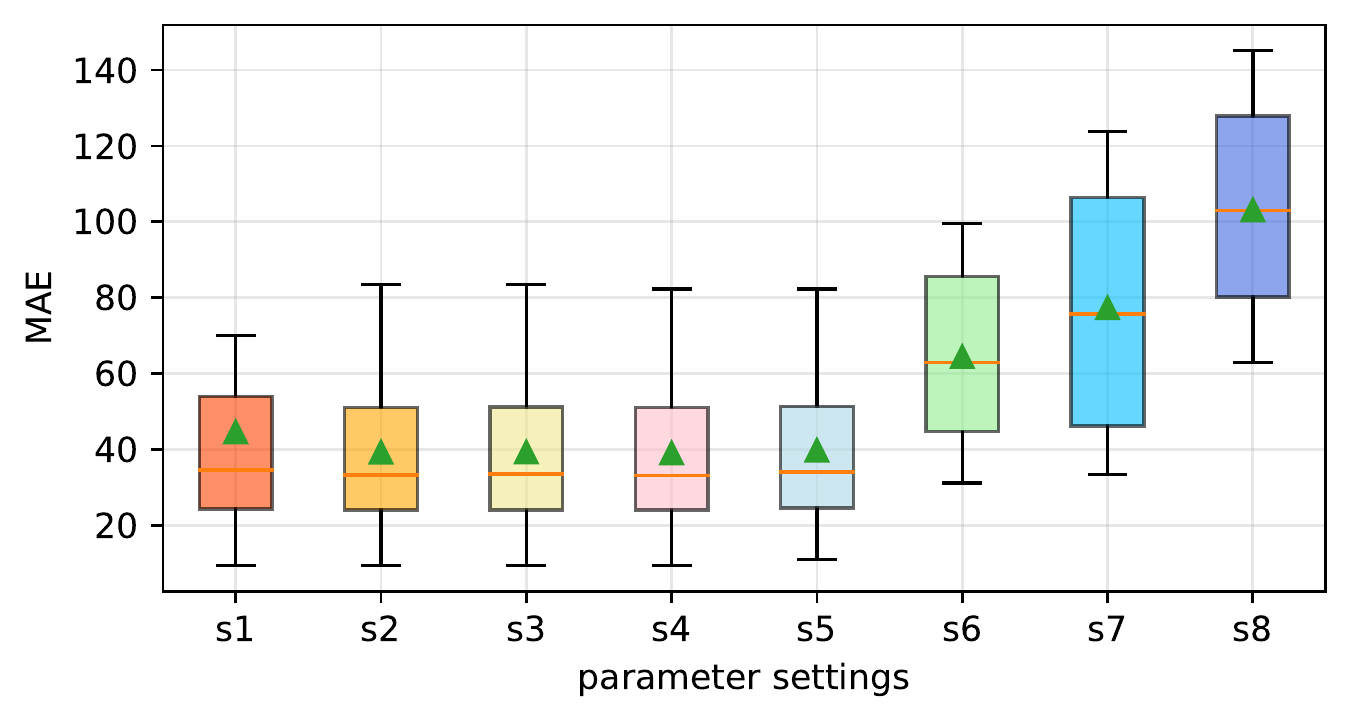}
}
\subfigure[Birmingham, FM-1]{
\centering
\includegraphics[scale=0.6]{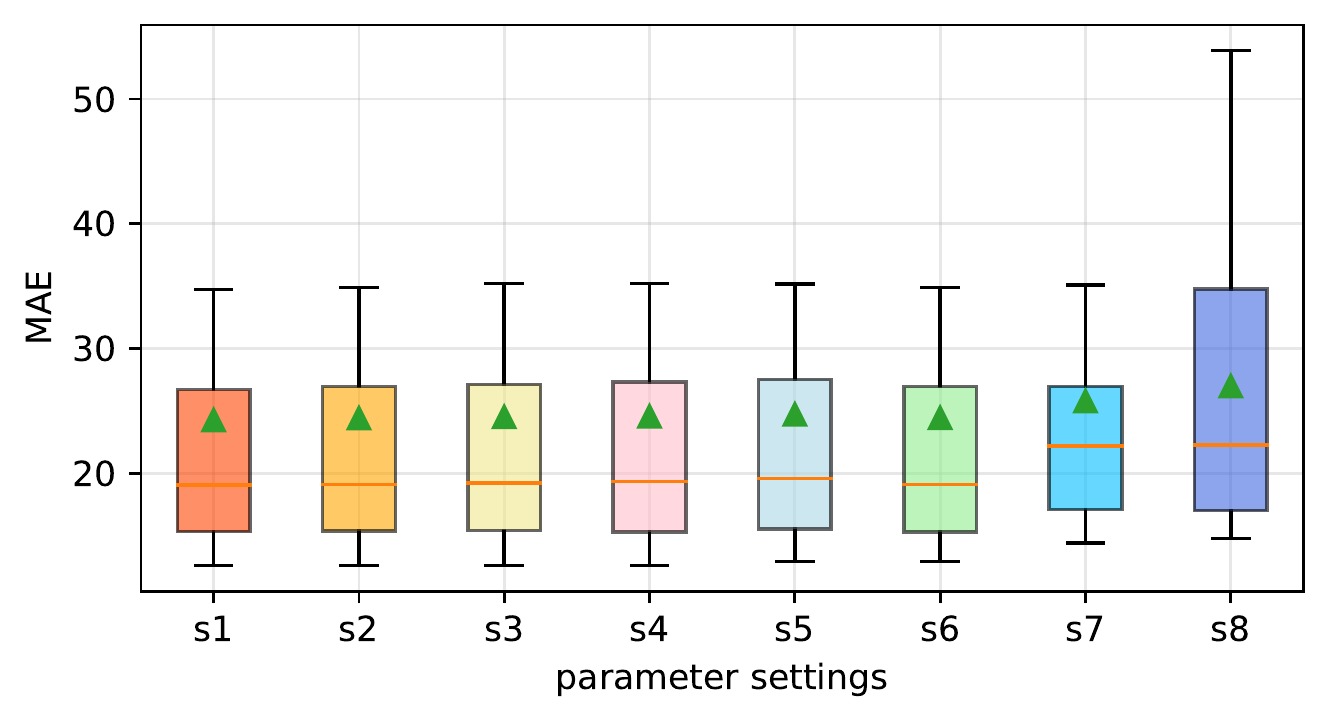}
}
\subfigure[Birmingham, FM-2]{
\centering
\includegraphics[scale=0.6]{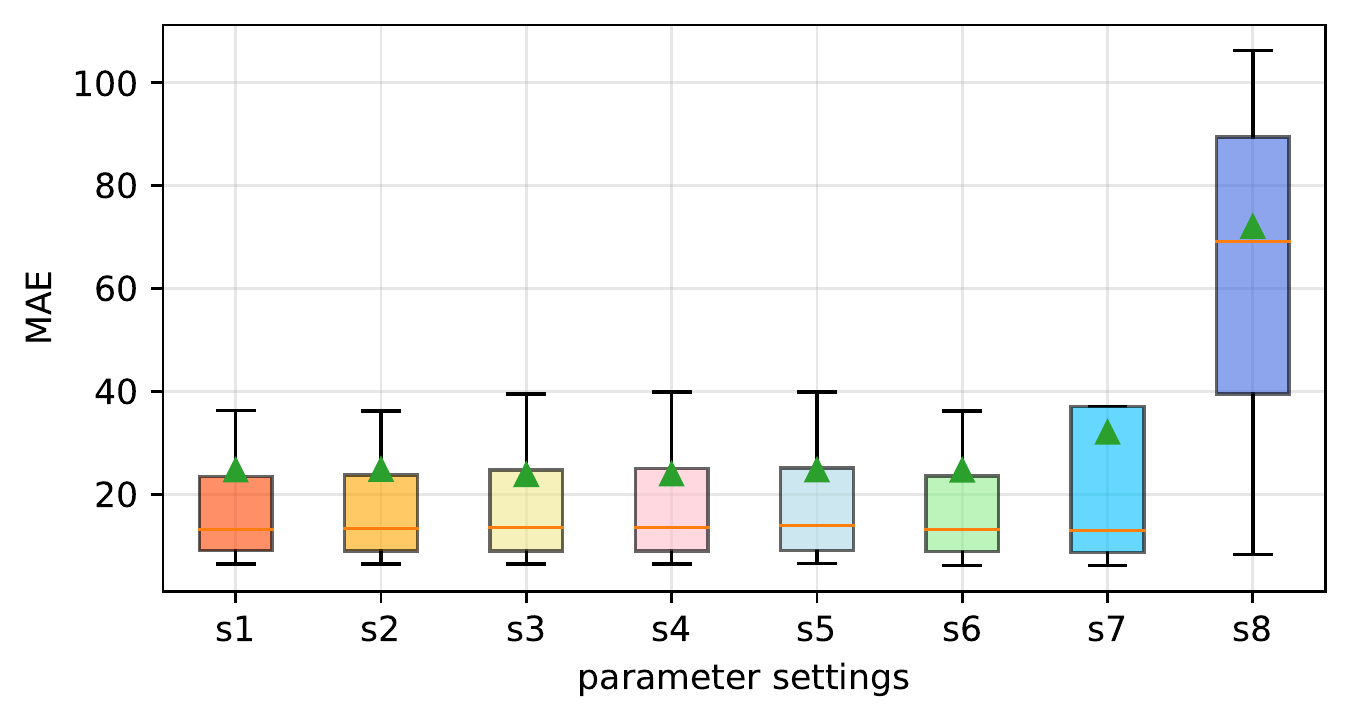}
}

\caption{Model performance under different $\theta$ parameter combinations on Birmingham data.}
\label{boxplot}
\end{figure}

There are mainly two observations from \minew{Figure} \ref{boxplot}: (1) the performances under groups s1, s2, s3, and s4 indicate that LRTC-TSpN is not sensitive to $\beta$, and an inappropriate $\beta$ setting could finally affect the upper bound of MAE while have minor impacts on the average or median performance; (2) the performances under setting s4, s5, s6, s7, and s8 reveal that the proposed method is more sensitive to $\theta_0$, which influences the lower bound of the MAE directly and the precise will significantly degrade if a too high $\theta_0$ is given like s8; however, through the decaying effect of $\beta$, the impacts of $\theta_0$ gradually diminish with increasing missing rate.

Generally, the truncation $\theta$ parameters in our method are basically not laborious to tune but effective to implement, as finding an appropriate $\theta_0$ is much easier than tuning a series of $\theta$ when dealing with varying missing rate. In practice, $\beta$ can be selected according to the actual within the given recommendation in \minew{Section} \ref{decay strategy}, and $\theta_0$ can be easily obtained by cross-validation or grid-search, etc.

\subsection{Further interpretation}
Although our method is a purely data-driven imputation model on the basis of optimization, it has good model interpretability, which is a noteworthy topic of imputation model recently \citep{lyngdoh2022prediction}. To better reveal the mechanism of truncation operation and the effects of model components, in this subsection we give further interpretations and discussions about the proposed method.

\subsubsection{Effect of truncation}
\label{Effect of truncation}
To give a intuitive explanation of the effect of truncation operator, we randomly select a partial observation matrix of five days from the Guangzhou speed dataset as an example in \minew{Figure} \ref{truncation example}. We generate $50\%$ missing data and recover it with LRTC-TSpN. Then for each of incomplete and recovered matrix, we calculate the SVD of it and accordingly restore it with a few largest singular values and remaining ones, resulting in a principal and a residual matrix, respectively. 

\begin{figure}[!htb]
\centering
\subfigure[\minew{Incomplete matrix}]{
\centering
\includegraphics[scale=0.26]{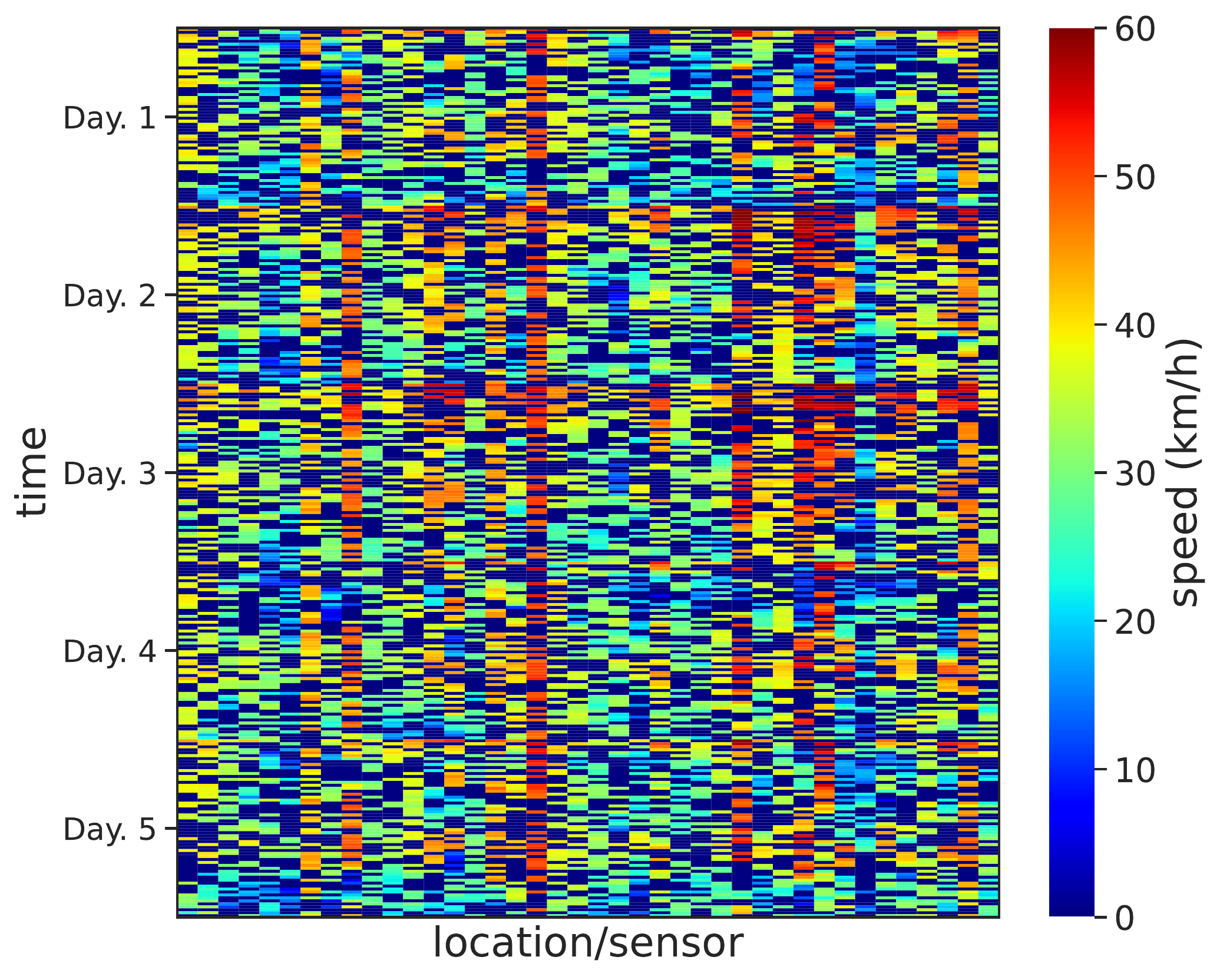}
}
\subfigure[Principal parts of incomplete matrix]{
\centering
\includegraphics[scale=0.26]{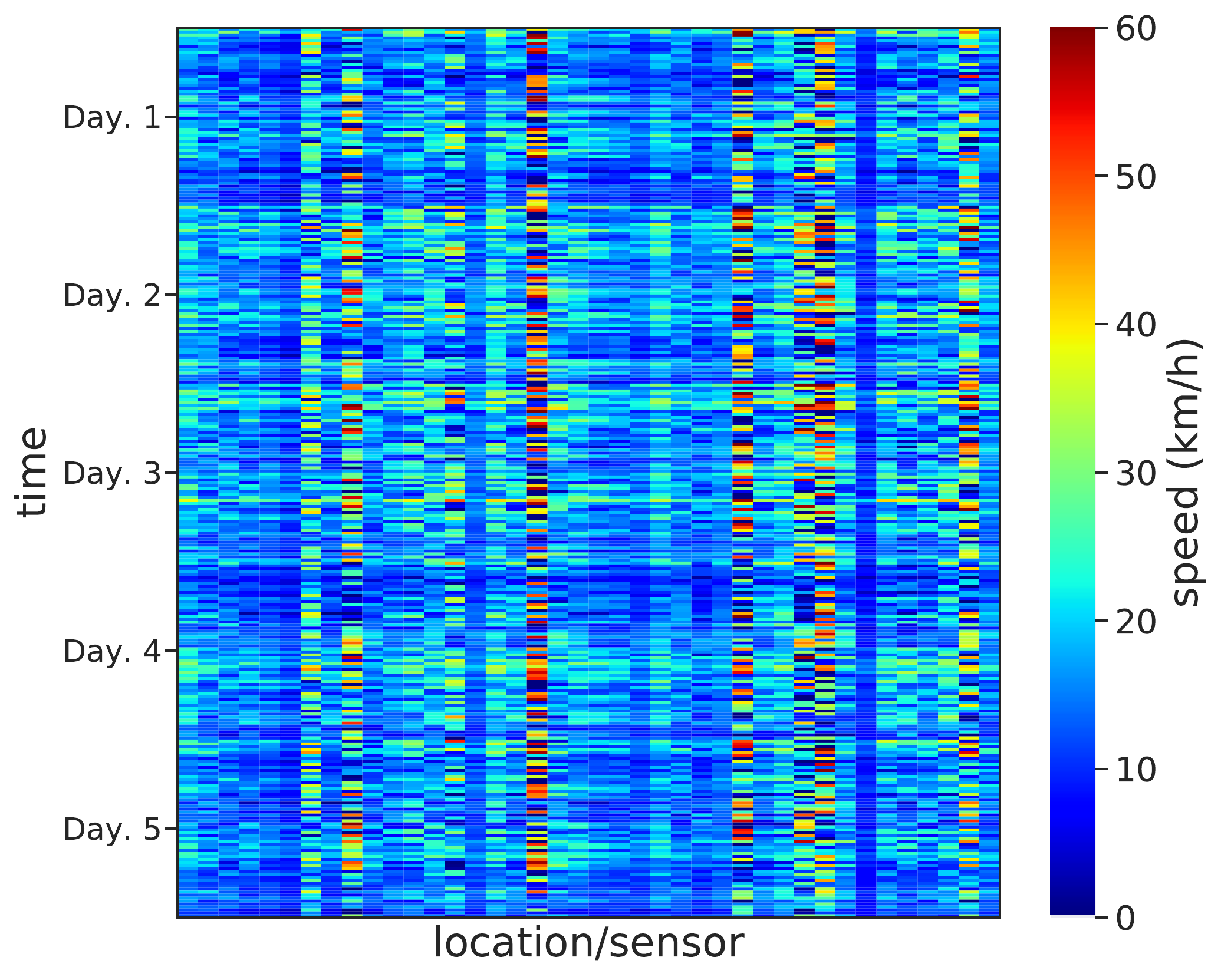}
}
\subfigure[Residual parts of incomplete matrix]{
\centering
\includegraphics[scale=0.26]{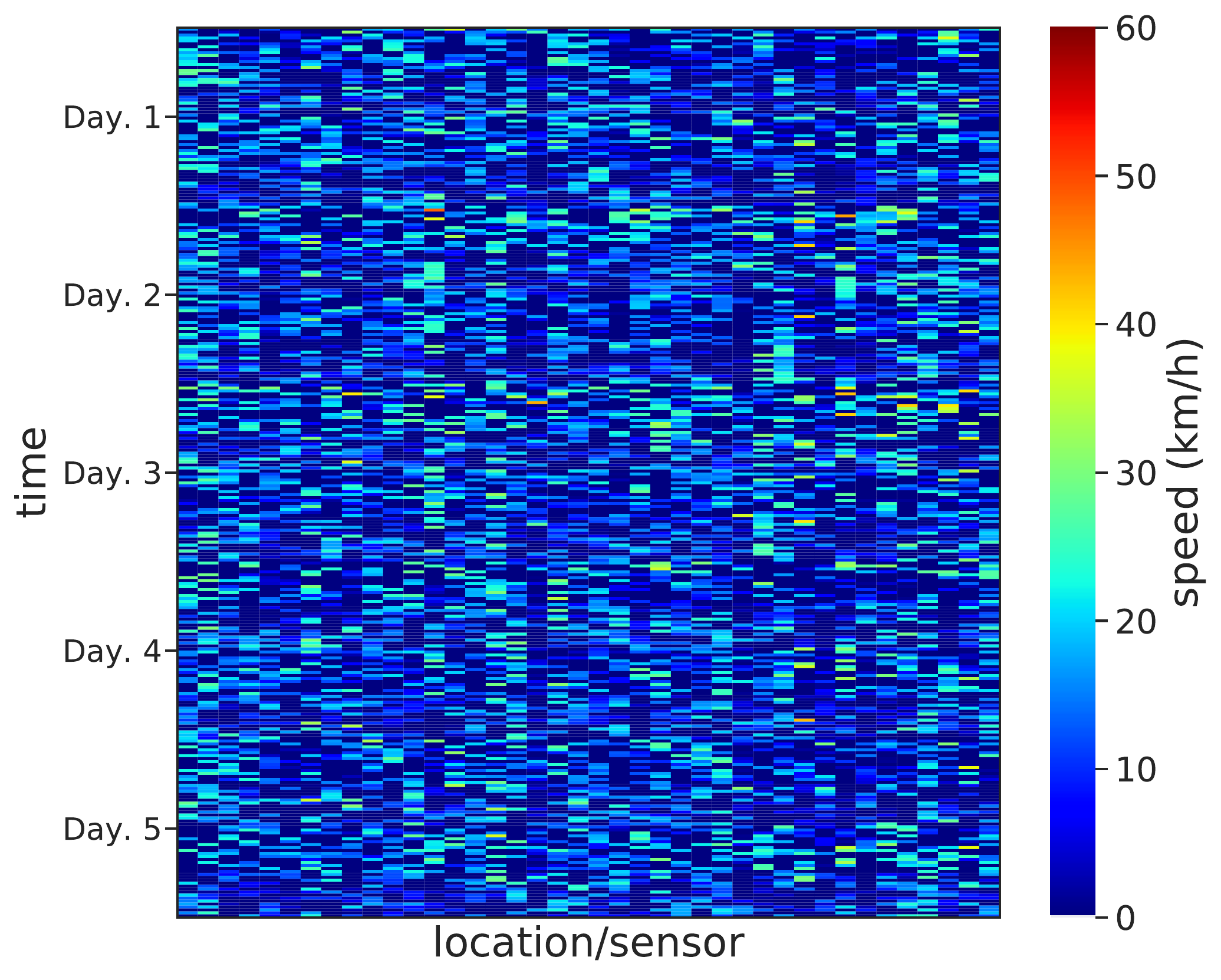}
}
\centering
\subfigure[Recovered matrix]{
\centering
\includegraphics[scale=0.26]{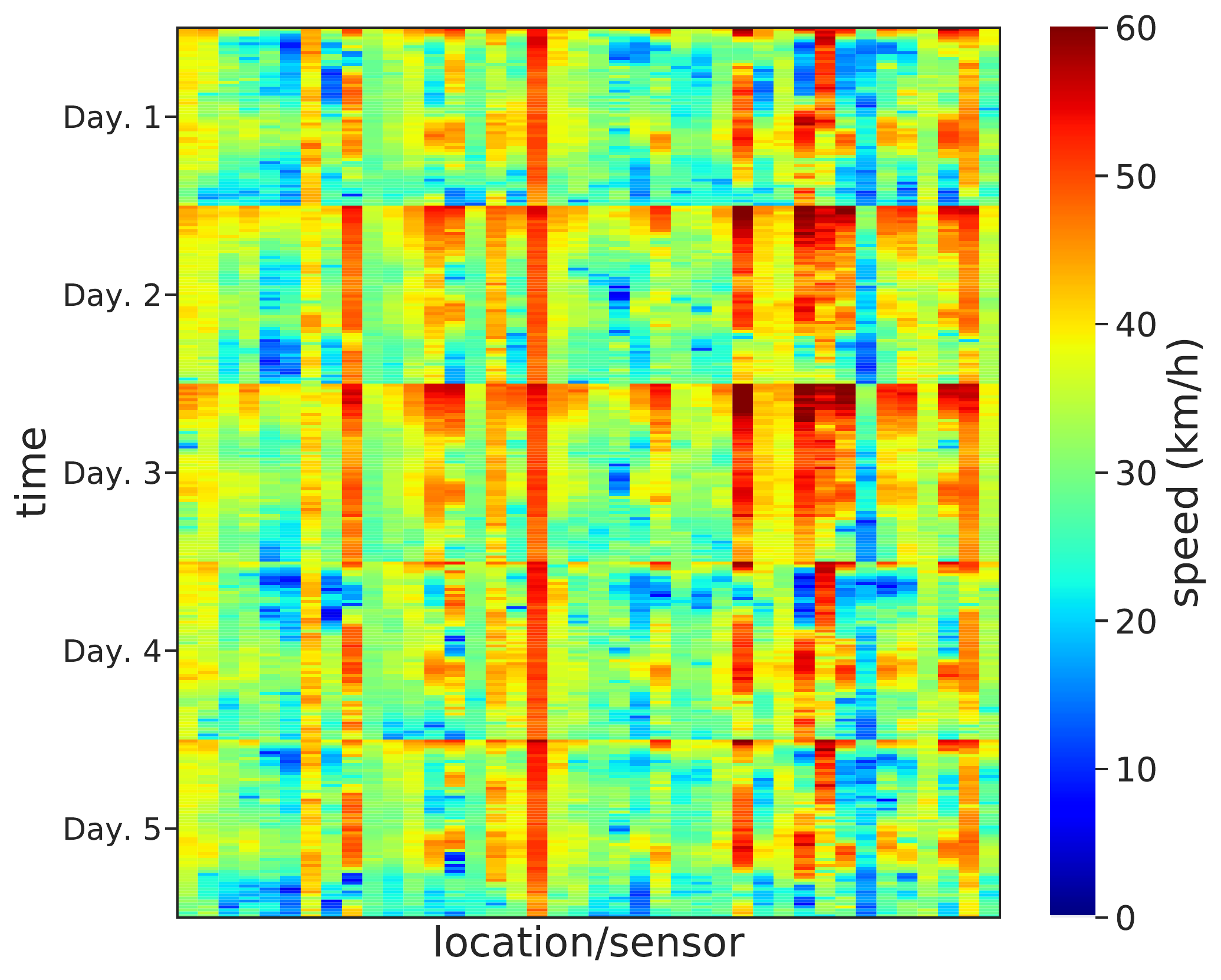}
}
\subfigure[Principal parts of recovered matrix]{
\centering
\includegraphics[scale=0.26]{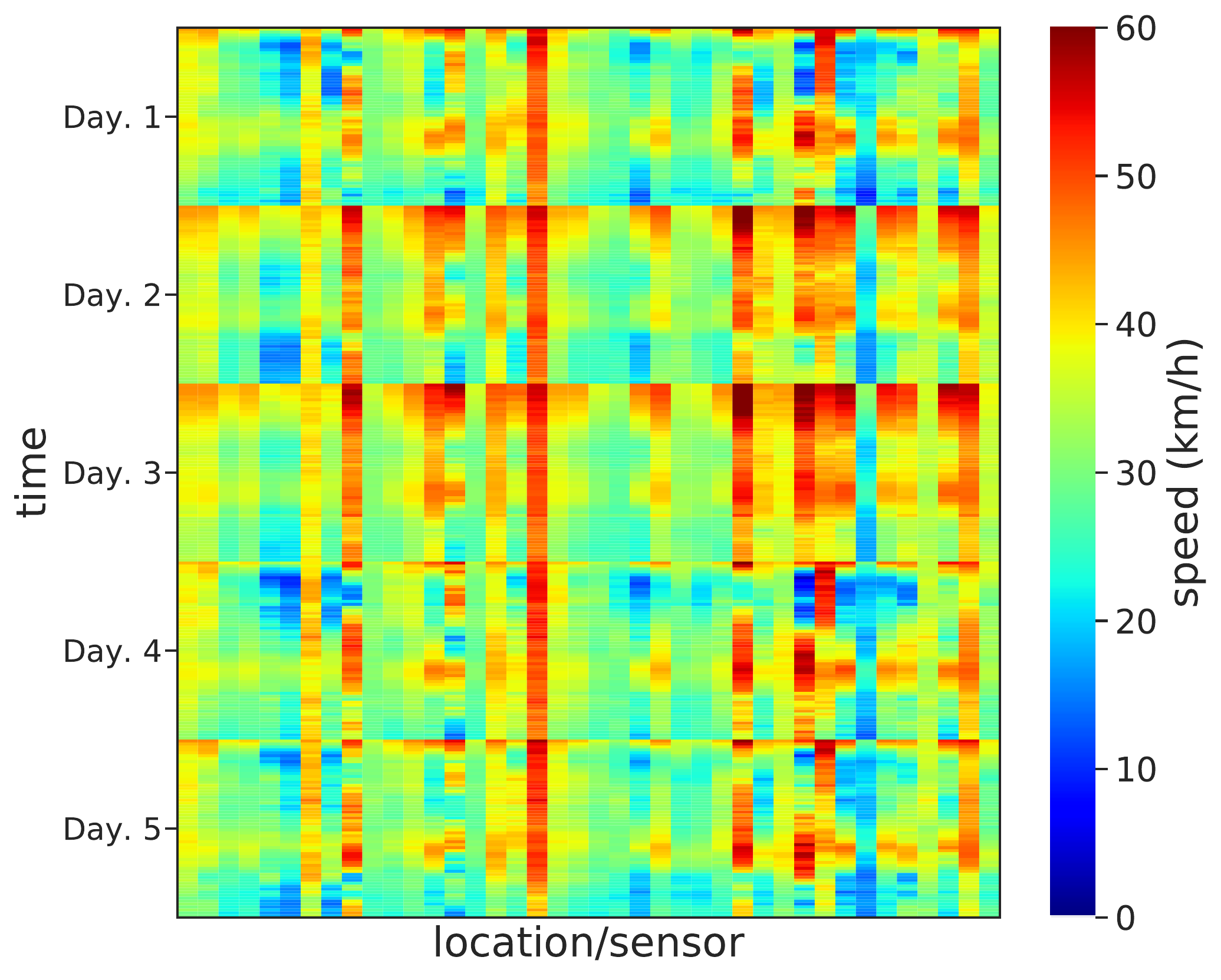}
}
\subfigure[Residual parts of recovered matrix]{
\centering
\includegraphics[scale=0.26]{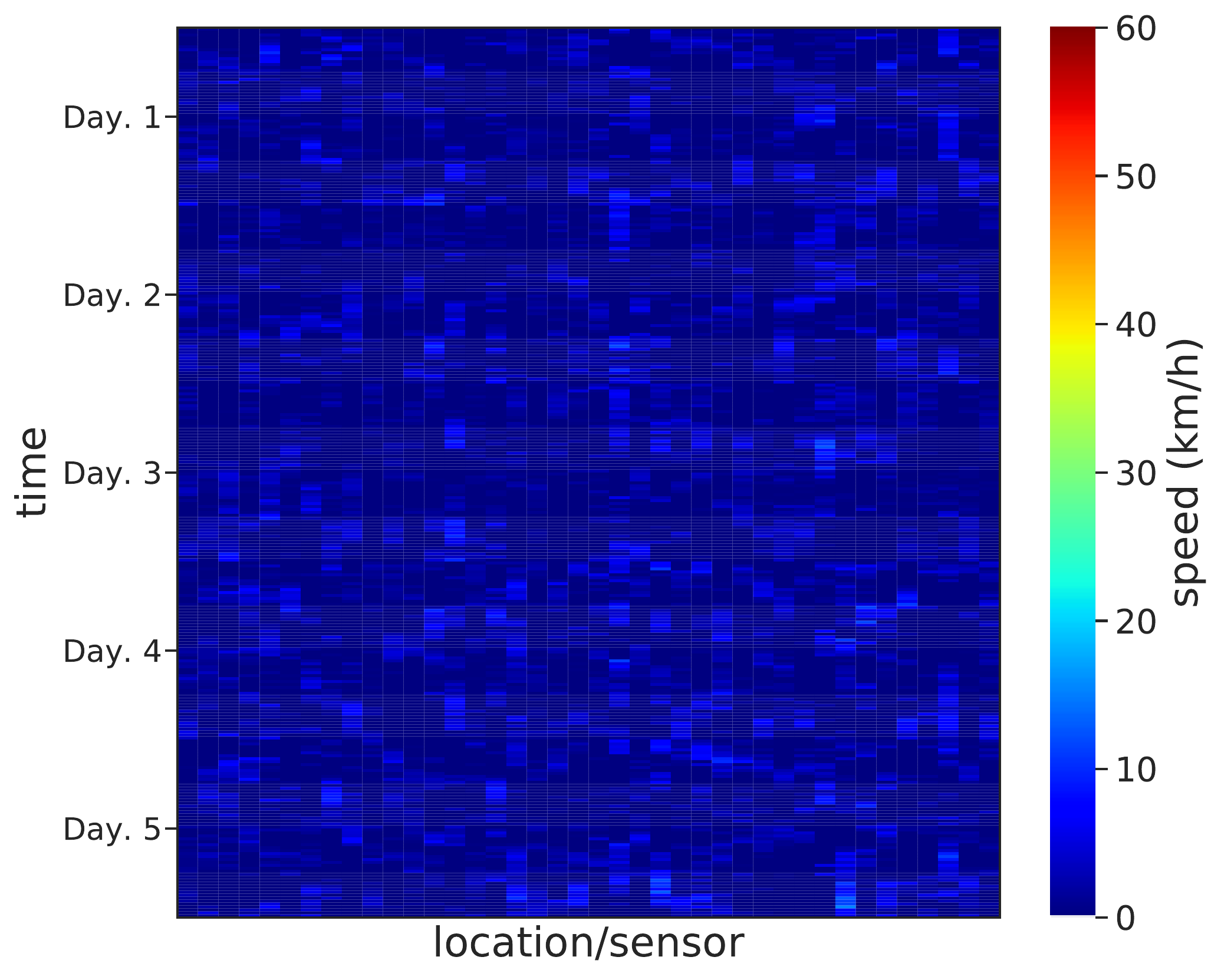}
}
\caption{\minew{Illustration of the truncation using five days' Guangzhou speed data. (a) Incomplete data observations organized in a time-location matrix with $50\%$ missing data. (b) Principal matrix reconstructed by the largest 5 singular values of the incomplete matrix. (c) Residual matrix reconstructed by the remaining singular values of the incomplete matrix. (d) Fully recovered matrix by LRTC-TSpN. (e) Principal matrix reconstructed by the largest 5 singular values of the recovered matrix. (f) Residual matrix reconstructed by the remaining singular values of the recovered matrix.} }
\label{truncation example}
\end{figure}

In \minew{Figure} \ref{truncation example}, it is obvious to find that the principal matrix of the incomplete matrix in \minew{Figure} \ref{truncation example}(b) still preserves the low-rank patterns to some extent, while the residual parts in \minew{Figure} \ref{truncation example}(c) construct a sparse residual matrix like noise. By keeping the major components unchanged and only incorporating the residual parts into the optimization process, LRTC-TSpN forces these residual parts to be 'low-rank' as much as possible and \minew{reconstructs} the principal components (i.e., the underlying low-rank pattern) of entire speed matrix using only a few largest singular values. The superposition of the reduced residual parts in (f) and the low-rank parts in (e) \minew{produces} the final recovered matrix in (d). Therefore, by introducing the truncation operator, we are capable of \minew{preventing the principal parts from over-shrinking, and reducing the residuals evidently, so as to construct a 'low-rank' output.}

\subsubsection{Comparison of the imputed singular values}
To give a further explanation of the strength of TSpN, we compare the imputed singular values of LRTC-TNN and LRTC-TSpN using Guangzhou dataset (50\% missing rate) and the results are reported in \minew{Figure} \ref{G-SVD}. We can see that for these singular values in the top of the order, the imputed ones by our model are closer to the real ones in all missing scenarios. This over-shrinkage problem of TNN makes the results deviate from the true singular values while the proposed TSpN model achieves a better approximation of larger singular values and keeps the low-rank components.

\begin{figure}[!ht]
\centering
\subfigure[$50\%$ RM]{
\centering
\includegraphics[scale=0.5]{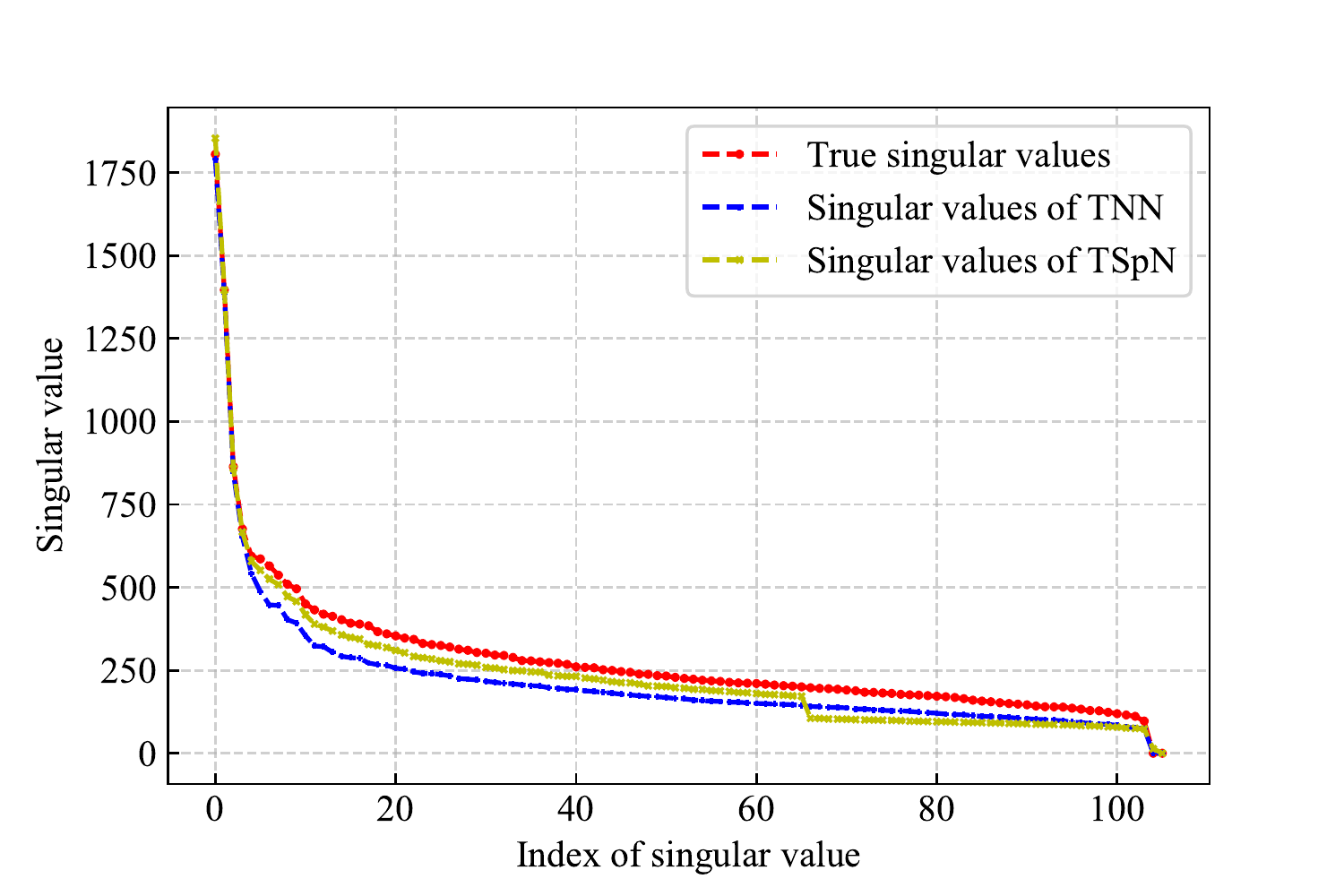}
}
\subfigure[$50\%$ FM-0]{
\centering
\includegraphics[scale=0.5]{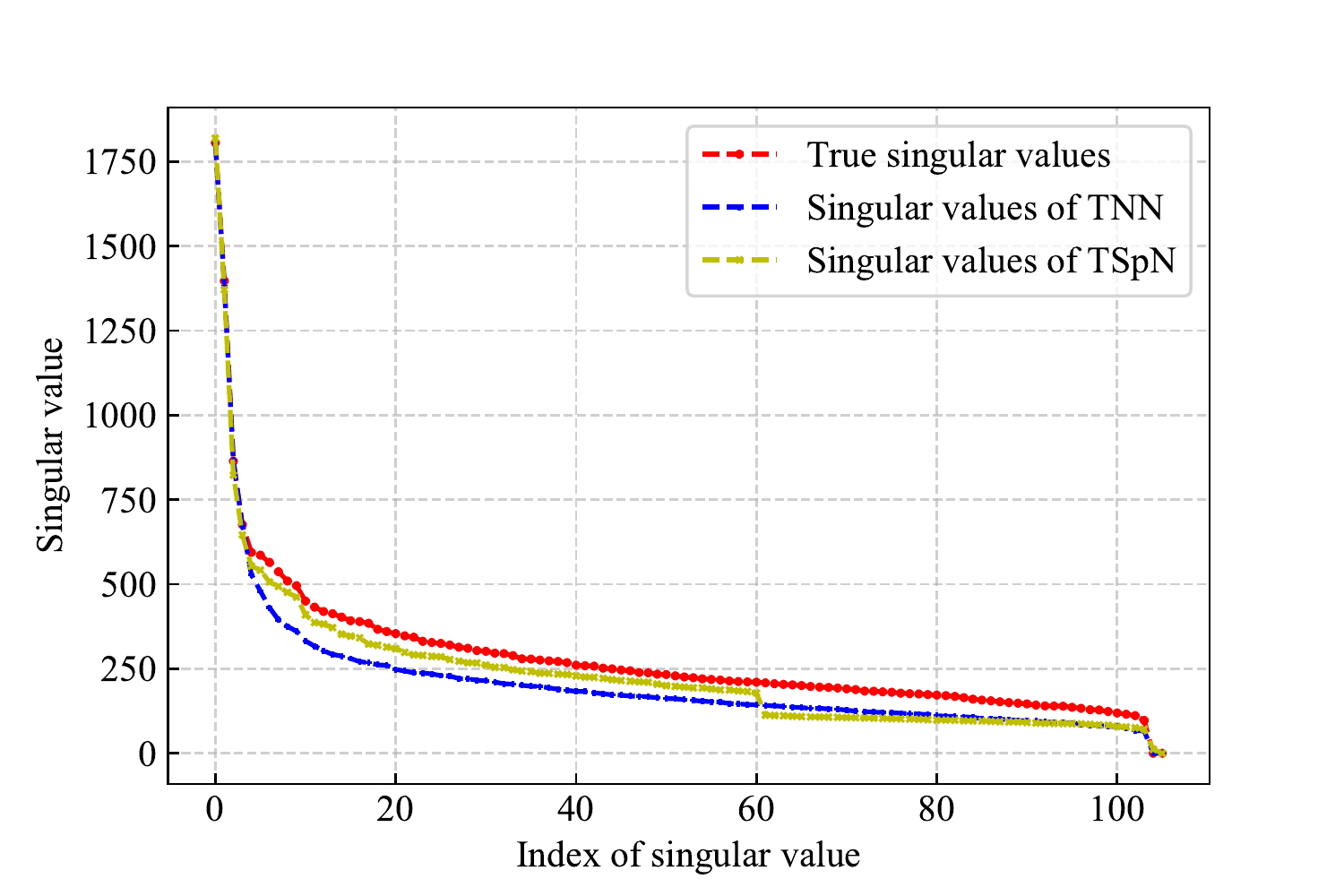}
}
\subfigure[$50\%$ FM-1]{
\centering
\includegraphics[scale=0.5]{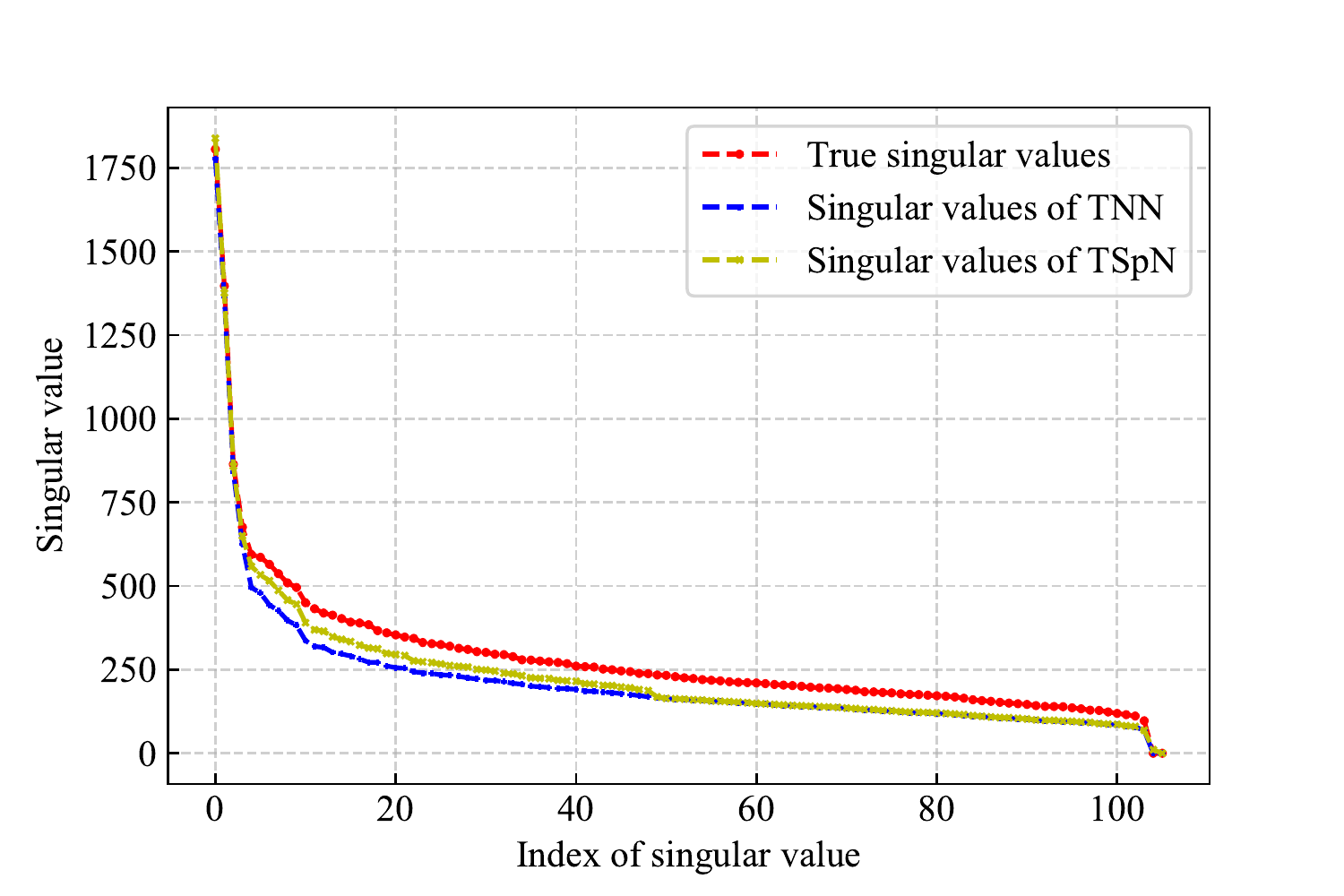}
}
\subfigure[$50\%$ FM-2]{
\centering
\includegraphics[scale=0.5]{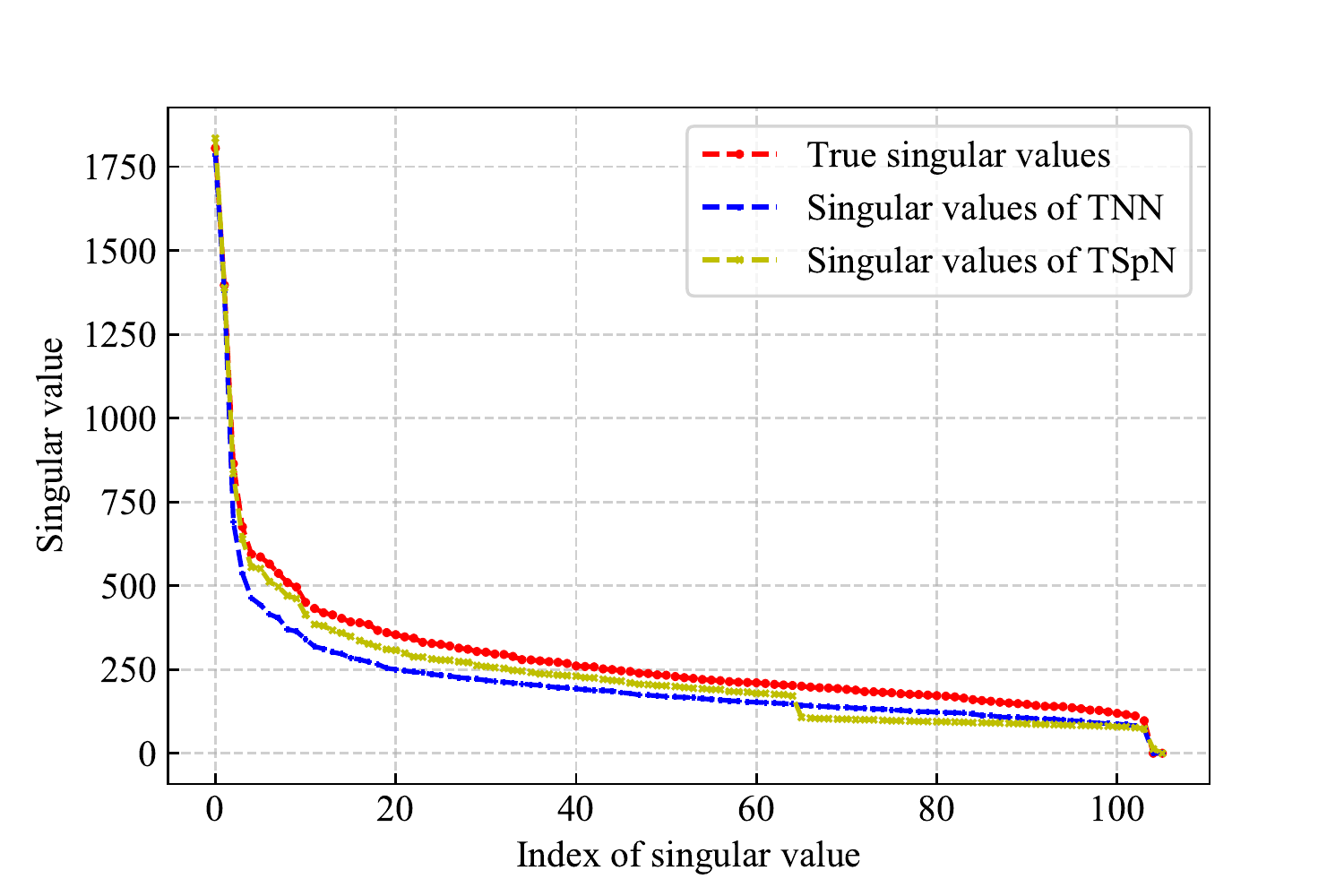}
}
\caption{Real and imputed singular values of Guangzhou speed data (the first 107 sensors). Note that the largest one singular value is not plotted on these figures for demonstration purposes. (a) $50\%$ random missing. (b) $50\%$ fiber mode-0 missing. (c) $50\%$ mode-1 missing. (d) $50\%$ mode-2 missing.}
\label{G-SVD}
\end{figure}

\subsubsection{Effect of truncation decay strategy}
To verify the effectiveness of truncation rate decay strategy as described in \minew{Section} \ref{decay strategy}, we further make a comparison between the constant truncation rate and decaying truncation rate scheme, named 'TSpN-fixed' and 'TSpN-decay', respectively. We take the Birmingham occupancy data as an example for demonstration with missing rates varying from $10\%$ to $90\%$ in four missing patterns. For 'TSpN-fixed', we keep $\theta = \theta_0$ by setting $\beta=0$, and keep other parameters the same as 'TSpN-decay'. Results are shown in \minew{Figure} \ref{decay compare plot}.

\begin{figure}[!htb]
\centering
\subfigure[Birmingham, RM]{
\centering
\includegraphics[scale=0.6]{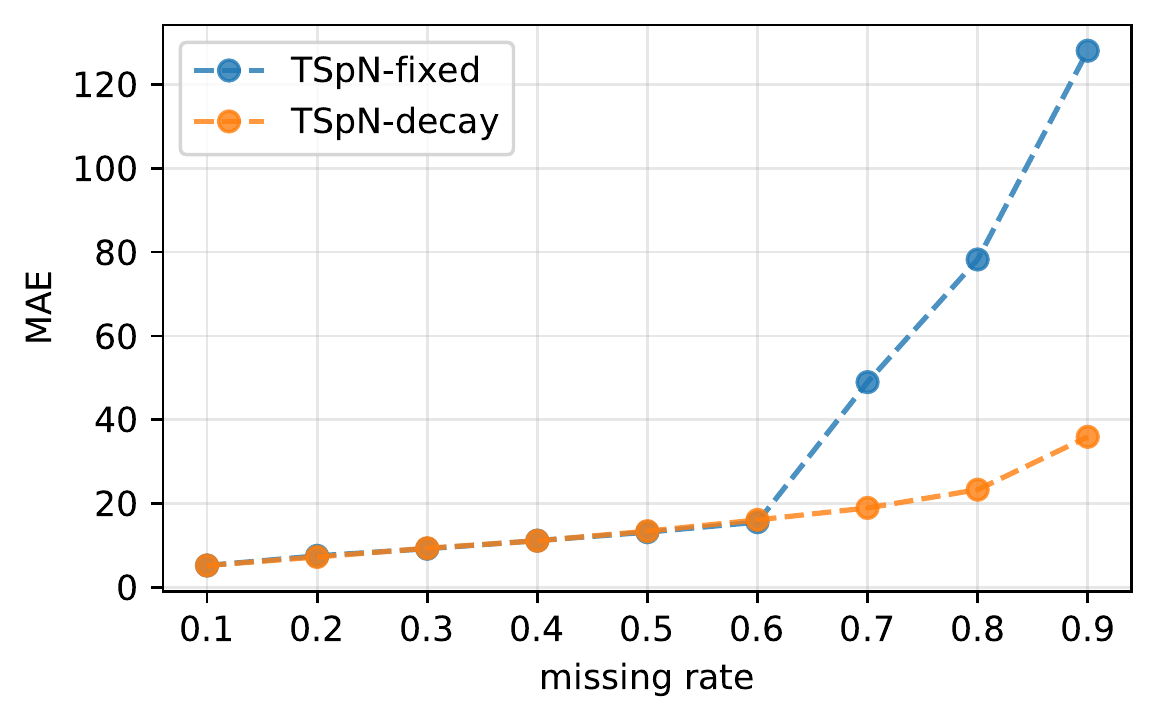}
}
\subfigure[Birmingham, FM-0]{
\centering
\includegraphics[scale=0.6]{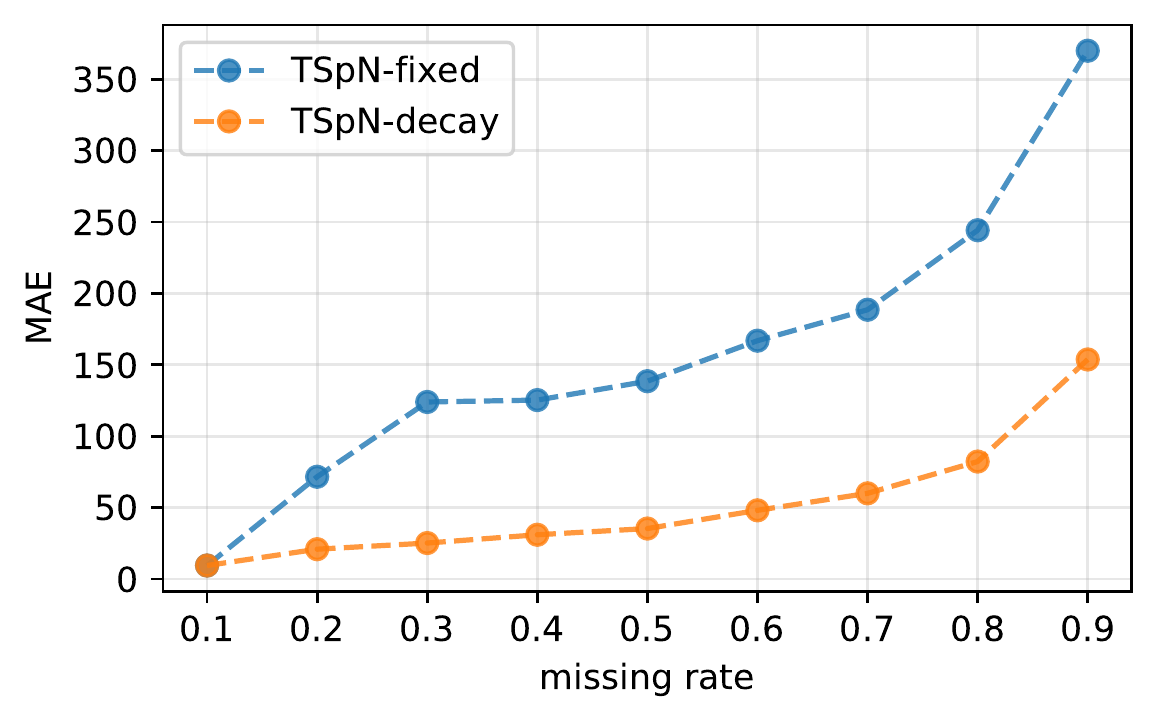}
}
\subfigure[Birmingham, FM-1]{
\centering
\includegraphics[scale=0.6]{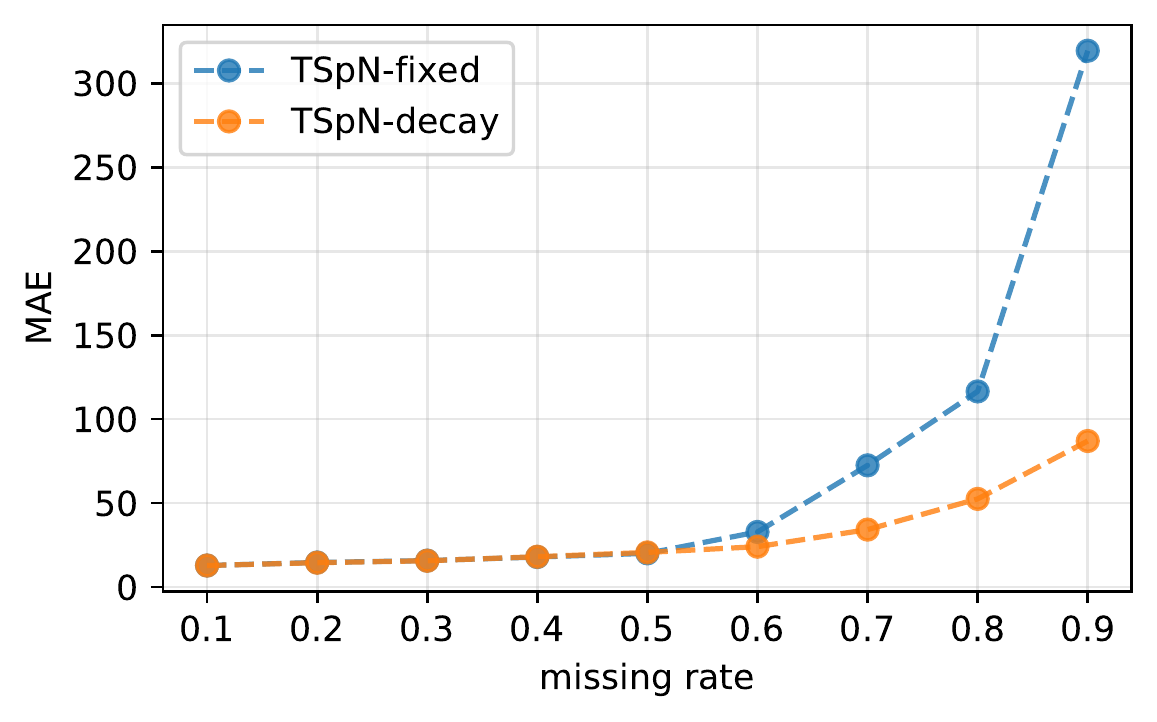}
}
\subfigure[Birmingham, FM-2]{
\centering
\includegraphics[scale=0.6]{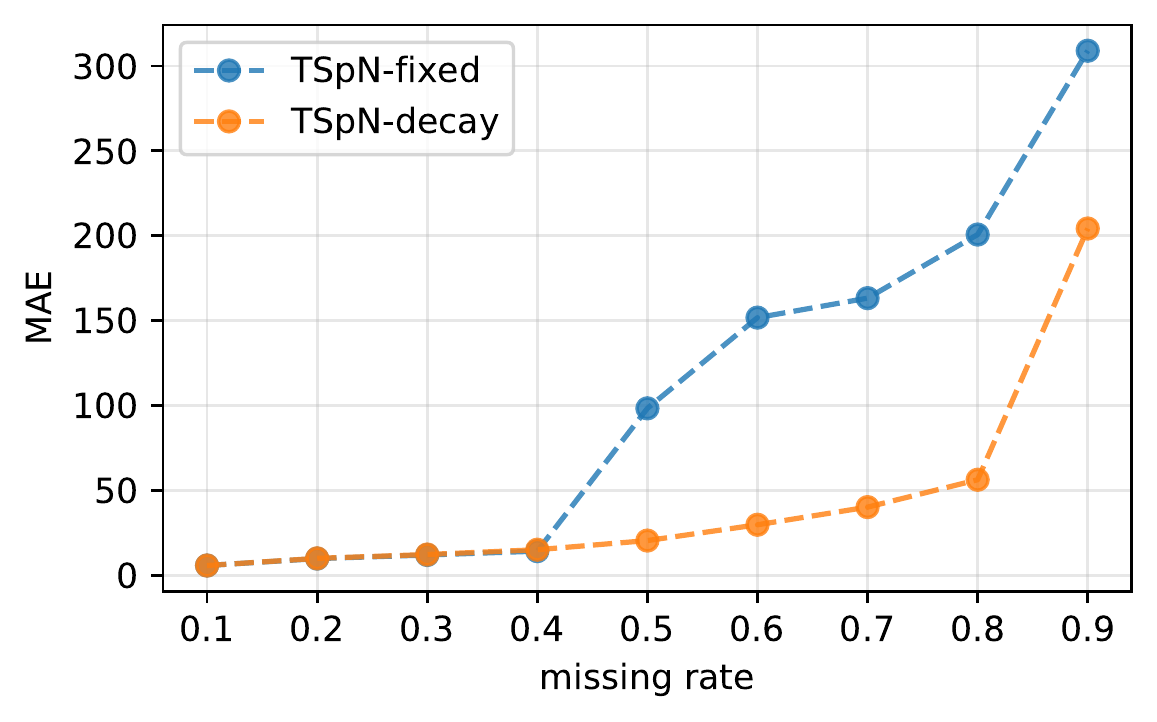}
}

\caption{Truncation scheme comparison in Birmingham data with different missing situations.}
\label{decay compare plot}
\end{figure}

These variation studies display that when loss rate is low (e.g., lower than $40\%$), there are sufficient observed data to model the latent patterns so that less information is conserved in the small singular values. This is reflected in \minew{Figure} \ref{decay compare plot} where the differences between the two variants are not significant at first, whereas the error of 'TSpN-fixed' surges thereafter. Therefore, this truncation decay strategy is an indispensable component to stabilize our method when missing rate becomes high.

\minew{
\subsubsection{Missing slices recovery}
A more challenging structural missing case is the recovery of missing 'slices', i.e., several entire columns or rows in the 'location' $\times$ 'time' matrix are not available. One can find that FM-1 missing in Figure \ref{FM} denotes the case that the 'time' rows are missing. When these missing rows occur at the same time interval every day, this becomes the 'time slices' missing scenario in the tensor. Experimental results of FM-1 cases reveal that LRTC-TSpN is capable of handling this kind of tasks. Moreover, \cite{yokota2018missing} have shown that multi-way delay embedding transform, aka., Hankelization, could be useful for augmenting shift-invariant features, thereby benefiting the missing 'time slices' recovery. While at the same time, this procedure increases the computational burden greatly.

Nevertheless, both our method and other LRTC approaches (even augmented by Hankelization) are not capable of dealing with the missing "spatial slices" situations (entire 'location' columns are missing), e.g., some locations do not equip with sensors at all. Unlike the inherent consistency of the time dimension, spatial consistency entails extra information. A spatial Laplacian regularization \citep{yang2021real} can be incorporated into the objective function of LRTC to tackle this issue. 
}

\section{Conclusion and future directions}\label{conclusions}

In this paper, we propose a novel truncated tensor Schatten $p$-norm minimization model  that preserves the advantages of both truncated nuclear norm and tensor Schatten $p$-norm for missing value imputation in spatiotemporal traffic data. Extensive experiments are conducted on four different types of traffic data under complicated missing patterns with the missing rates ranging from 10\% to 90\%, including random missing and three fiber mode-$n$ missing cases. Results reveal that: 1) our LRTC-TSpN model shows competitive performance compared with some state-of-the-art models and is robust under increasing missing rates and different missing cases, for instance, the proposed method achieves up to a 52\% MAE improvement in Birmingham dataset compared to the most competitive baseline even when the extreme missing ratio is as high as 90\%; 2) the Schatten $p$-norm serves as a better nonconvex surrogate for rank than nuclear norm so that the proposed LRTC-TSpN can capture the low-rank components of tensors more effective than other NN or TNN based models; 3) missing patterns have significant impacts on imputation accuracy, and the degree of influence is susceptible to data type; 4) truncation operation plays a key role in keeping the larger singular values from being over-shrinking, and the truncation rate is supposed to be lower when dealing with higher missing rate. Moreover, theoretical and empirical evidences also demonstrate that our method has robust and efficient computational performances, especially for convergence capability.

Despite promising results the proposed LRTC-TSpN model has shown, there are still some directions worth investigation in our future research. First, the truncation rate parameter, which is set globally to make an easy realization, may need to adjust to each unfolding matrix by setting various degrees of truncation to improve performance. \minew{Second, the kind of unfolding-based norm, including sum of nuclear norm and sum of Schatten $p$-norm, etc., is dependent on the SVD performed on each unfolding matrix. Although the proposed method has promising computational performances, the scalability of Schatten $p$-norm based model on massive data scenes could be further improved by accelerating the computation of SVD. In the future, we will explore the potential of randomized techniques for fast computation, such as randomized SVD \citep{halko2011finding},  to further reduce the computational cost.}

\section*{Acknowledgement}

This research was sponsored by the National Natural Science Foundation of China (52125208) and the Shanghai Municipal Science and Technology Major Project (No. 2021SHZDZX0100).

%\section*{Appendix A.}
%\label{Appendix A.}

% \section*{Appendix B.}
% \label{Appendix B.}

% \begin{figure}[!htb]
% \centering
% \begin{minipage}[t]{0.48\textwidth}
% \subfigure[Birmingham, RM]{
% \centering
% \includegraphics[scale=0.55]{p-experiment1-B-RM.pdf}
% }
% \end{minipage}
% \begin{minipage}[t]{0.48\textwidth}
% \subfigure[Birmingham, FM-2]{
% \centering
% \includegraphics[scale=0.55]{p-experiment1-B-FM2.pdf}
% }
% \end{minipage}
% \caption{\minew{Model performance under varying $p$ values on Birmingham dataset with different missing rates.}}
% \end{figure}

% \begin{figure}[!htb]
% \centering
% \begin{minipage}[t]{0.48\textwidth}
% \subfigure[Guangzhou-small, FM-1]{
% \centering
% \includegraphics[scale=0.65]{model compare-G1.pdf}
% }
% \end{minipage}
% \begin{minipage}[t]{0.48\textwidth}
% \subfigure[Portland-small, FM-2]{
% \centering
% \includegraphics[scale=0.65]{model compare-P2.pdf}
% }
% \end{minipage}
% \caption{\minew{Model comparison in Guangzhou-small and Portland-small data with different fiber-like missing situations.}}
% \end{figure}

% \begin{figure}[!htb]
% \centering
% \begin{minipage}[t]{0.48\textwidth}
% \subfigure[Guangzhou-small, $95\%$ FM-2]{
% \centering
% \includegraphics[scale=0.65]{convergence-G-FM2.pdf}
% }
% \end{minipage}
% \begin{minipage}[t]{0.48\textwidth}
% \subfigure[Portland-small, $95\%$ FM-2]{
% \centering
% \includegraphics[scale=0.65]{convergence-P-FM2.pdf}
% }
% \end{minipage}
% \caption{\minew{The MAE variation curve with the epoch increasing until convergence.}}
% \end{figure}

%\newpage
% \bibliographystyle{elsarticle-harv}
% \bibliography{references}

\begin{thebibliography}{33}
\expandafter\ifx\csname natexlab\endcsname\relax\def\natexlab#1{#1}\fi
\expandafter\ifx\csname url\endcsname\relax
  \def\url#1{\texttt{#1}}\fi
\expandafter\ifx\csname urlprefix\endcsname\relax\def\urlprefix{URL }\fi


\bibitem[{Asif et~al.(2016)Asif, Mitrovic, Dauwels, and
  Jaillet}]{asif2016matrix}
Asif, M.~T., Mitrovic, N., Dauwels, J., Jaillet, P., 2016. Matrix and tensor
  based methods for missing data estimation in large traffic networks. IEEE
  Transactions on intelligent transportation systems 17~(7), 1816--1825.


\bibitem[{Bolte et~al.(2014)Bolte, Sabach, and Teboulle}]{bolte2014proximal}
Bolte, J., Sabach, S., Teboulle, M., 2014. Proximal alternating linearized
  minimization for nonconvex and nonsmooth problems. Mathematical Programming
  146~(1), 459--494.

\bibitem[{Boquet et~al.(2019)Boquet, Vicario, Morell, and
  Serrano}]{boquet2019missing}
Boquet, G., Vicario, J.~L., Morell, A., Serrano, J., 2019. Missing data in
  traffic estimation: A variational autoencoder imputation method. In: ICASSP
  2019-2019 IEEE International Conference on Acoustics, Speech and Signal
  Processing (ICASSP). IEEE, pp. 2882--2886.

\bibitem[{Boyd et~al.(2011)Boyd, Parikh, Chu, Peleato, Eckstein,
  et~al.}]{boyd2011distributed}
Boyd, S., Parikh, N., Chu, E., Peleato, B., Eckstein, J., et~al., 2011.
  Distributed optimization and statistical learning via the alternating
  direction method of multipliers. Foundations and Trends{\textregistered} in
  Machine learning 3~(1), 1--122.


\bibitem[{Cai et~al.(2010{\natexlab{a}})Cai, Cand{\`e}s, and
  Shen}]{cai2010singular}
Cai, J.-F., Cand{\`e}s, E.~J., Shen, Z., 2010{\natexlab{a}}. A singular value
  thresholding algorithm for matrix completion. SIAM Journal on optimization
  20~(4), 1956--1982.


\bibitem[{Cao et~al.(2013)Cao, Sun, and Xu}]{cao2013fast}
Cao, W., Sun, J., Xu, Z., 2013. Fast image deconvolution using closed-form
  thresholding formulas of lq (q= 12, 23) regularization. Journal of visual
  communication and image representation 24~(1), 31--41.

\bibitem[{Cao et~al.(2021)Cao, Tang, Sun, and Ji}]{cao2021day}
Cao, Y., Tang, K., Sun, J., Ji, Y., 2021. Day-to-day dynamic
  origin--destination flow estimation using connected vehicle trajectories and
  automatic vehicle identification data. Transportation Research Part C:
  Emerging Technologies 129, 103241.

\bibitem[{Chen et~al.(2018{\natexlab{a}})Chen, Sun, Xia, Feng, and
  Li}]{chen2018human}
Chen, B., Sun, H., Xia, G., Feng, L., Li, B., 2018{\natexlab{a}}. Human motion
  recovery utilizing truncated schatten p-norm and kinematic constraints.
  Information Sciences 450, 89--108.


\bibitem[{Chen et~al.(2020{\natexlab{a}})Chen, Jiang, Liu, and
  Zhou}]{chen2020robust}
Chen, L., Jiang, X., Liu, X., Zhou, Z., 2020{\natexlab{a}}. Robust low-rank
  tensor recovery via nonconvex singular value minimization. IEEE Transactions
  on Image Processing 29, 9044--9059.

\bibitem[{Chen et~al.(2018{\natexlab{b}})Chen, Cai, Ye, Chen, and
  Li}]{chen2018graph}
Chen, X., Cai, Y., Ye, Q., Chen, L., Li, Z., 2018{\natexlab{b}}. Graph
  regularized local self-representation for missing value imputation with
  applications to on-road traffic sensor data. Neurocomputing 303, 47--59.

\bibitem[{Chen et~al.(2021{\natexlab{c}})Chen, Chen, Saunier, and
  Sun}]{chen2021scalable}
Chen, X., Chen, Y., Saunier, N., Sun, L., 2021{\natexlab{c}}. Scalable low-rank
  tensor learning for spatiotemporal traffic data imputation. Transportation
  Research Part C: Emerging Technologies 129, 103226.

\bibitem[{Chen et~al.(2019)Chen, He, and Sun}]{chen2019bayesian}
Chen, X., He, Z., Sun, L., 2019. A bayesian tensor decomposition approach for
  spatiotemporal traffic data imputation. Transportation research part C:
  emerging technologies 98, 73--84.

\bibitem[{Chen et~al.(2018{\natexlab{c}})Chen, He, and Wang}]{chen2018spatial}
Chen, X., He, Z., Wang, J., 2018{\natexlab{c}}. Spatial-temporal traffic speed
  patterns discovery and incomplete data recovery via svd-combined tensor
  decomposition. Transportation research part C: emerging technologies 86,
  59--77.

\bibitem[{Chen et~al.(2021{\natexlab{d}})Chen, Lei, Saunier, and
  Sun}]{chen2021Autoregressive}
Chen, X., Lei, M., Saunier, N., Sun, L., 2021{\natexlab{d}}. Low-rank
  autoregressive tensor completion for spatiotemporal traffic data imputation.
  IEEE Transactions on Intelligent Transportation Systems, 1--10.

\bibitem[{Chen and Sun(2021)}]{chen2021bayesian}
Chen, X., Sun, L., 2021. Bayesian temporal factorization for multidimensional
  time series prediction. IEEE Transactions on Pattern Analysis and Machine
  Intelligence, 1--1.

\bibitem[{Chen et~al.(2017)Chen, Wei, Li, Liang, Cai, and
  Zhang}]{chen2017ensemble}
Chen, X., Wei, Z., Li, Z., Liang, J., Cai, Y., Zhang, B., 2017. Ensemble
  correlation-based low-rank matrix completion with applications to traffic
  data imputation. Knowledge-Based Systems 132, 249--262.

\bibitem[{Chen et~al.(2020{\natexlab{b}})Chen, Yang, and
  Sun}]{chen2020nonconvex}
Chen, X., Yang, J., Sun, L., 2020{\natexlab{b}}. A nonconvex low-rank tensor
  completion model for spatiotemporal traffic data imputation. Transportation
  Research Part C: Emerging Technologies 117, 102673.

\bibitem[{Cipriani et~al.(2021)Cipriani, Gemma, Mannini, Carrese, and
  Crisalli}]{cipriani2021traffic}
Cipriani, E., Gemma, A., Mannini, L., Carrese, S., Crisalli, U., 2021. Traffic
  demand estimation using path information from bluetooth data. Transportation
  Research Part C: Emerging Technologies 133, 103443.

\bibitem[{Deng et~al.(2021)Deng, Liu, Zheng, Feng, and Chen}]{deng2021graph}
Deng, L., Liu, X.-Y., Zheng, H., Feng, X., Chen, Y., 2021. Graph spectral
  regularized tensor completion for traffic data imputation. IEEE Transactions
  on Intelligent Transportation Systems, 1--15.

\bibitem[{Du et~al.(2020{\natexlab{a}})Du, Hu, Sun, Liu, Qiao, and
  Lv}]{du2020traffic}
Du, B., Hu, X., Sun, L., Liu, J., Qiao, Y., Lv, W., 2020{\natexlab{a}}. Traffic
  demand prediction based on dynamic transition convolutional neural network.
  IEEE Transactions on Intelligent Transportation Systems 22~(2), 1237--1247.

\bibitem[{Du et~al.(2020{\natexlab{b}})Du, Hu, and Zhang}]{du2020missing}
Du, J., Hu, M., Zhang, W., 2020{\natexlab{b}}. Missing data problem in the
  monitoring system: A review. IEEE Sensors Journal 20~(23), 13984--13998.

\bibitem[{Duan et~al.(2016)Duan, Lv, Liu, and Wang}]{duan2016efficient}
Duan, Y., Lv, Y., Liu, Y.-L., Wang, F.-Y., 2016. An efficient realization of
  deep learning for traffic data imputation. Transportation research part C:
  emerging technologies 72, 168--181.

\bibitem[{Feng et~al.(2016)Feng, Sun, Sun, and Xia}]{feng2016image}
Feng, L., Sun, H., Sun, Q., Xia, G., 2016. Image compressive sensing via
  truncated schatten-p norm regularization. Signal Processing: Image
  Communication 47, 28--41.


\bibitem[{Gao and Fan(2020)}]{gao2020robust}
Gao, S., Fan, Q., 2020. Robust schatten-p norm based approach for tensor
  completion. Journal of Scientific Computing 82~(1), 1--23.

\bibitem[{Goulart et~al.(2017)Goulart, Kibangou, and
  Favier}]{goulart2017traffic}
Goulart, J. d.~M., Kibangou, A., Favier, G., 2017. Traffic data imputation via
  tensor completion based on soft thresholding of tucker core. Transportation
  Research Part C: Emerging Technologies 85, 348--362.

\bibitem[{{Gu} et~al.(2014){Gu}, {Zhang}, {Zuo}, and {Feng}}]{gu2014weighted}
{Gu}, S., {Zhang}, L., {Zuo}, W., {Feng}, X., June 2014. Weighted nuclear norm
  minimization with application to image denoising. In: IEEE Conference on
  Computer Vision and Pattern Recognition. pp. 2862--2869.

\bibitem[{Guo et~al.(2019)Guo, Li, and Ban}]{guo2019urban}
Guo, Q., Li, L., Ban, X.~J., 2019. Urban traffic signal control with connected
  and automated vehicles: A survey. Transportation research part C: emerging
  technologies 101, 313--334.

\bibitem[{Habtemichael and Cetin(2016)}]{habtemichael2016short}
Habtemichael, F.~G., Cetin, M., 2016. Short-term traffic flow rate forecasting
  based on identifying similar traffic patterns. Transportation research Part
  C: emerging technologies 66, 61--78.

\bibitem[{Halko et~al.(2011)Halko, Martinsson, and Tropp}]{halko2011finding}
Halko, N., Martinsson, P.-G., Tropp, J.~A., 2011. Finding structure with
  randomness: Probabilistic algorithms for constructing approximate matrix
  decompositions. SIAM review 53~(2), 217--288.

\bibitem[{H{\aa}stad(1990)}]{haastad1990tensor}
H{\aa}stad, J., 1990. Tensor rank is np-complete. Journal of Algorithms 11~(4),
  644--654.


\bibitem[{Jia et~al.(2021)Jia, Dong, Chen, and Yu}]{jia2021missing}
Jia, X., Dong, X., Chen, M., Yu, X., 2021. Missing data imputation for traffic
  congestion data based on joint matrix factorization. Knowledge-Based Systems
  225, 107114.

\bibitem[{Kolda and Bader(2009)}]{kolda2009tensor}
Kolda, T.~G., Bader, B.~W., 2009. Tensor decompositions and applications. SIAM
  review 51~(3), 455--500.

\bibitem[{Li et~al.(2020)Li, Li, Lin, He, and Wang}]{li2020spatiotemporal}
Li, H., Li, M., Lin, X., He, F., Wang, Y., 2020. A spatiotemporal approach for
  traffic data imputation with complicated missing patterns. Transportation
  research part C: emerging technologies 119, 102730.

\bibitem[{Li et~al.(2013)Li, Li, and Li}]{li2013efficient}
Li, L., Li, Y., Li, Z., 2013. Efficient missing data imputing for traffic flow
  by considering temporal and spatial dependence. Transportation research part
  C: emerging technologies 34, 108--120.

\bibitem[{Liu et~al.(2013)Liu, Musialski, Wonka, and Ye}]{liu2013tensor}
Liu, J., Musialski, P., Wonka, P., Ye, J., 2013. Tensor completion for
  estimating missing values in visual data. IEEE Transactions on Pattern
  Analysis and Machine Intelligence 35~(1), 208--220.

\bibitem[{Lu et~al.(2019)Lu, Peng, and Wei}]{lu2019low}
Lu, C., Peng, X., Wei, Y., 2019. Low-rank tensor completion with a new tensor
  nuclear norm induced by invertible linear transforms. In: Proceedings of the
  IEEE/CVF Conference on Computer Vision and Pattern Recognition. pp.
  5996--6004.


\bibitem[{Lyngdoh et~al.(2022)Lyngdoh, Zaki, Krishnan, and
  Das}]{lyngdoh2022prediction}
Lyngdoh, G.~A., Zaki, M., Krishnan, N.~A., Das, S., 2022. Prediction of
  concrete strengths enabled by missing data imputation and interpretable
  machine learning. Cement and Concrete Composites, 104414.

\bibitem[{Mirsky(1975)}]{mirsky1975trace}
Mirsky, L., 1975. A trace inequality of john von neumann. Monatshefte f{\"u}r
  mathematik 79~(4), 303--306.

\bibitem[{Mousavizadeh et~al.(2021)Mousavizadeh, Keyvan-Ekbatani, and
  Logan}]{mousavizadeh2021real}
Mousavizadeh, O., Keyvan-Ekbatani, M., Logan, T.~M., 2021. Real-time turning
  rate estimation in urban networks using floating car data. Transportation
  Research Part C: Emerging Technologies 133, 103457.

\bibitem[{Nie et~al.(2012{\natexlab{a}})Nie, Huang, and Ding}]{nie2012low}
Nie, F., Huang, H., Ding, C., 2012{\natexlab{a}}. Low-rank matrix recovery via
  efficient schatten p-norm minimization. In: Twenty-sixth AAAI conference on
  artificial intelligence.

\bibitem[{Nie et~al.(2012{\natexlab{b}})Nie, Wang, Cai, Huang, and
  Ding}]{nie2012robust}
Nie, F., Wang, H., Cai, X., Huang, H., Ding, C., 2012{\natexlab{b}}. Robust
  matrix completion via joint schatten p-norm and lp-norm minimization. In:
  2012 IEEE 12th International Conference on Data Mining. IEEE, pp. 566--574.

\bibitem[{Powell(1969)}]{powell1969method}
Powell, M.~J., 1969. A method for nonlinear constraints in minimization
  problems. Optimization, 283--298.

\bibitem[{Qu et~al.(2009)Qu, Li, Zhang, and Hu}]{qu2009ppca}
Qu, L., Li, L., Zhang, Y., Hu, J., 2009. Ppca-based missing data imputation for
  traffic flow volume: A systematical approach. IEEE Transactions on
  intelligent transportation systems 10~(3), 512--522.



\bibitem[{Ran et~al.(2016)Ran, Tan, Wu, and Jin}]{ran2016tensor}
Ran, B., Tan, H., Wu, Y., Jin, P.~J., 2016. Tensor based missing traffic data
  completion with spatial--temporal correlation. Physica A: Statistical
  Mechanics and its Applications 446, 54--63.

\bibitem[{Rempe et~al.(2022)Rempe, Franeck, and
  Bogenberger}]{rempe2022estimation}
Rempe, F., Franeck, P., Bogenberger, K., 2022. On the estimation of traffic
  speeds with deep convolutional neural networks given probe data.
  Transportation Research Part C: Emerging Technologies 134, 103448.

\bibitem[{Rodrigues et~al.(2018)Rodrigues, Henrickson, and
  Pereira}]{rodrigues2018multi}
Rodrigues, F., Henrickson, K., Pereira, F.~C., 2018. Multi-output gaussian
  processes for crowdsourced traffic data imputation. IEEE Transactions on
  Intelligent Transportation Systems 20~(2), 594--603.

\bibitem[{Sportisse et~al.(2020)Sportisse, Boyer, and
  Josse}]{sportisse2020imputation}
Sportisse, A., Boyer, C., Josse, J., 2020. Imputation and low-rank estimation
  with missing not at random data. Statistics and Computing 30~(6), 1629--1643.

\bibitem[{Stolfi et~al.(2017)Stolfi, Alba, and Yao}]{stolfi2017predicting}
Stolfi, D.~H., Alba, E., Yao, X., 2017. Predicting car park occupancy rates in
  smart cities. In: International Conference on Smart Cities. Springer, pp.
  107--117.

\bibitem[{Sun et~al.(2015)Sun, Luo, and Ye}]{sun2015expected}
Sun, R., Luo, Z.-Q., Ye, Y., 2015. On the expected convergence of randomly
  permuted admm. arXiv preprint arXiv:1503.06387 4~(6).

\bibitem[{Tak et~al.(2016)Tak, Woo, and Yeo}]{tak2016data}
Tak, S., Woo, S., Yeo, H., 2016. Data-driven imputation method for traffic data
  in sectional units of road links. IEEE Transactions on Intelligent
  Transportation Systems 17~(6), 1762--1771.

\bibitem[{Tan et~al.(2013{\natexlab{a}})Tan, Feng, Feng, Wang, Zhang, and
  Li}]{tan2013tensor}
Tan, H., Feng, G., Feng, J., Wang, W., Zhang, Y.-J., Li, F.,
  2013{\natexlab{a}}. A tensor-based method for missing traffic data
  completion. Transportation Research Part C: Emerging Technologies 28, 15--27.


\bibitem[{Tan et~al.(2016)Tan, Wu, Shen, Jin, and Ran}]{tan2016short}
Tan, H., Wu, Y., Shen, B., Jin, P.~J., Ran, B., 2016. Short-term traffic
  prediction based on dynamic tensor completion. IEEE Transactions on
  Intelligent Transportation Systems 17~(8), 2123--2133.

\bibitem[{Tang et~al.(2015)Tang, Zhang, Wang, Wang, and Liu}]{tang2015hybrid}
Tang, J., Zhang, G., Wang, Y., Wang, H., Liu, F., 2015. A hybrid approach to
  integrate fuzzy c-means based imputation method with genetic algorithm for
  missing traffic volume data estimation. Transportation Research Part C:
  Emerging Technologies 51, 29--40.

\bibitem[{Tang et~al.(2021)Tang, Zhang, Yin, Zou, and Wang}]{tang2021missing}
Tang, J., Zhang, X., Yin, W., Zou, Y., Wang, Y., 2021. Missing data imputation
  for traffic flow based on combination of fuzzy neural network and rough set
  theory. Journal of Intelligent Transportation Systems 25~(5), 439--454.

\bibitem[{Tufte et~al.(2015)Tufte, Bertini, and Harvey}]{tufte2015evolution}
Tufte, K.~A., Bertini, R.~L., Harvey, M., 2015. Evolution and usage of the
  portal data archive: 10-year retrospective. Transportation Research Record
  2527~(1), 18--28.

\bibitem[{Wang et~al.(2019)Wang, Yin, and Zeng}]{wang2019global}
Wang, Y., Yin, W., Zeng, J., 2019. Global convergence of admm in nonconvex
  nonsmooth optimization. Journal of Scientific Computing 78~(1), 29--63.

\bibitem[{Wang et~al.(2018)Wang, Zhang, Piao, Liu, and Zhang}]{wang2018traffic}
Wang, Y., Zhang, Y., Piao, X., Liu, H., Zhang, K., 2018. Traffic data
  reconstruction via adaptive spatial-temporal correlations. IEEE Transactions
  on Intelligent Transportation Systems 20~(4), 1531--1543.

\bibitem[{Watson(1992)}]{watson1992characterization}
Watson, G.~A., 1992. Characterization of the subdifferential of some matrix
  norms. Linear algebra and its applications 170~(0), 33--45.


\bibitem[{Xie et~al.(2018)Xie, Wang, Wang, Xie, Wen, Zhang, Cao, and
  Zhang}]{xie2018accurate}
Xie, K., Wang, L., Wang, X., Xie, G., Wen, J., Zhang, G., Cao, J., Zhang, D.,
  2018. Accurate recovery of internet traffic data: A sequential tensor
  completion approach. IEEE/ACM Transactions on Networking 26~(2), 793--806.

\bibitem[{Xie et~al.(2016)Xie, Gu, Liu, Zuo, Zhang, and
  Zhang}]{xie2016weighted}
Xie, Y., Gu, S., Liu, Y., Zuo, W., Zhang, W., Zhang, L., 2016. Weighted
  schatten $ p $-norm minimization for image denoising and background
  subtraction. IEEE transactions on image processing 25~(10), 4842--4857.

\bibitem[{Yang et~al.(2021{\natexlab{a}})Yang, Peng, and Lin}]{yang2021real}
Yang, J.-M., Peng, Z.-R., Lin, L., 2021{\natexlab{a}}. Real-time spatiotemporal
  prediction and imputation of traffic status based on lstm and graph laplacian
  regularized matrix factorization. Transportation Research Part C: Emerging
  Technologies 129, 103228.

\bibitem[{Yang et~al.(2021{\natexlab{b}})Yang, Han, Liu, Zhu, and
  Yan}]{yang2021novel}
Yang, Y., Han, L., Liu, Y., Zhu, J., Yan, H., 2021{\natexlab{b}}. A novel
  regularized model for third-order tensor completion. IEEE Transactions on
  Signal Processing 69, 3473--3483.

\bibitem[{Yao et~al.(2018{\natexlab{b}})Yao, Kwok, Wang, and
  Liu}]{yao2018large}
Yao, Q., Kwok, J.~T., Wang, T., Liu, T.-Y., 2018{\natexlab{b}}. Large-scale
  low-rank matrix learning with nonconvex regularizers. IEEE transactions on
  pattern analysis and machine intelligence 41~(11), 2628--2643.

\bibitem[{Yokota et~al.(2018)Yokota, Erem, Guler, Warfield, and
  Hontani}]{yokota2018missing}
Yokota, T., Erem, B., Guler, S., Warfield, S.~K., Hontani, H., 2018. Missing
  slice recovery for tensors using a low-rank model in embedded space. In: IEEE
  Conference on Computer Vision and Pattern Recognition. pp. 8251--8259.

\bibitem[{Yu et~al.(2020)Yu, Stettler, Angeloudis, Hu, and Chen}]{yu2020urban}
Yu, J., Stettler, M.~E., Angeloudis, P., Hu, S., Chen, X.~M., 2020. Urban
  network-wide traffic speed estimation with massive ride-sourcing gps traces.
  Transportation Research Part C: Emerging Technologies 112, 136--152.

\bibitem[{Zhang et~al.(2019)Zhang, Chen, Zheng, Zhu, Yu, Wang, and
  Liu}]{zhang2019missing}
Zhang, H., Chen, P., Zheng, J., Zhu, J., Yu, G., Wang, Y., Liu, H.~X., 2019.
  Missing data detection and imputation for urban anpr system using an
  iterative tensor decomposition approach. Transportation Research Part C:
  Emerging Technologies 107, 337--355.

\bibitem[{Zhang et~al.(2021{\natexlab{a}})Zhang, Wang, and
  Liu}]{zhang2021gated}
Zhang, T., Wang, J., Liu, J., 2021{\natexlab{a}}. A gated generative
  adversarial imputation approach for signalized road networks. IEEE
  Transactions on Intelligent Transportation Systems, 1--17.

\bibitem[{Zhang et~al.(2009)Zhang, Roughan, Willinger, and
  Qiu}]{zhang2009spatio}
Zhang, Y., Roughan, M., Willinger, W., Qiu, L., 2009. Spatio-temporal
  compressive sensing and internet traffic matrices. In: Proceedings of the ACM
  SIGCOMM 2009 conference on Data communication. pp. 267--278.

\bibitem[{Zhang et~al.(2021{\natexlab{b}})Zhang, Lin, Li, and
  Wang}]{zhang2021customized}
Zhang, Z., Lin, X., Li, M., Wang, Y., 2021{\natexlab{b}}. A customized deep
  learning approach to integrate network-scale online traffic data imputation
  and prediction. Transportation Research Part C: Emerging Technologies 132,
  103372.

\bibitem[{Zhang et~al.(2021{\natexlab{c}})Zhang, Ling, He, and
  Qi}]{zhang2021tensor}
Zhang, Z., Ling, C., He, H., Qi, L., 2021{\natexlab{c}}. A tensor train
  approach for internet traffic data completion. Annals of Operations Research,
  1--19.

\bibitem[{Zhou et~al.(2015)Zhou, Zhang, Xie, and Chen}]{zhou2015spatio}
Zhou, H., Zhang, D., Xie, K., Chen, Y., 2015. Spatio-temporal tensor completion
  for imputing missing internet traffic data. In: 2015 ieee 34th international
  performance computing and communications conference (ipccc). IEEE, pp. 1--7.

\bibitem[{Zhuang et~al.(2019)Zhuang, Ke, and Wang}]{zhuang2019innovative}
Zhuang, Y., Ke, R., Wang, Y., 2019. Innovative method for traffic data
  imputation based on convolutional neural network. IET Intelligent Transport
  Systems 13~(4), 605--613.

\bibitem[{Zuo et~al.(2013)Zuo, Meng, Zhang, Feng, and
  Zhang}]{zuo2013generalized}
Zuo, W., Meng, D., Zhang, L., Feng, X., Zhang, D., 2013. A generalized iterated
  shrinkage algorithm for non-convex sparse coding. In: Proceedings of the IEEE
  international conference on computer vision. pp. 217--224.
\end{thebibliography}

\end{document}